\documentclass{article}

\usepackage[preprint]{neurips_2026}

\usepackage[utf8]{inputenc} 
\usepackage[T1]{fontenc}    
\usepackage{hyperref}       
\usepackage{url}            
\usepackage{booktabs}       
\usepackage{amsfonts}       
\usepackage{nicefrac}       
\usepackage{microtype}      
\usepackage{xcolor}         
\usepackage{graphicx}
\usepackage{wrapfig}
\usepackage{bbm}
\usepackage{multirow}
\usepackage[table]{xcolor}
\usepackage{makecell}
\usepackage{caption}
\usepackage{subcaption}
\usepackage{amsmath,amssymb}
 \usepackage{hyperref}
 \usepackage{placeins}
\usepackage{amsthm}
\usepackage{float}
\newtheorem{lemma}{Lemma}

\title{E4GEN: Event-level Explainable Extreme-Enhanced Time-series Generation}

\author{%
Lin Jiang, Dahai Yu, Ximiao Li, Guang Wang \\
Florida State University \\
\texttt{\{lin.jiang, dahai.yu, xl24g\}@fsu.edu, guang@cs.fsu.edu}
}

\begin{document}

\maketitle
\vspace{-10pt}
\begin{abstract}

Generating realistic time series is essential for scientific research and real-world applications. However, existing methods often emphasize overall distributional fidelity while failing to faithfully capture extreme events. To advance existing research, we propose E4GEN, an explainable diffusion framework for extreme event-aware time-series generation, which provides systematic insights into \textbf{when}, \textbf{what}, and \textbf{how} to control extreme-event generation through three key components. (i) \textbf{E-Activator} to learn the dataset-adaptive \emph{extreme-control signal activation step} during the denoising process without interfering with regular temporal components (trend and seasonality); (ii) \textbf{E-Predictor}, which determines what control signal to enforce through \emph{Self-Driven Semantic Prediction},
where each sample derives its own control signal by inferring the latent extreme-event information during the generation process. It also includes a novel \emph{Data-Conditioned Training, Noise-Initiated Sampling} mechanism to address the issue of unavailable training labels. and (iii) \textbf{E-Control}, which specifies how to control extreme-event generation through a trainable \emph{Extreme Control Network}, which transforms the semantic control signal into layer-wise signals and injects it into the denoising process. We evaluate our E4GEN with six datasets and 17 metrics, and extensive experiments show that E4GEN outperforms state-of-the-art models from different dimensions, including overall fidelity, extreme-event fidelity, and downstream utility.
\end{abstract}

\section{Introduction}\label{sec:introduction}

Time-series generation has emerged as an important research direction~\cite{yoon2019time,jarrett2021time,galib2024fide,yuan2024diffusionts,rousseau2025forging}, because it can support a wide range of applications such as simulations~\cite{lou2023pcf}, data augmentation~\cite{nikitin2024tsgm}, counterfactual analysis~\cite{yan2023self}, and hypothesis testing~\cite{schreiber2000surrogate}. In recent years, many methods have been proposed for this task, covering diverse paradigms, including GANs~\cite{yoon2019time,jeon2022gt}, VAEs~\cite{desai2021timevae,naimangenerative}, Fourier Flows~\cite{alaa2021generative}, LLMs~\cite{rousseau2025forging,tan2024language}, and diffusion models~\cite{yuan2024diffusionts,kongdiffwave,coletta2023constrained}. 
Although some works~\cite{yuan2024diffusionts, rousseau2025forging} have achieved high overall distributional fidelity, they usually attenuate sparse extreme-event patterns~\cite{galib2024fide,brophy2023generative}, which are also meaningful to capture because they often correspond to high-impact phenomena in the data, such as heatwaves, heavy rainfall, and sudden demand surges~\cite{gu2025beyond}.

To achieve extreme-aware generation, some recent studies have explored long-tailed generation~\cite{zhang2024long,rodriguez2025improving} and heavy-tailed generation~\cite{galib2024fide,jaini2020tails,kim2024t3vae,pandey2025heavy}.
Long-tailed generation improves the fidelity and diversity of tail samples by identifying \emph{which samples are rare} without capturing \emph{which values within each sample are extreme}, so it cannot explicitly characterize within-sample extreme patterns. Although heavy-tailed generation considers within-sample extremes, it mainly focuses on the \textbf{value level} by treating extreme values as isolated observations without modeling their correlations or temporal dependencies. Hence, neither type of method can explicitly represent event-level features, such as location, duration, and temporal shape, which are essential for characterizing time-series data and supporting diverse downstream applications.


To advance existing work, we explore extreme-aware time-series generation from the \textbf{event-level} perspective. 
Instead of treating extreme values as isolated points, we consider extreme events as coherent temporal structures driven by underlying physical mechanisms. For example, as illustrated 
\begin{wrapfigure}{r}{0.5\textwidth}
    \centering
    \begin{minipage}[t]{0.49\linewidth}
        \centering
        \includegraphics[width=\linewidth]{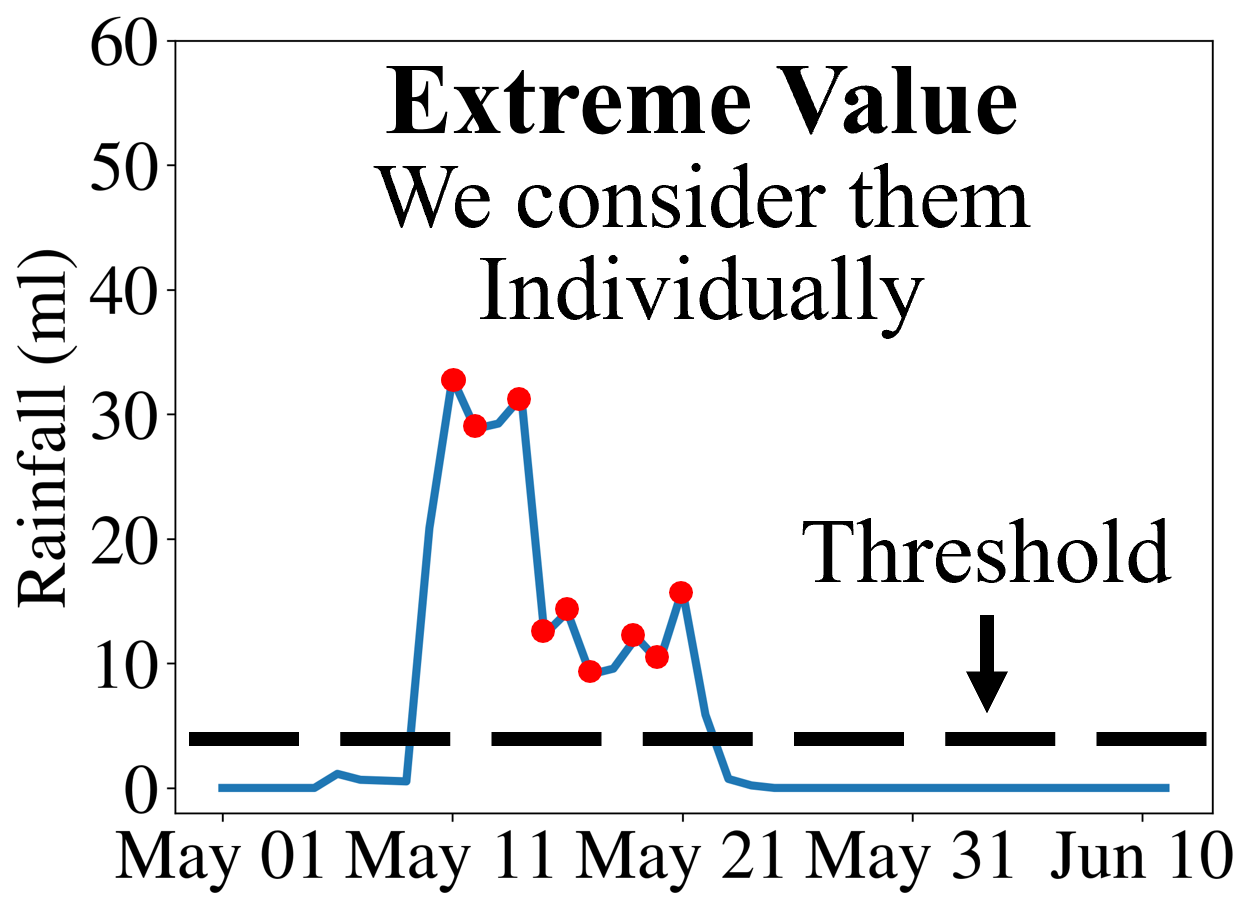}
        \vspace{0.3em}
        \small (a) Value-level view
    \end{minipage}
    \hfill
    \begin{minipage}[t]{0.49\linewidth}
        \centering
        \includegraphics[width=\linewidth]{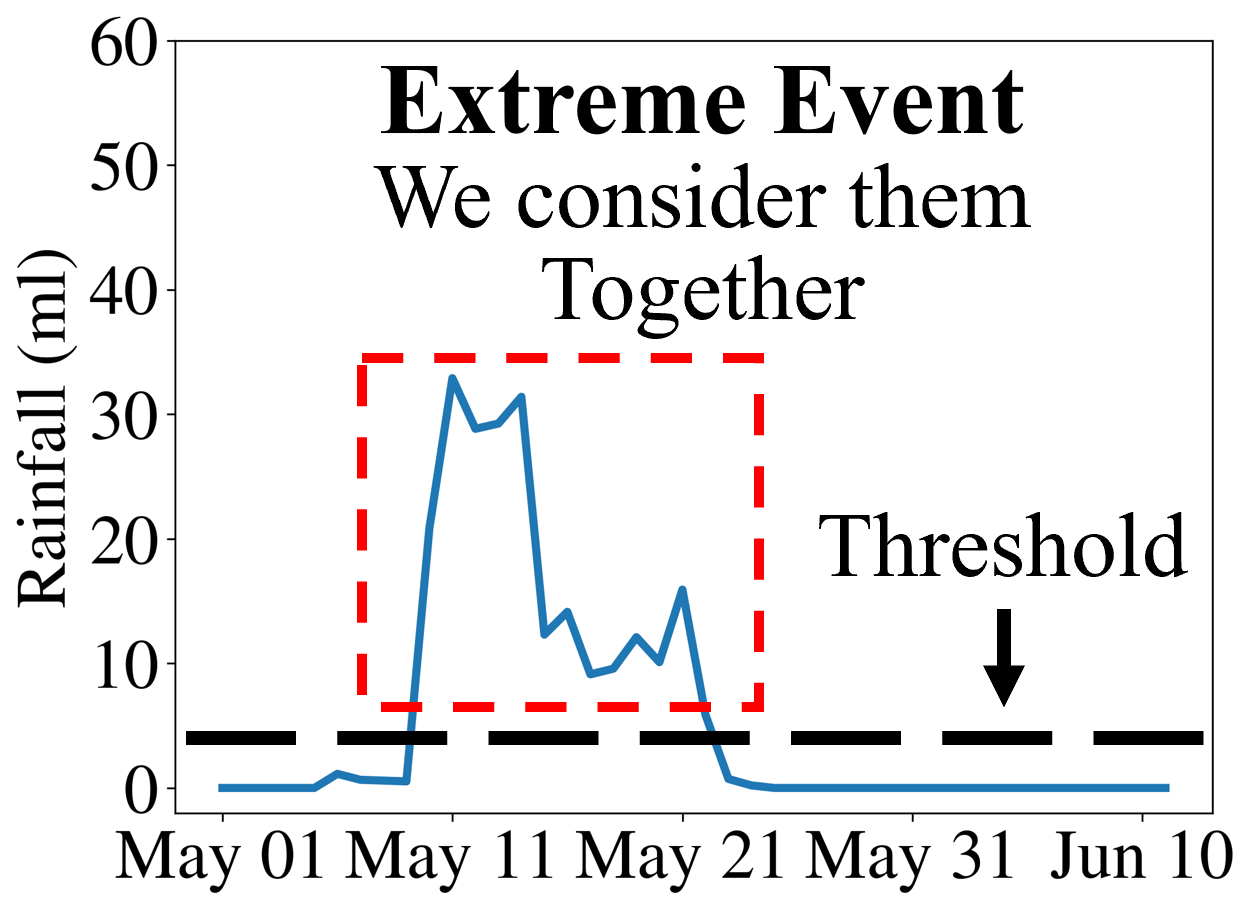}
        \vspace{0.3em}
        \small (b) Event-level view
    \end{minipage}
    \vspace{0.3em}
    \captionof{figure}{Value-Level vs. Event-Level Extremes}
    \label{fig:value_event_extreme}
\end{wrapfigure}
in Fig.~\ref{fig:value_event_extreme}, an extreme rainstorm typically spans a continuous temporal interval, during which rainfall intensities at neighboring timestamps are highly temporally dependent, so the event should be modeled as a coherent temporal process rather than a collection of isolated extreme points. However, there are two major challenges for event-level extreme-aware time-series generation. First, extreme events are localized and sparse in the data, with tightly constrained timing and shape, making them particularly vulnerable to distortion during generation. 
Second, the generation process of extreme events remains poorly understood, as existing methods provide little insight into when and how extreme events emerge during generation, making it difficult to precisely control and preserve them.

To address these challenges, we propose \textbf{E4GEN}, an \underline{E}vent-level \underline{E}xplainable \underline{E}xtreme-\underline{E}nhanced diffusion framework for time-series \underline{GEN}eration. E4GEN provides systematic insights into when, what, and how to control extreme-event generation through three key designs. (i) \textbf{E-Activator} is designed to decide when extreme-control signals should be activated during the generation process. Unlike existing methods that generate general patterns (e.g., trend and seasonality) and extreme events in a coupled manner, E-Activator exploits the coarse-to-fine generation dynamics of diffusion models to learn a \emph{dataset-adaptive control activation step}, where the extreme-control signal is activated only after general temporal structures have been sufficiently established. This design effectively decouples extreme-event guidance from general temporal structure formation. (ii) \textbf{E-Predictor} is designed to determine what control signals should be activated. Motivated by the observation that extreme events arise from the underlying temporal context, E-Predictor introduces a \emph{self-driven semantic prediction} (SDSP) module, enabling each sample to derive its own semantic-based extreme-control signal from the general temporal patterns already established during the denoising process. It also introduces a novel \emph{Data-Conditioned Training, Noise-Initiated Sampling} (DCT-NIS) mechanism to address the issue of unavailable training
labels.
(iii) \textbf{E-Control} is designed to answer how semantic-based control signals can be leveraged for event-level extreme-aware generation. Specifically, we design a trainable \emph{Extreme Control Network} that transforms predicted semantics into layer-aligned residual offsets, which are used to guide extreme-event generation while preserving general temporal structures. The contributions of this paper are summarized as follows:

(i) \textbf{Conceptually}, we shift extreme-aware time-series generation from sample-level or value-level to an event-level perspective,
which are essential for characterizing temporal dynamics in time-series data and supporting diverse downstream applications.

(ii) \textbf{Technically}, we propose E4GEN, an explainable diffusion framework for extreme-aware time-series generation, which provides systematic insights into the controllability of extreme-event generation through three core designs: E-Activator identifies when to activate the extreme-event control signal via a dataset-adaptive control activation step; E-Predictor defines what to control through self-driven semantic prediction combined with a DCT-NIS mechanism that addresses the training-sampling mismatch issue; and E-Control determines how to control generation by injecting semantic control signals into the denoising process through a trainable Extreme Control Network.

(iii) \textbf{Experimentally}, we evaluate E4GEN on six datasets against 9 baselines in terms of 17 metrics. Extensive experiments show that E4GEN achieves superior performance across different dimensions, such as overall generation fidelity, extreme-aware generation fidelity, and downstream utility.


\section{Data-Driven Insights} \label{sec:analysis}

In this section, we present a data-driven analysis to motivate our problem formulation and design. We summarize three key observations below.

\textbf{Observation I: Value-level generation often fails to recover the true event-level structure of extremes.}
In Fig.~\ref{fig:value_level_distribution}, we compare the generated results of two representative methods: (i) a standard light-tailed diffusion model implemented with Diffusion-TS~\cite{yuan2024diffusionts}, and (ii) a value-level extreme-enhanced heavy-tailed diffusion model implemented with HeavyDiff~\cite{pandey2025heavy} with the Original Data. 
From Fig.~\ref{fig:value_level_distribution}(a), we found that the generated data generally match the overall distribution, but fail to match the tail. 
\begin{wrapfigure}{r}{0.46\textwidth}
    \centering
    \begin{minipage}[t]{0.46\linewidth}
        \centering
        \includegraphics[width=\linewidth]{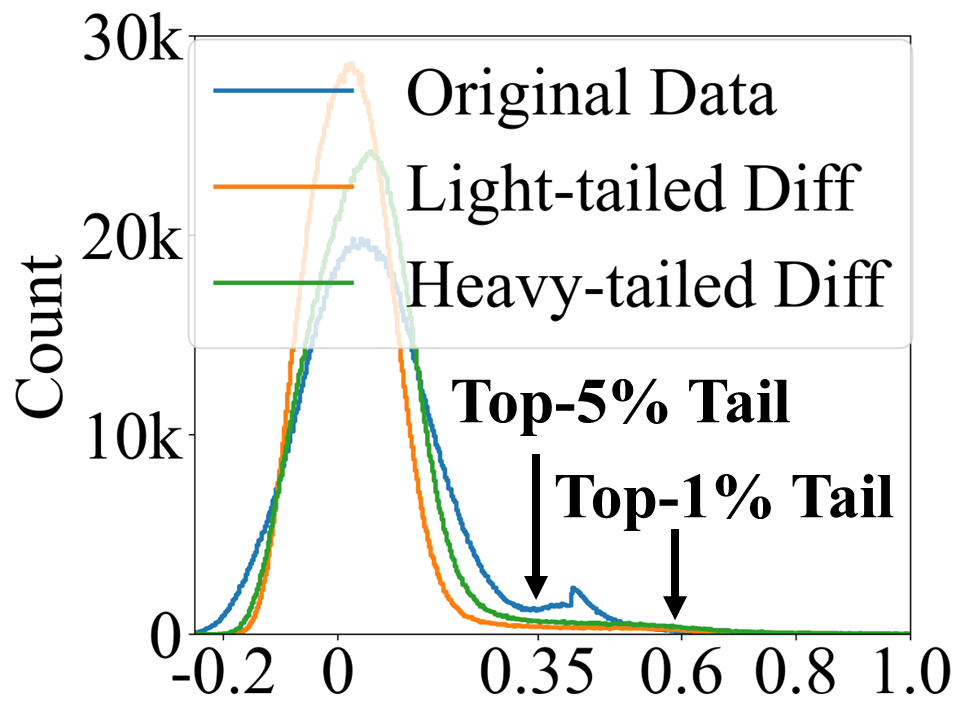}
        \vspace{0.2em}
        \small (a) Overall Distribution
    \end{minipage}
    \hfill
    \begin{minipage}[t]{0.48\linewidth}
        \centering
        \includegraphics[width=\linewidth]{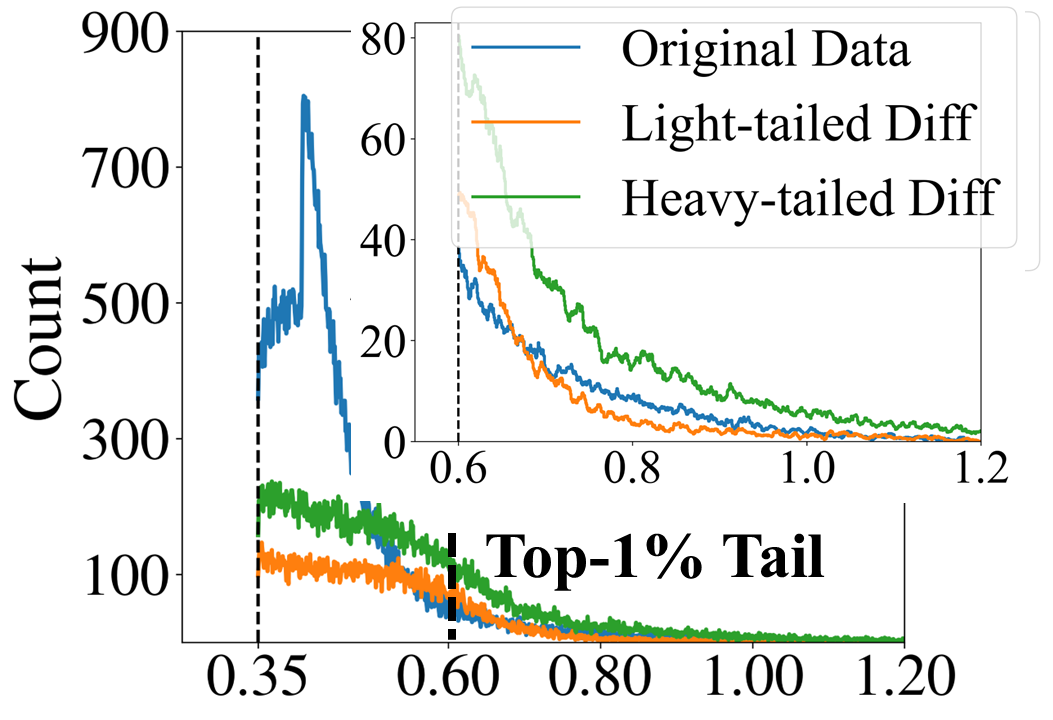}
        \vspace{0.2em}
        \small (b) Tail Distribution
    \end{minipage}
    \captionof{figure}{Value-level Extreme Enhancement.}
    \label{fig:value_level_distribution}
\end{wrapfigure}
Fig.~\ref{fig:value_level_distribution}(b) illustrates the top-5\% tail distribution, where the original data exhibits a pronounced density surge within the range of 0.35--0.6, indicating that extreme values tend to occur within a specific and non-random value range with clear structural regularity. In contrast, both methods fail to recover this characteristic pattern. Although the heavy-tailed diffusion model generates a heavier tail distribution, it still fails to faithfully capture the underlying event-level extreme structure. Further discussion of value-level extreme-aware generation is provided in Appendix~\ref{Appendix:Discussion about value-level enhencement}.


\vspace{-6pt}
\begin{figure*}[h]
    \centering
    \begin{minipage}[t]{0.72\textwidth}
        \centering
        \includegraphics[width=\linewidth]{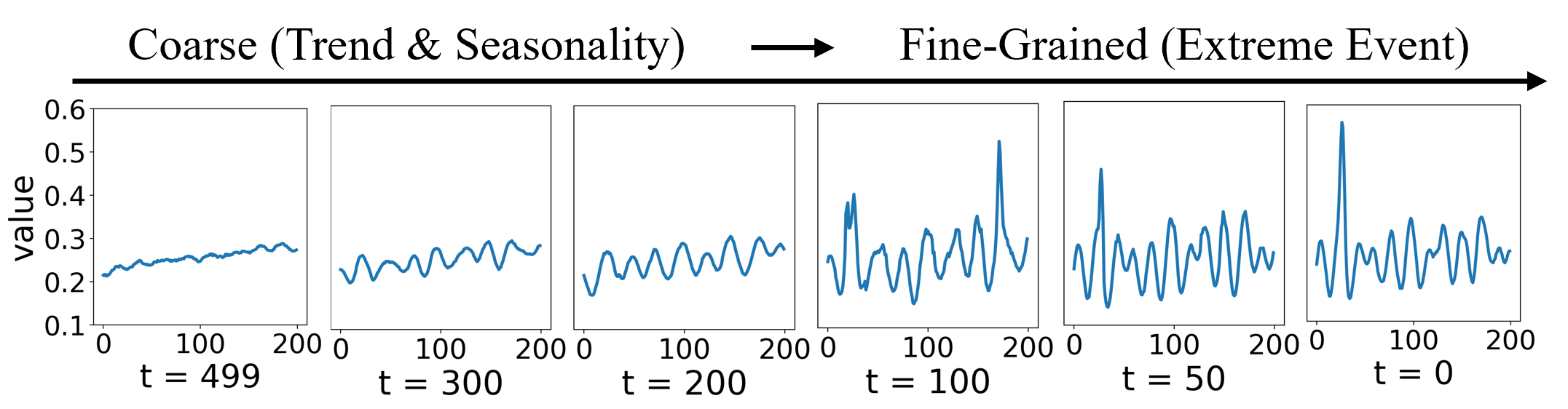}
        \captionof{figure}{A Sample Denoising Process}
        \label{fig:trajectory_sample}
    \end{minipage}
    \hfill
    \begin{minipage}[t]{0.27\textwidth}
        \centering
        \includegraphics[width=\linewidth]{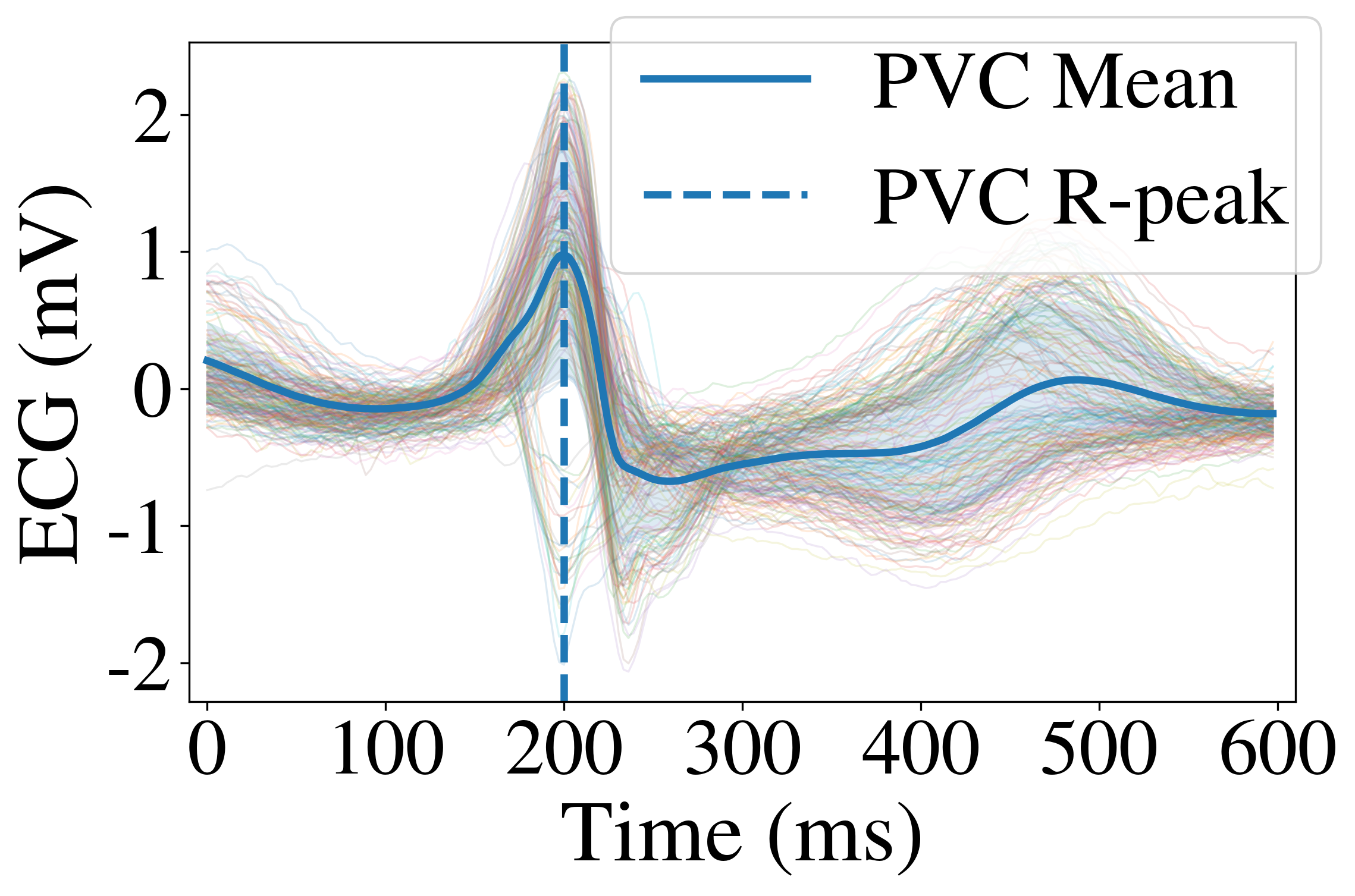}
        \captionof{figure}{Stability}
        \label{fig:stability}
    \end{minipage}
    \vspace{-10pt}
\end{figure*}
\textbf{Observation II: Diffusion models exhibit coarse-to-fine generation dynamics, progressively evolving from global temporal structures to localized extreme-event patterns.}
Fig.~\ref{fig:trajectory_sample} visualizes the generation process of a time-series sample along the 500-step sampling trajectory of a diffusion model~\cite{li2025back}.
The trajectory reveals a staged generation process in which the model first establishes coarse global structures, primarily trend and seasonality, and then progressively refines localized extreme events during later sampling stages, especially within the final 100 steps. This observation motivates the design of E4GEN, which dynamically activates extreme-control signals during denoising rather than applying them uniformly throughout the entire generation trajectory. Additional visualizations are provided in Appendix~\ref{Appendix:Additional Visualizations for Coarse-to-fine Generation Process}.

\textbf{Observation III: Extreme events are typically induced by underlying physical processes, and therefore exhibit \emph{intrinsic patterns} rather than random \emph{extreme noise}.} To validate this observation,
we perform three complementary analyses including cross-sample stability, predictability, and context dependency across six datasets. Fig.~\ref{fig:stability} shows a cross-sample stability analysis, which indicates that premature ventricular contractions (PVCs) in the ECG dataset preserve clear event-level structures despite variations in amplitude and peak timing. Results for predictability and context dependency are provided in Appendix~\ref{Appendix:Supporting Analysis for Intrinsic Patterns in Extreme Events}.
This observation can also be theoretically supported by the limited degrees of freedom (DOF) that characterize extreme-event variability~\cite{lorenz2017deterministic}: although extreme events may appear irregular, their variations are often governed by a small set of latent factors, making recurring structures learnable by neural networks~\cite{altmann2005recurrence,bunde2003effect}.

\section{PRELIMINARIES} \label{sec:preliminary}
\subsection{Basic Definitions} \label{sec:basic definition}

In this paper, we use Greek letters to indicate time-series indices and Latin letters for diffusion steps. 
A multivariate time series is denoted as $\mathbf{x}\in\mathbb{R}^{L\times C}$, where $L$ is the sequence length and $C$ is the number of feature channels. 
Throughout this paper, the \textbf{diffusion process} refers to the forward process that transforms a clean sample $\mathbf{x}_d(t_0)$ into pure noise $\mathbf{x}_d(t_T)$, while the \textbf{denoising process} refers to the reverse sampling process from $\mathbf{x}_n(t_T)$ back toward $\mathbf{x}_n(t_0)$.
We denote the noise-initiated denoising trajectory as
$\{\mathbf{x}_n(t)\}_{t=t_T}^{t_0}$ and each intermediate state denotes
$\mathbf{x}_n(t)=\smash[b]{[x_{n,\tau,\nu}(t)]_{\tau=1,\nu=1}^{L,C}}\in\mathbb{R}^{L\times C}$, where $\tau$ indexes positions within the time series, $\nu$ is the feature channel, and $t$ is the diffusion step. 



\subsection{$\mathbf{x}$-Prediction Network as Backbone Denoiser}
\label{sec:backbone denoiser}
In E4GEN, the \textbf{Backbone Denoiser} refers to the base denoising network that serves as the foundation of the time-series generation process, upon which the extreme-event generation modules are further built. Its primary role is to capture the global temporal patterns of the sequence, particularly long-range trend and seasonality patterns. In this work, we instantiate the Backbone Denoiser using the $\mathbf{x}$-prediction network 
$\hat{\mathbf{x}}^\theta_{0}(\cdot)$, a widely adopted and effective paradigm for time-series generation~\cite{li2025back,nichol2021improved,yuan2024diffusionts,park2026timebridge}, which predicts the clean sample $\mathbf{x}_n(t_0)$ from an intermediate noisy state $\mathbf{x}_n(t)$ at diffusion step $t$:
\begin{equation} \label{eq:x0-prediction}
\hat{\mathbf{x}}_{0}(t) \triangleq 
\hat{\mathbf{x}}^\theta_{0}\big(\mathbf{x}_n(t),t\big) 
\;\equiv\; 
\hat{\mathbf{x}}^\theta_{0}\big(\mathbf{x}_n(t)\big)\ \text{for short,}
\end{equation}
where $\hat{\mathbf{x}}_{0}(t)$ denotes the predicted clean sample at step $t$. 
$\mathbf{x}$-prediction network can provide an interpretable generation trajectory with theoretical support from manifold analysis~\cite{li2025back}. 

Following classical seasonal-trend decomposition~\cite{cleveland1990stl}, we parameterize the $\mathbf{x}$-prediction network as $\smash{\hat{\mathbf{x}}^\theta_{0}}(\mathbf{x}(t),t)
=\mathrm{Trend}(\mathbf{x}_n(t),\theta_{tr})+
\mathrm{Seasonality}(\mathbf{x}_n(t),\theta_{se})$. The trend and seasonality components are instantiated as~\cite{yuan2024diffusionts}:
\begin{align}
   \mathrm{trend}(\mathbf{x}_n(t),\theta_{tr})
   &= \phi(\tau)^\top \big( W_{tr}\, h(\mathbf{x}_n(t);a_{tr}) + b_{tr} \big), 
   \label{eq:3} \\
   \mathrm{seasonality}\!\left(\mathbf{x}_n(t),\theta_{se}\right)
   &= \mathcal{IDFT}\!\left(
      \mathcal{M}_K\!\left(\mathcal{DFT}(\mathbf{x}_n(t)),\theta_{se}\right)
   \right).
   \label{eq:4}
\end{align}
where $\phi(\tau)=[1,\tau,\tau^2,\ldots,\tau^p]^\top$ is a polynomial basis, and $(W_{tr},a_{tr},b_{tr})$ are the learnable trend parameters in $\theta_{tr}$. $\mathcal{DFT}$ and $\mathcal{IDFT}$ denote the discrete Fourier transform and its inverse, respectively, while $\mathcal{M}_K(\cdot,\theta_{se})$ selects and refines the Top-$K$ frequency components. More details of the Backbone Denoiser are provided in Appendix~\ref{Appendix:Backbone Denoiser}.

\subsection{Extreme-Event Definition and Representation} \label{sec:extreme definition}

Extreme events are inherently domain-dependent, and their definitions often vary across application scenarios~\cite{mcphillips2018defining}. In E4GEN, we adopt a threshold-based definition following the threshold-exceedance perspective commonly used in extreme value analysis~\cite{davison1990models}, which provides a general and flexible mechanism for identifying extreme events across diverse time-series domains. Specifically, an extreme event is defined as a maximal contiguous segment whose values exceed an extreme threshold $\theta$ for upper-tail extremes, or fall below $\theta$ for lower-tail extremes. The threshold $\theta$ is domain-specific and determined according to prior knowledge or task-specific requirements, rather than being fixed universally across datasets. Appendix~\ref{Appendix:fluctuate_extreme_event_definition} further discusses how E4GEN handles cases where short non-extreme gaps appear within an extreme event.

Since extreme events cannot be directly incorporated into diffusion generation, we propose a novel extreme-event semantic representation that characterizes each event with structured attributes. This design transforms extreme events into a model-accessible and controllable semantic form, enabling explicit event-level control during generation. Specifically, for each time-series sample, we define its \textbf{extreme-event semantics} as $EES=\{\mathcal{L},\mathcal{I},\mathcal{S}\}$, where $\mathcal{L}$, $\mathcal{I}$, and $\mathcal{S}$ characterize the location, intensity, and temporal shape of all extreme events within it, respectively.

\textbf{Location $\mathcal{L}$:}
The location component $\mathcal{L}$ characterizes the temporal occurrence of extreme events, including their count and spans.
We define $\mathcal{L}=\{n_{\mathrm{ev}}, \{(\tau_s^m,\tau_e^m)\}_{m=1}^{n_{\mathrm{ev}}}\}$, where $n_{\mathrm{ev}}$ denotes the number of events, and $(\tau_s^m,\tau_e^m)$ are the start and end time indices of the $m$-th event. Thus, $\mathcal{L}$ specifies when each event occurs and how long it persists.

\textbf{Intensity $\mathcal{I}$:}
The intensity component $\mathcal{I}$ quantifies the strength of each event relative to its local background.
We define
$\mathcal{I}=\{(a_{\mathrm{peak}}^m,a_{\mathrm{prom}}^m,a_{\mathrm{eng}}^m)\}_{m=1}^{n_{\mathrm{ev}}}$.
Here, $a_{\mathrm{peak}}^m=\max_{\tau\in[\tau_s^m,\tau_e^m]}|x[\tau]|$ denotes the peak magnitude within the event window;
$a_{\mathrm{prom}}^m=a_{\mathrm{peak}}^m-b^m$ denotes the peak prominence above the background level $b^m$, estimated from nearby non-event points;
and $\smash{a_{\mathrm{eng}}^m=\sum_{\tau=\tau_s^m}^{\tau_e^m}(x[\tau]-b^m)^2}$ denotes the cumulative deviation energy relative to this background level.

\textbf{Temporal Shape $\mathcal{S}$:}
The temporal-shape component $\mathcal{S}$ provides a sparse keypoint-based representation of each extreme event, summarizing its local evolution with informative points rather than all values in the event segment.
For the $m$-th event over $[\tau_s^m,\tau_e^m]$, keypoints are selected from $x[\tau_s^m:\tau_e^m]$ to capture its temporal profile, including onset and termination, rising or falling trends, and major local variations.
Formally, we define $\mathcal{S}=\smash\{{\{(p_{\mathrm{len}}^m,\mathbf{p}_\tau^m,\mathbf{p}_x^m)\}_{m=1}^{n_{\mathrm{ev}}}}\}$, where each tuple encodes the keypoint-based shape of the $m$-th event.
Here, $p_{\mathrm{len}}^m$ denotes the number of keypoints, $\smash{\mathbf{p}_\tau^m=(p_{\tau,1}^m,\ldots,p_{\tau,p_{\mathrm{len}}^m}^m)}$ stores their time indices, and $\mathbf{p}_x^m=(p_{x,1}^m,\ldots,p_{x,p_{\mathrm{len}}^m}^m)$ stores the corresponding values, with $p_{x,i}^m=x[p_{\tau,i}^m]$.
In practice, we select structural points that characterize the event trajectory, such as endpoints, turning points, and local extrema.
This design is motivated by the intuition that a set of informative points can effectively summarize complex temporal patterns.~\cite{ye2009time}.

\section{METHODOLOGY} \label{sec:method}
Figure~\ref{fig:framework} shows the overall framework of our proposed \textbf{E4GEN}, which integrates three key components that systematically address \emph{when}, \emph{what}, and \emph{how} to control extreme-event generation during the denoising process. 
Section~\ref{sec:E-Activate} introduces \textbf{E-Activator}, which learns to decide the dataset-adaptive control activation step $t_{\mathrm{CA}}$. Section~\ref{sec:E-Predictor} presents \textbf{E-Predictor}, which estimates the extreme-control signal at $t_{\mathrm{CA}}$ with two coupled designs: Self-Driven Semantic Prediction (SDSP) and Data-Conditioned Training with Noise-Initiated Sampling (DCT-NIS) mechanism. 
Section~\ref{sec:E-control} presents \textbf{E-Control}, which includes a trainable \emph{Extreme Control Network} to inject the predicted semantic-based control signal into the Backbone Denoiser for guided extreme-event generation.


\subsection{E-Activator: Learning the Control Activation Step}
\label{sec:E-Activate}

Motivated by the coarse-to-fine generation dynamics identified in Observation~II (Section~\ref{sec:analysis}), E-Activator learns the optimal control activation step, denoted by $t_{\mathrm{CA}}$, for extreme-event control. This activation step is inherently dataset-dependent, since extreme patterns may emerge at different steps of the denoising process in various datasets. For instance, in a fixed 1000-step generation process, heavy-rainfall extremes may begin to appear around $t=100$ in precipitation data, whereas traffic congestion extremes may emerge around $t=50$ in traffic data. To avoid disrupting the generation of regular temporal patterns, we only apply the extreme-control signals starting from $t_{\mathrm{CA}}$ during the denoising process, which is defined based on two criteria: 
(1) regular patterns such as trend and seasonality should already be sufficiently stabilized, so that activating the control signal will not interfere with their formation; and (2) extreme-event patterns should still be at an early stage of generation, leaving enough remaining denoising steps for the control signal to effectively guide their generation.

\begin{figure*}[htbp]
  \centering
   \vspace{-5pt}
  \includegraphics[width=\linewidth]{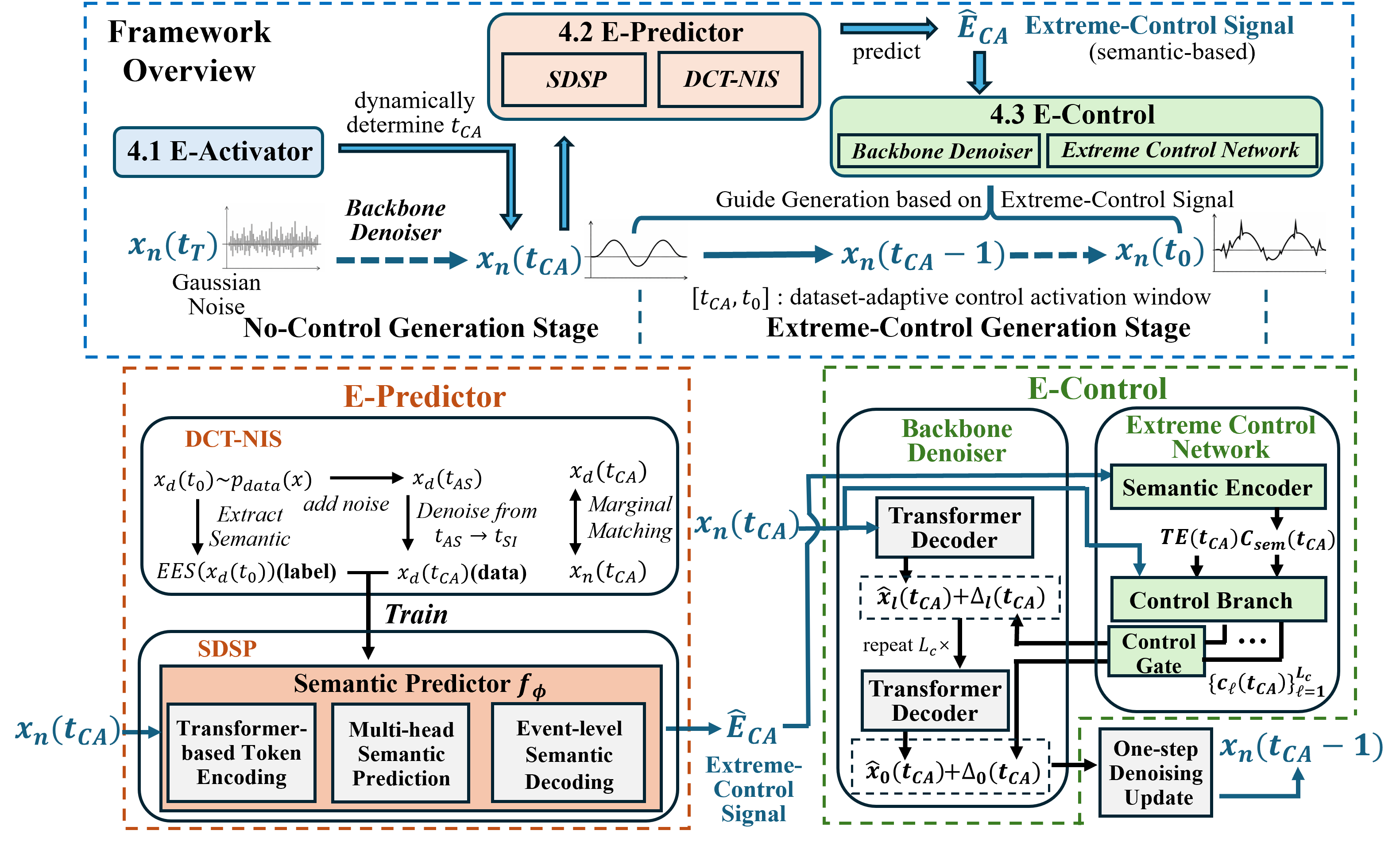}  
  \caption{The Overall Framework of E4GEN}
  \label{fig:framework}
  \vspace{-5pt}
\end{figure*}

We quantify these criteria using two complementary metrics: Backbone Drift (BD) and Outlier Emergence Score (OES), which are formulated as follows:
\begin{equation}
\mathrm{BD}(t)=\mathbb{E}\left[\frac{1}{k}\sum_{j=0}^{k-1}
\frac{\sqrt{\frac{1}{LC}\sum_{\tau=1}^{L}\sum_{\nu=1}^{C}\left(bb_{t-j,\tau,\nu}-bb_{t-j-1,\tau,\nu}\right)^2}}
{\sqrt{\frac{1}{LC}\sum_{\tau=1}^{L}\sum_{\nu=1}^{C}bb_{t-j-1,\tau,\nu}^2}}\right]
\end{equation}
\begin{equation}
\mathrm{OES}(t)=\mathbb{E}\left[\max\left(\left|\frac{\hat{x}_{0,\tau,\nu}(t)-\mathrm{med}_{\tau}(\hat{x}_{0,\tau,\nu}(t))}{1.4826\,\mathrm{mad}_{\tau}(\hat{x}_{0,\tau,\nu}(t))+\varepsilon}\right|-\kappa,\,0\right)\right]
\end{equation}

\begin{wrapfigure}{r}{0.59\textwidth}
    \centering
    \includegraphics[width=0.58\textwidth]{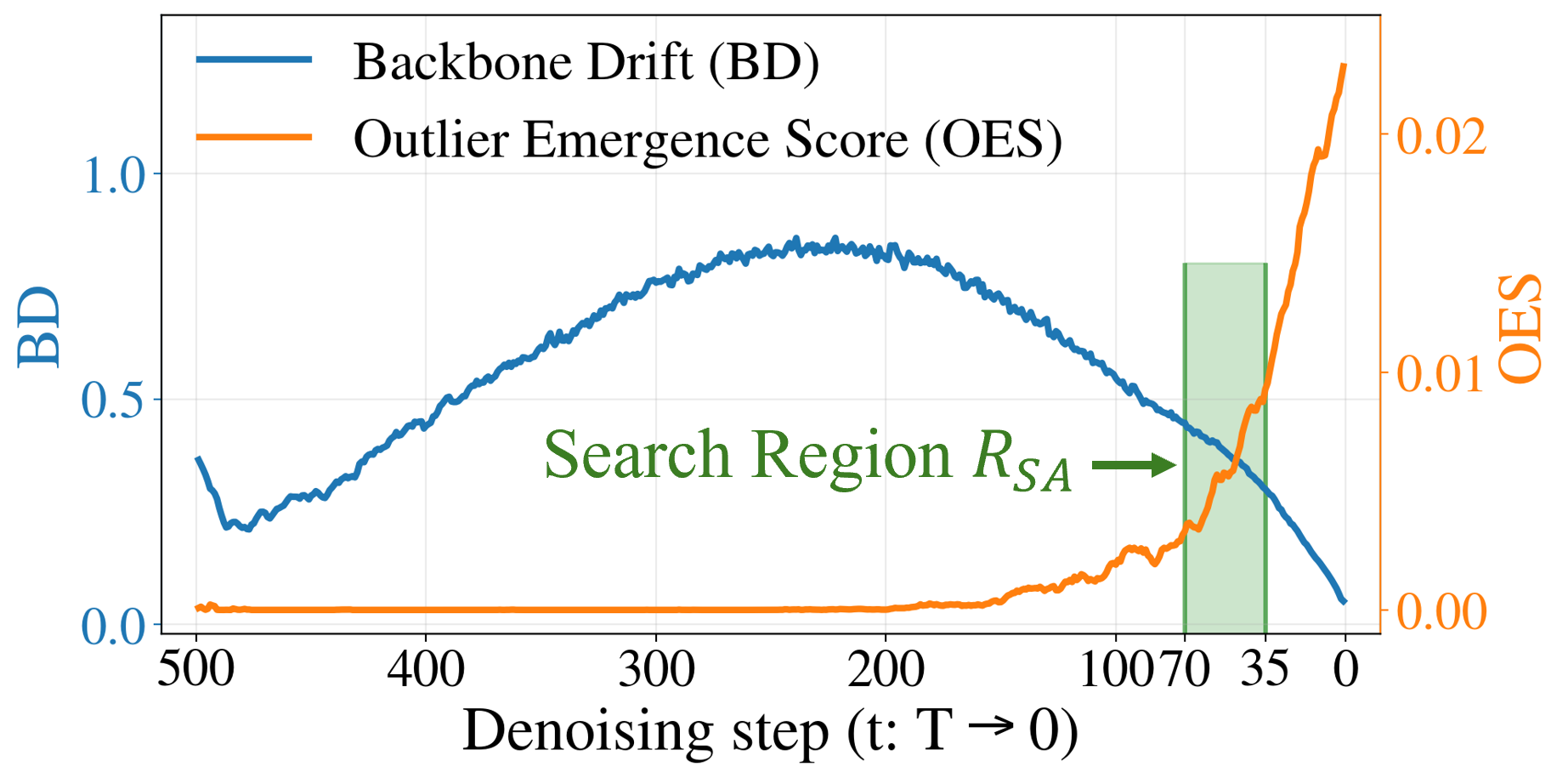}
    \caption{Control Activation Window via BD and OES}
    \label{fig:EActivation_interval}
    \vspace{-1.0em}
\end{wrapfigure}

BD quantifies the stability of regular temporal patterns (e.g., trend and seasonality), while OES characterizes the emergence of extreme-event patterns. Smaller values of BD and OES indicate stronger satisfaction of Criteria~1 and~2, respectively. We select $t_{\mathrm{CA}}$ through a two-stage procedure. First, BD and OES narrow the search space to a compact interval in which both metrics remain sufficiently low. Then, a lightweight validation grid search is performed within this interval to identify the optimal control activation step. More details are provided in Appendix~\ref{Appendix:E-Activation}.

Using Figure~\ref{fig:EActivation_interval} as an illustrative example, we visualize the evolution of BD and OES on a dataset to identify a compact interval for $t_{\mathrm{CA}}$. Specifically, we define the search space as
$\mathcal{R}_{\mathrm{CA}}=\{t \mid \mathrm{BD}(t)\le \tau_{\mathrm{BD}} \ \text{and}\ \mathrm{OES}(t)\le \tau_{\mathrm{OES}}\}$,
where the regular temporal patterns are relatively stable while extreme-event patterns remain underdeveloped. This interval narrows the search space for $t_{\mathrm{CA}}$, allowing a lightweight validation grid search to select the optimal control activation step. Once $t_{\mathrm{CA}}$ is selected, as shown in Figure~\ref{fig:framework}, we activate the Extreme-Control Generation Stage and apply the predicted control signals throughout the control activation window $[t_{\mathrm{CA}}, t_0]$.

\subsection{E-Predictor: Predicting the Extreme-Control Signal}
\label{sec:E-Predictor}
E-Predictor predicts the extreme-control signal at step $t_{\mathrm{CA}}$ for the subsequent Extreme-Control Generation Stage. It consists of two key designs: Self-Driven Signal Prediction (SDSP) and Data-Conditioned Training with Noise-Initiated Sampling (DCT-NIS). Specifically, SDSP incorporates a Semantic Predictor $f_\phi$ to enable a self-driven prediction process, deriving sample-specific control signals from the current intermediate denoising state. DCT-NIS further addresses the label-unavailability issue involved in training $f_\phi$.

\subsubsection{ Self-Driven Signal Prediction (SDSP)}
\label{sec:SDSP}


Motivated by Observation~III in Section~\ref{sec:analysis}, which shows that extreme events are learnable with context-dependent structures rather than random outliers, we design SDSP to predict extreme-control signals from the intermediate denoising state $\mathbf{x}_n(t_{\mathrm{CA}})$.
As discussed in Section~\ref{sec:E-Activate}, the intermediate state $\mathbf{x}_n(t_{\mathrm{CA}})$ already contains stabilized trend and seasonality patterns, providing informative temporal context for predicting potential extreme events. Since extreme events are strongly associated with such context, we design a Semantic Predictor $f_\phi$ in SDSP, which consists of \textit{Transformer-based token encoding}, \textit{multi-head semantic prediction}, and \textit{event-level semantic decoding}, to infer the most likely extreme-event semantics (Section~\ref{sec:extreme definition}) that will emerge in the final denoised sample $\mathbf{x}_n(t_0)$ based on $\mathbf{x}_n(t_{\mathrm{CA}})$.
The predicted semantics $\smash{\hat{E}_{\mathrm{CA}}}$ are then used as extreme-control signals to guide the subsequent denoising process toward realistic extreme-event patterns. In this way, each sample will receive a distinct control signal derived from its own intermediate denoising state.
Detailed design will be provided in Appendix~\ref{Appendix:E-Predictor design}.

\subsubsection{Data-Conditioned Training with Noise-Initiated Sampling (DCT-NIS) Mechanism} \label{sec:DCT-NIS}
We also design the DCT-NIS mechanism to address the label-unavailability issue in training the Semantic Predictor $f_\phi$. Specifically, during sampling, the intermediate state $\mathbf{x}_n(t_{\mathrm{CA}})$ is generated from pure noise, while its corresponding final denoised sample $\mathbf{x}_n(t_0)$ has not yet been produced. Therefore, the ground-truth extreme-event semantics required to supervise $f_\phi$ are unavailable at this stage. DCT-NIS resolves this challenge by training $f_\phi$ on data-conditioned intermediate states paired with accessible real-data semantic labels, and then transferring the learned predictor to guide noise-initiated states during sampling. It consists of two processes: data-conditioned training and noise-initiated sampling.

\textbf{Data-conditioned training process}. Given a clean sample 
$\mathbf{x}_d(t_0)\sim p_{\mathrm{data}}$, we introduce an earlier alignment start step $t_{\mathrm{AS}}$ positioned before and close to $t_{\mathrm{CA}}$. The clean sample is first diffused to $t_{\mathrm{AS}}$ and then denoised for a few steps from $t_{\mathrm{AS}}$ to $t_{\mathrm{CA}}$. This allows the resulting data-conditioned state $\mathbf{x}_d(t_{\mathrm{CA}})$ to experience the same denoising path as the noise-initiated state used during sampling:
\begin{equation}
    \mathbf{x}_d(t-1) \sim p_\theta\!\left(\mathbf{x}_d(t-1)\mid \mathbf{x}_d(t)\right),
    \quad t=t_{\mathrm{AS}},\ldots,t_{\mathrm{CA}}+1 .
\end{equation}
Here, $p_\theta(\mathbf{x}(t-1)|\mathbf{x}(t))$ denotes the one-step denoising transition introduced in Eq.~\ref{eq:x_update}. The Semantic Predictor $f_\phi$ takes the $\mathbf{x}_d(t_{\mathrm{CA}})$ as input, with the corresponding extreme-event semantics $EES(\mathbf{x}_d(t_0))$ in final denoised sample $\mathbf{x}_d(t_0)$ serving as the supervision label:
\begin{equation}
\text{Semantic Predictor} \ f_\phi:\ 
\hat{\mathbf{x}}^\theta_{0}\big(\mathbf{x}_d(t_{\mathrm{CA}})\big)
\mapsto 
EES(\mathbf{x}_d(t_0)).
\end{equation}
\textbf{Noise-initiated sampling process}. This stage is initialized with pure Gaussian noise 
$\mathbf{x}_n(t_T)$, which is denoised to obtain $\mathbf{x}_n(t_{\mathrm{CA}})$. The trained predictor $f_\phi$ is then applied to $\mathbf{x}_n(t_{\mathrm{CA}})$:
\begin{equation}\label{eq:ees_pre}
\hat{E}_{\mathrm{CA}}
\triangleq f_\phi\big(\hat{\mathbf{x}}^\theta_{0}\big(\mathbf{x}_n(t_{\mathrm{CA}})\big)\big)
= EES(\mathbf{x}_n(t_0)),
\end{equation}
where $\hat{E}_{\mathrm{CA}}$ denotes the predicted extreme-event semantics and serves as the extreme-control signal. Although DCT-NIS induces a training-sampling distribution shift between $\mathbf{x}_d(t_{\mathrm{CA}})$ and $\mathbf{x}_n(t_{\mathrm{CA}})$, we prove that the shift remains bounded under the DDPM setting~\cite{ho2020denoising}, with two lemmas shown below and proofs and empirical validations are provided in Appendix~\ref{Appendix:Proof}.

\begin{lemma}[Semantic Consistency]\label{lem:semantic_consistency}
Let $\mathbf{x}_d(t_0)\sim p_{\mathrm{data}}$ and let 
$\mathbf{x}_d(t_{\mathrm{AS}})\sim q(\mathbf{x}_d(t_{\mathrm{AS}})\mid \mathbf{x}_d(t_0))$, 
where $q$ denotes the forward diffusion process. 
Suppose that the denoising trajectory satisfies bounded reconstruction drift and that the event-level semantics is locally stable. 
If the reverse span from the alignment start step $t_{\mathrm{AS}}$ to the control activation step $t_{\mathrm{CA}}$ is sufficiently short, 
then the anchored semantics $EES(\mathbf{x}_d(t_0))$ remains a valid supervision signal for the predictor at 
$\mathbf{x}_d(t_{\mathrm{CA}})$.
\end{lemma}

\begin{lemma}[Marginal Alignment]\label{lem:marginal_alignment}
Under the current DDPM setting, the marginal distribution of the data-conditioned state $\mathbf{x}_d(t_{\mathrm{CA}})$ remains close to that of the noise-initiated state $\mathbf{x}_n(t_{\mathrm{CA}})$.
\end{lemma}

\subsection{E-Control: Semantics-Conditioned Control for Extreme-Event Generation}\label{sec:E-control} 
We then propose \textbf{E-Control}, a semantics-conditioned control module for extreme-event time-series generation. E-Control consists of two parts: a \textbf{Backbone Denoiser}, which serves as the base denoising network 
(Section~\ref{sec:backbone denoiser}), and a trainable \textbf{Extreme Control Network}, 
which is a plugin module to inject semantic control signals $\smash{\hat{E}_{\mathrm{CA}}}$ for extreme-event generation. The Extreme Control Network contains three key components: \textit{Semantic Encoder}, \textit{Control Branch}, and \textit{Control Gate}.


\textbf{Semantic Encoder.}
The Semantic Encoder maps the extreme-control signal $\hat{E}_{\mathrm{CA}}$ into semantic features for the Control Branch. It embeds each event's location, intensity, and temporal shape into an event-aware token, and aggregates all tokens with a lightweight set encoder:
\begin{equation}\label{eq:SE_highlevel}
\mathbf{e}_{m}=\mathrm{Emb}(\mathcal{L}^{m},\mathcal{I}^{m},\mathcal{S}^{m}),\quad
\mathbf{C}_{\mathrm{sem}}(t_{\mathrm{CA}})
=\mathrm{SetEnc}(\{\{\mathbf{e}_{m}\}_{m=1}^{n_{\mathrm{ev}}}\}).
\end{equation} 
The resulting representation $\mathbf{C}_{\mathrm{sem}}(t_{\mathrm{CA}})$ provides event-level conditioning for layer-aligned control.

\textbf{Control Branch.}
The Control Branch projects $\mathbf{C}_{\mathrm{sem}}(t_{\mathrm{CA}})$ into layer-wise control features aligned with the Backbone Denoiser. It matches the injection depth of the Backbone Denoiser, producing one control feature per injection layer. Conditioned on the denoising state $\mathbf{x}_n(t_{\mathrm{CA}})$, the time-step embedding $\mathrm{TE}(t_{\mathrm{CA}})$, and the semantic $\mathbf{C}_{\mathrm{sem}}(t_{\mathrm{CA}})$, the Control Branch computes:
\begin{equation}\label{eq:cb_block}
\mathbf{c}_\ell(t_{\mathrm{CA}})
=
\mathrm{CB}_\ell(
\mathbf{c}_{\ell-1}(t_{\mathrm{CA}}),
\mathrm{TE}(t_{\mathrm{CA}}),
\mathbf{C}_{\mathrm{sem}}(t_{\mathrm{CA}})
),
\quad \ell=1,\ldots,L_c ,
\end{equation}
where $\mathbf{c}_0(t_{\mathrm{CA}})$ is initialized from $\mathbf{x}_n(t_{\mathrm{CA}})$, and $L_c$ equals the number of injection layers in the Backbone Denoiser. Here, $\mathbf{c}_\ell(t_{\mathrm{CA}})$ denotes the control feature for the $\ell$-th injection layer, $\mathrm{CB}_\ell$ the corresponding control block, and $\mathrm{TE}(\cdot)$ the time-step embedding function.

\textbf{Control Gate.}
The Control Gate injects layer-aligned control features into the Backbone Denoiser as gated residual offsets. For each injection layer, it computes
\begin{equation}\label{eq:control_gate}
\Delta_\ell(t_{\mathrm{CA}})
=
\sigma\!\left(\smash{\mathbf{W}^{(g)}_\ell}\mathbf{c}_\ell(t_{\mathrm{CA}})\right)
\odot
\left(\smash{\mathbf{W}^{(\Delta)}_\ell}\mathbf{c}_\ell(t_{\mathrm{CA}})\right),
\quad \ell=1,\ldots,L_c .
\end{equation}
Here, $\sigma(\cdot)$ is the sigmoid function, $\odot$ denotes element-wise multiplication, and $\smash{\mathbf{W}^{(g)}_\ell}$ and $\smash{\mathbf{W}^{(\Delta)}_\ell}$ map the control feature to the gate value and residual offset. The offset $\Delta_\ell(t_{\mathrm{CA}})$ is added to the hidden representation $\hat{\mathbf{x}}_\ell(t_{\mathrm{CA}})$ or output prediction $\hat{\mathbf{x}}_0(t_{\mathrm{CA}})$ of the Backbone Denoiser, enabling adaptive extreme-event enhancement while preserving background temporal patterns.

Together, these three components transform $\smash{\hat{E}_{\mathrm{CA}}}$ into layer-specific residual offsets, with each offset injected into the output representation of its corresponding layer in the Backbone Denoiser.
More design details will be provided in Appendix~\ref{Appendix:E-Control design}.

\section{EVALUATION} \label{sec:evaluation}

\textbf{Research Questions (RQs).} 
Our evaluation is organized around nine key research questions.
RQ1: Does E4GEN faithfully generate extreme events?
RQ2: Does E4GEN preserve overall generation fidelity?
RQ3: Are the data generated by E4GEN useful for downstream tasks?
RQ4: Can E4GEN provide interpretable insights into the denoising process?
RQ5: How accurately can E4GEN predict extreme-event semantics?
RQ6: How does each key module contribute to the performance of E4GEN?
RQ7: How sensitive is E4GEN to key hyperparameters?
RQ8: Is E4GEN computationally efficient?
and RQ9: Can E4GEN controllably generate specified extreme events?

\textbf{Datasets.} We evaluate E4GEN on six datasets (five real and one synthetic) across four domains: climate, healthcare, energy, and transportation. The real-world datasets include \href{https://docs.deweydata.io/docs/custom-weather-weather-data}{\textsc{Wea-Temp}} and \href{https://docs.deweydata.io/docs/custom-weather-weather-data}{\textsc{Wea-Prec}}, \href{https://www.physionet.org/physiobank/database/ltstdb/}{\textsc{LTST-ECG}}, \href{https://archive.ics.uci.edu/dataset/235/individual+household+electric+power+consumption}{\textsc{HH-Power}}, and \href{https://archive.ics.uci.edu/ml/datasets/pems-sf}{\textsc{PEMS-SF}}. These datasets cover representative extreme events, including cold temperature, heavy rainfall, ST-segment elevation, high electricity consumption, and traffic congestion. Detailed descriptions are provided in Appendix~\ref{Appendix:Dataset}.



\textbf{Metrics.} We evaluate generation performance using 17 metrics from \textit{overall generation quality} and \textit{extreme-event generation quality} perspectives. For overall generation, we report five fidelity metrics (i.e., Wass.~\cite{villani2009optimal}, KS~\cite{smirnov1939estimation}, JS~\cite{lin2002divergence}, MMD~\cite{gretton2012kernel}, and ACD~\cite{mirylenka2017data}) following common practice~\cite{ang2023tsgbench}, and three general downstream utility metrics (i.e., Context-FID~\cite{jeha2022psa} for representation learning, Pred~\cite{yoon2019time} for prediction, and Recon~\cite{malhotra2016lstm} for reconstruction). For extreme-event generation, we introduce seven event-level fidelity metrics, including EMean-W1, ECount-Diff, EDur-W1, EPeak-W1, EWass, EJS, and EMMD, along with two extreme-oriented downstream utility metrics, i.e., EPred for extreme-event prediction and ERecon for anomaly detection. More details are in Appendix~\ref{Appendix:Metric}.

\textbf{Baselines.} We compare E4GEN with nine representative generative baselines across six paradigms: TimeGAN~\cite{yoon2019time} (GAN), TimeVAE~\cite{desai2021timevae} and koVAE~\cite{naimangenerative} (VAE), F-Flow~\cite{alaa2021generative} (flow), DiffWave~\cite{kongdiffwave} and Diffusion-TS~\cite{yuan2024diffusionts} (diffusion), SDForger~\cite{rousseau2025forging} (LLM), and two extreme-aware methods, FIDE~\cite{galib2024fide} and HeavyDiff~\cite{pandey2025heavy}. Detailed baseline descriptions are provided in Appendix~\ref{Appendix:Baseline}.

\subsection{Generation Fidelity and Downstream Utility}
\begin{table*}[!h]
\centering
\footnotesize
\renewcommand{\arraystretch}{0.92}
\setlength{\tabcolsep}{2pt}

\caption{Overall Generation Fidelity and General Downstream Utility}
\label{tab:overall_wea_temp}
\begin{tabular}{cclcccccccc}
\toprule
\textbf{Dataset} & \textbf{Type} & \textbf{Method} & \textbf{Wass.} & \textbf{KS} & \textbf{JS} & \textbf{MMD} & \textbf{ACD} & \textbf{Context\_FID} & \textbf{Pred} & \textbf{Recon} \\
\midrule
\multirow{11}{*}{\makecell{\textbf{WEA-}\\\textbf{TEMP}}}
& \textbf{GAN}
& TimeGAN (NeurIPS'19)~\cite{yoon2019time}
& 1.1950 & \underline{0.0947} & 0.1576 & 0.3650 & 0.8405 & 5.1615 & 0.1662 & 0.1165 \\

\arrayrulecolor{gray!60}\cmidrule(lr){2-11}\arrayrulecolor{black}

& \multirow{2}{*}{\textbf{VAE}}
& TimeVAE (ArXiv'21)~\cite{desai2021timevae}
& 1.8238 & 0.1265 & 0.1385 & 0.1781 & 1.0735 & 1.2312 & 0.1202 & 0.0326 \\

&
& koVAE (ICLR'24)~\cite{naimangenerative}
& 3.5138 & 0.2662 & 0.3984 & 0.3388 & 1.8261 & 2.9273 & 0.1851 & 0.0417 \\

\arrayrulecolor{gray!60}\cmidrule(lr){2-11}\arrayrulecolor{black}

& \textbf{Flow}
& F-Flow (ICLR'21)~\cite{alaa2021generative}
& 21.002 & 0.9955 & 0.8256 & 1.2082 & 10.519 & 32.013 & 1.5189 & 1.4458 \\

\arrayrulecolor{gray!60}\cmidrule(lr){2-11}\arrayrulecolor{black}

& \multirow{2}{*}{\textbf{Diffusion}}
& DiffWave (ICLR'21)~\cite{kongdiffwave}
& 3.9978 & 0.3056 & 0.4726 & 0.3808 & 2.0716 & 3.4627 & 0.1850 & 0.0303 \\

&
& Diffusion-TS (ICLR'24)~\cite{yuan2024diffusionts}
& 1.2845 & 0.1033 & 0.0916 & 0.1116 & 0.6605 & 0.2271 & \underline{0.1116} & 0.0283 \\

\arrayrulecolor{gray!60}\cmidrule(lr){2-11}\arrayrulecolor{black}

& \textbf{LLM}
& SDForger (NeurIPS'25)~\cite{rousseau2025forging}
& 1.5069 & 0.1076 & 0.1551 & 0.1317 & 0.8370 & 1.5304 & 0.1308 & 0.0539 \\

\arrayrulecolor{gray!60}\cmidrule(lr){2-11}\arrayrulecolor{black}

& \multirow{2}{*}{\makecell{\textbf{Extreme-}\\\textbf{Aware}}}
& FIDE (NeurIPS'24)~\cite{galib2024fide}
& 22.225 & 0.7448 & 0.6314 & 0.7962 & 11.290 & 24.775 & 0.3035 & 0.1573 \\

&
& HeavyDiff (ICLR'25)~\cite{pandey2025heavy}
& \underline{1.0998} & 0.0950 & \underline{0.0862} & \underline{0.0995} & \underline{0.5793}
& \underline{0.1986} & 0.1185 & \textbf{0.0257} \\

\arrayrulecolor{gray!60}\cmidrule(lr){2-11}\arrayrulecolor{black}

& \textbf{Ours}
& \textbf{E4GEN}
& \textbf{0.9346} & \textbf{0.0846} & \textbf{0.0851} & \textbf{0.0858} & \textbf{0.4938}
& \textbf{0.1584} & \textbf{0.1084} & \underline{0.0276} \\
\bottomrule
\end{tabular}

\vspace{2.0em}

\caption{Extreme-Event Generation Fidelity and Extreme-Oriented Downstream Utility}
\label{tab:extreme_wea_temp}
\begin{tabular}{cclccccccccc}
\toprule
\textbf{Dataset} & \textbf{Type} & \textbf{Method} & \textbf{EM-W1} & \textbf{EC-Diff} & \textbf{ED-W1} & \textbf{EP-W1} & \textbf{EWass.} & \textbf{EJS} & \textbf{EMMD} & \textbf{EPred} & \textbf{ERecon} \\
\midrule
\multirow{11}{*}{\makecell{\textbf{WEA-}\\\textbf{TEMP}}}
& \textbf{GAN}
& TimeGAN ~\cite{yoon2019time}
& 3.7958 & 5583 & 2.0428 & 3.8054 & 10.961 & 0.6318 & 10.893 & 0.3093 & 0.1461 \\

\arrayrulecolor{gray!60}\cmidrule(lr){2-12}\arrayrulecolor{black}

& \multirow{2}{*}{\textbf{VAE}}
& TimeVAE ~\cite{desai2021timevae}
& 1.3089 & 4864 & 1.6862 & 1.1608 & 10.600 & 0.5942 & 17.539 & 0.1629 & 0.0615 \\

&
& koVAE ~\cite{naimangenerative}
& 1.8624 & 11425 & 2.3364 & 1.5521 & $\infty$ & $\infty$ & $\infty$ & $\infty$ & $\infty$ \\

\arrayrulecolor{gray!60}\cmidrule(lr){2-12}\arrayrulecolor{black}

& \textbf{Flow}
& F-Flow ~\cite{alaa2021generative}
& 1.8883 & 20822 & 14.176 & 1.6071 & 22.109 & 0.8305 & \underline{5.0478} & 0.6524 & 0.5295 \\

\arrayrulecolor{gray!60}\cmidrule(lr){2-12}\arrayrulecolor{black}

& \multirow{2}{*}{\textbf{Diffusion}}
& DiffWave ~\cite{kongdiffwave}
& 1.7314 & 10654 & 2.2313 & 1.5311 & $\infty$ & $\infty$ & $\infty$ & $\infty$ & $\infty$ \\

&
& Diffusion-TS ~\cite{yuan2024diffusionts}
& 0.4373 & \underline{2587} & \underline{0.4235} & 0.4267 & \underline{4.7470} & 0.4708 & 8.4300 & \underline{0.1163} & 0.0593 \\

\arrayrulecolor{gray!60}\cmidrule(lr){2-12}\arrayrulecolor{black}

& \textbf{LLM}
& SDForger ~\cite{rousseau2025forging}
& \underline{0.3644} & 3242 & 1.0477 & 0.3736 & 9.1781 & \underline{0.4411} & \textbf{4.1654} & 0.1282 & 0.0763 \\

\arrayrulecolor{gray!60}\cmidrule(lr){2-12}\arrayrulecolor{black}

& \multirow{2}{*}{\makecell{\textbf{Extreme-}\\\textbf{Aware}}}
& FIDE ~\cite{galib2024fide}
& 5.1486 & 15767 & 5.8725 & 12.335 & 64.284 & 0.7782 & 23.947 & 0.2452 & 0.6922 \\

&
& HeavyDiff ~\cite{pandey2025heavy}
& 0.4115 & 3115 & 0.6909 & \underline{0.3675} & 5.8522 & 0.4909 & 6.9037 & 0.1233 & \underline{0.0579} \\

\arrayrulecolor{gray!60}\cmidrule(lr){2-12}\arrayrulecolor{black}

& \textbf{Ours}
& \textbf{E4GEN}
& \textbf{0.3574} & \textbf{1090} & \textbf{0.2916} & \textbf{0.3594} & \textbf{2.2081} & \textbf{0.4216} & 6.1463 & \textbf{0.1158} & \textbf{0.0522} \\
\bottomrule
\end{tabular}

\renewcommand{\arraystretch}{1.0}
\normalsize
\end{table*}

To answer RQs~1--3, we evaluate E4GEN from three perspectives: extreme-event generation fidelity, overall generation fidelity, and downstream utility. Specifically, downstream utility is assessed under the train-on-generated, test-on-original protocol, where models are trained on generated data and evaluated on real data. Across 102 dataset--metric comparisons, E4GEN achieves state-of-the-art performance in 79.4\% of cases and ranks among the top two methods in 95.1\%, demonstrating strong and consistent advantages. Below, we present the results on the WEA-TEMP dataset as a representative case study, while additional experimental results are provided in Appendix~\ref{Appendix:more_experiment_results}.

As shown in Tables~\ref{tab:overall_wea_temp} and~\ref{tab:extreme_wea_temp}, E4GEN achieves the best performance on most metrics across overall fidelity, extreme-event fidelity, and downstream utility. In the tables, $\infty$ indicates that the corresponding metric cannot be computed because no extreme-event time points are available.
For overall fidelity, E4GEN achieves the strongest performance on the distributional discrepancy metrics (Wass., KS, JS, and MMD) as well as the temporal-dependency metric (ACD). For extreme-event fidelity, E4GEN also outperforms competing methods on most metrics, with particularly notable gains on ED-W1 and EWass., demonstrating superior modeling of extreme-event duration distributions and temporal-pattern distributions. These improvements mainly stem from the proposed extreme-event semantic design, which explicitly captures event location, intensity, and temporal shape. For downstream utility, E4GEN achieves strong Context-FID performance for representation learning and the best Pred score for prediction tasks. Notably, its leading EPred and ERecon scores further demonstrate the effectiveness of E4GEN-generated data for extreme-oriented downstream applications, such as extreme-event prediction and anomaly detection. Among the baselines, diffusion-based methods (e.g., Diffusion-TS~\cite{yuan2024diffusionts} and HeavyDiff~\cite{pandey2025heavy}) and the LLM-based SDForger~\cite{rousseau2025forging} remain competitive, showing their capability in global temporal generation and fine-grained detail preservation.

\begin{figure}[H]
    \centering
    \begin{minipage}{0.49\linewidth}
        \centering
        \includegraphics[width=\linewidth]{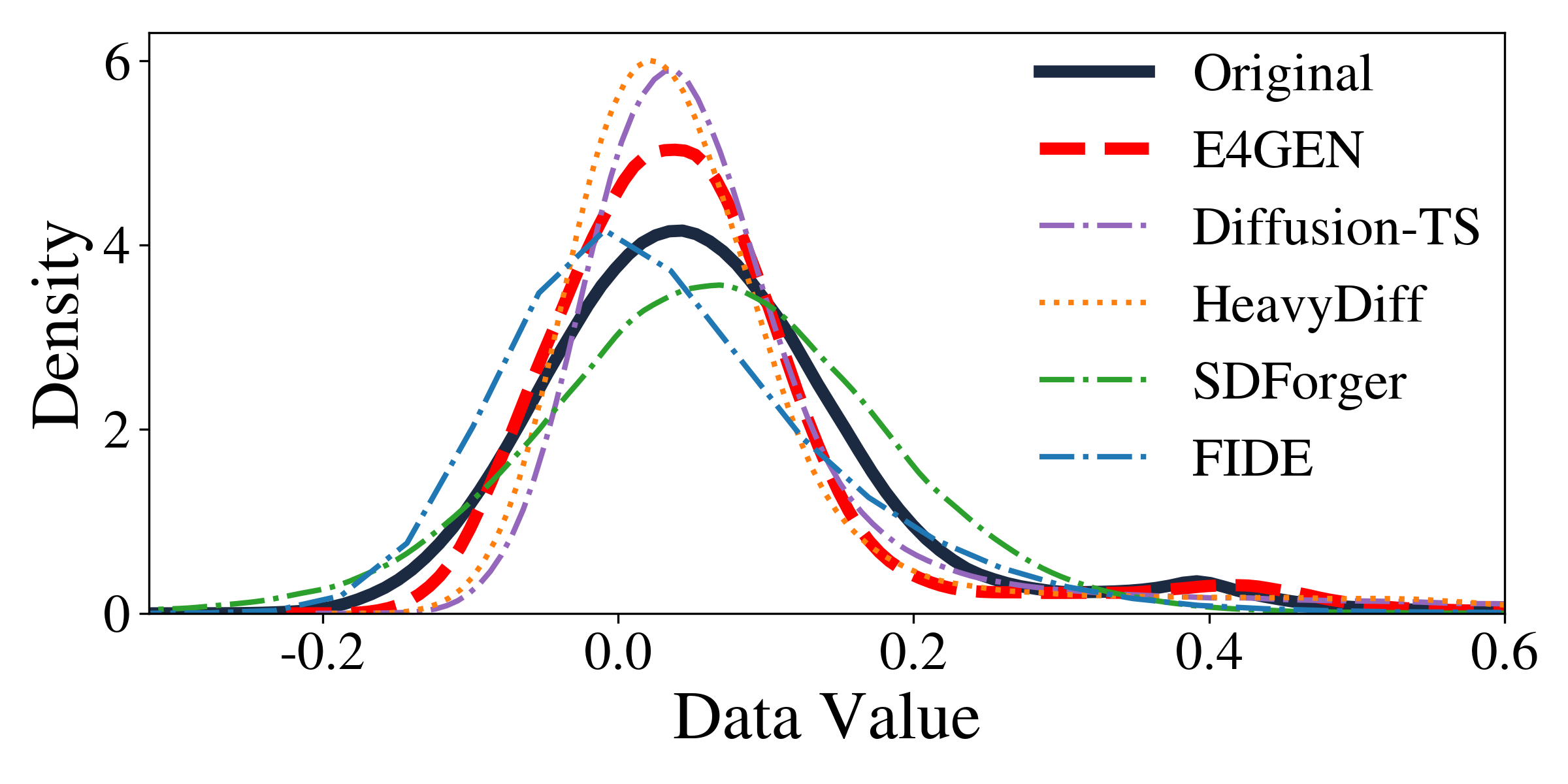} 
        \caption{Overall distribution comparison.}
        \label{fig:overall_distribution}
    \end{minipage}
    \hfill
    \begin{minipage}{0.49\linewidth}
        \centering
        \includegraphics[width=\linewidth]{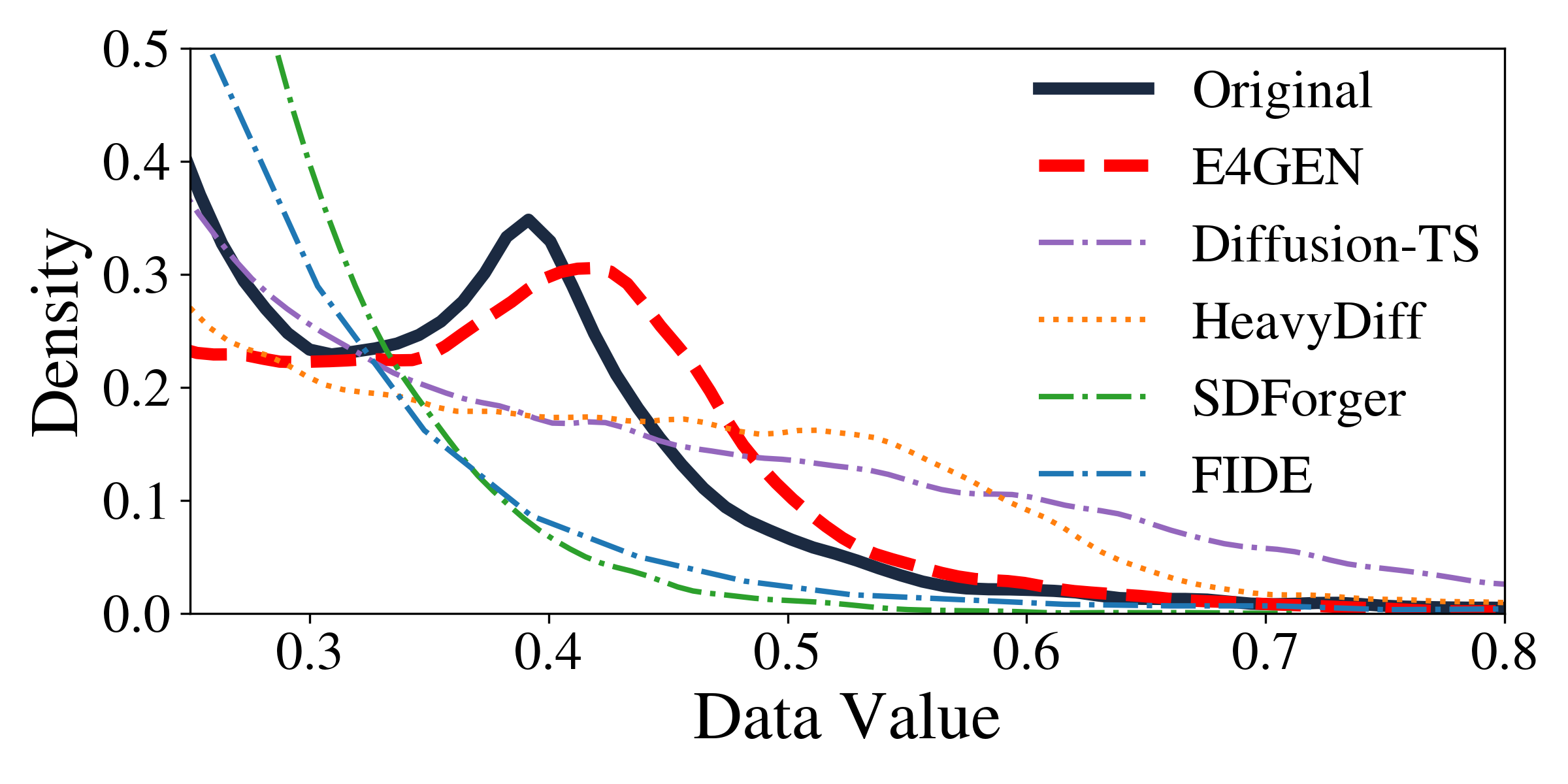} 
        \caption{Tail distribution comparison.}
        \label{fig:extreme_distribution}
    \end{minipage}
\end{figure}

As shown in Fig.~\ref{fig:overall_distribution}, all methods generally capture the overall distribution, although with varying levels of deviation. However, in Fig.~\ref{fig:extreme_distribution}, only E4GEN (red) closely aligns with the original data distribution (black) in the extreme-value region, while the other methods largely fail to reproduce this characteristic pattern. These results suggest that existing methods struggle to capture localized extreme-event structures, whereas E4GEN more faithfully recovers their distributional characteristics through explicit extreme-event semantic modeling.



\subsection{Further Analysis} \vspace{-5pt}
To answer RQ4, we provide detailed results of \textbf{Interpretable Generation} in Appendix~\ref{Appendix:Interpretable_Generation_Dynamics}.
To answer RQs~5--6, we show detailed results on \textbf{Effectiveness Analysis and Ablation Study} in Appendices~\ref{Appendix:E_Predictor_Intermediate_Output_Analysis} and~\ref{Appendix:ablation}
To answer RQ7, we provide detailed \textbf{Sensitivity Analysis} in Appendix~\ref{Appendix:Sensitivity Analysis}.
To answer RQ8, we show the \textbf{Computational Efficiency} in Appendix~\ref{Appendix:efficiency}.
To answer RQ9, we show the \textbf{Controllability} of E4GEN for Extreme-Event Generation in Appendix~\ref{Appendix:Controllable Extreme-Event Generation}.

\vspace{-10pt}
\section{CONCLUSION} \label{sec:conclusion}
\vspace{-10pt}
In this paper, we propose \textbf{E4GEN}, an explainable diffusion framework for extreme-aware time-series generation. To address the limitation that existing methods often fail to faithfully capture extreme events, we formulate extreme-event generation as an interpretable control problem during denoising, centered on when, what, and how to guide the generation process. E4GEN introduces three key components: (i) \emph{E-Activator}, which identifies when to activate extreme-control signals; (ii) \emph{E-Predictor}, which determines what semantic-based extreme-control signals to enforce in the denoising process; and (iii) \emph{E-Control}, which specifies how to inject these signals to enhance extreme-event generation while preserving global temporal patterns. Extensive experiments on six datasets demonstrate that E4GEN outperforms baselines across 17 metrics spanning overall generation fidelity, extreme-event generation fidelity, and downstream utility.


\bibliographystyle{unsrt}
\bibliography{references}

\appendix


\newpage

\newpage

\section{Reproducibility Statement}

To improve reproducibility, we release an anonymous GitHub repository at
\href{https://anonymous.4open.science/r/E4GEN}{https://anonymous.4open.science/r/E4GEN}
as part of the supplementary materials during the review process. The repository includes all
datasets in the \texttt{Data} folder, the implementations of all baseline methods in the
\texttt{Baseline} folder, the parameter configuration files for all experiments in the \texttt{Config}
folder, and the source files for generating all figures in the \texttt{Analysis} folder.
Together with the scripts for data preparation and model training, these resources enable
independent researchers to reproduce all reported results and figures.

\section{Related Work}
\label{Appendix:related work}
\subsection{Time Series Generation}

\subsubsection{GAN-, VAE-, and Flow-based Time-Series Generation}
Time-series generation aims to synthesize realistic sequential data that preserve both marginal distributions and temporal dependencies. Early studies adapted general-purpose generative paradigms to sequential settings, including generative adversarial networks (GANs), variational autoencoding (VAE), and likelihood-based flow modeling. Among them, GANs learn realistic sequences through a competition between the generator and discriminator. For instance, C-RNN-GAN~\cite{Mogren2016} and RCGAN~\cite{Esteban2017} pioneer the combination of recurrent neural networks with adversarial training to generate continuous temporal data and condition on auxiliary information, respectively. TimeGAN~\cite{yoon2019time} further introduces a supervised temporal objective in a learned embedding space, encouraging generated sequences to follow the stepwise dynamics of real data. 
In contrast, VAE-based methods provide a more stable probabilistic framework by learning compact latent representations for sequence reconstruction and sampling. Early approaches like VRNN~\cite{Chung2015} integrate latent random variables into RNN hidden states to model structured sequential variability. TimeVAE~\cite{desai2021timevae} extends VAEs to multivariate time series with sequence-aware components such as level, trend, and residual structures, while koVAE~\cite{naimangenerative} replaces the conventional static latent prior with a Koopman-inspired dynamical prior to capture latent temporal evolution. 
Flow-based models provide another probabilistic approach to time-series generation through explicit likelihood estimation. Fourier Flows~\cite{alaa2021generative} transform time series into spectral representations and learn an invertible flow over Fourier-domain coefficients, enabling the model to capture global temporal and periodic patterns. While these methods provide strong general-purpose generators, their objectives are usually dominated by common temporal patterns and are not explicitly designed to preserve sparse extreme patterns.

\subsubsection{Diffusion-based Time-series Generation}
Diffusion models have recently become a strong paradigm for sequential generation because they synthesize samples through iterative denoising, allowing complex temporal distributions to be refined progressively~\cite{xu2026synhat,xu2026geogen}. Their general formulation, established by DDPMs and subsequent improvements, provides a flexible probabilistic framework for transforming noise into structured data distributions~\cite{ho2020denoising,nichol2021improved}. In one-dimensional sequential domains, DiffWave~\cite{Kong2020} demonstrates that diffusion backbones can synthesize high-fidelity audio waveforms from white noise, suggesting their suitability for continuous temporal signals.

Building on this denoising perspective, recent studies have adapted diffusion models~\cite{xu2025autostdiff} to time-series tasks under different settings. TimeGrad~\cite{Rasul2021} combines autoregressive networks with score-based diffusion for multivariate probabilistic forecasting, while CSDI~\cite{Tashiro2021} extends score-based diffusion to probabilistic time-series imputation through self-supervised masking and conditional score matching. Other works further explore controlled or structured diffusion for time-series generation. For example, constrained time-series generation~\cite{coletta2023constrained} incorporates user-specified temporal constraints into the sampling process, while TimeBridge~\cite{park2026timebridge} improves diffusion prior design through bridge-based modeling. More recently, Diffusion-TS~\cite{Yuan2024} combines a Transformer denoiser with trend--seasonality decomposition and directly reconstructs the clean sample during diffusion, improving both fidelity and interpretability.

Overall, these methods establish diffusion models as effective backbones for time-series generation. However, their conditioning mechanisms are primarily designed for sample-level or value-level guidance rather than event-level control, making it difficult to explicitly preserve rare and localized extreme-event structures.

\subsubsection{LLM-based Time-series Generation}
Recent studies have also explored whether the sequence modeling capacity of large language models and foundation models can be transferred to time-series forecasting and generation. Since raw time series are continuous, these methods usually convert numerical sequences into textual tokens, discrete value tokens, or patch-level embeddings before applying autoregressive modeling. Early prompt-based studies such as PromptCast~\cite{xue2023promptcast} formulate forecasting as a text-to-text problem, while LLMTime~\cite{gruver2023large} serializes numerical values as text and samples future continuations from pretrained LLMs in a zero-shot manner.
Another line of work reuses pretrained language backbones through representation alignment. GPT4TS~\cite{zhou2023one} segments time series into patches and projects them into the input space of a pretrained GPT-style model, enabling the language backbone to process temporal sequences with lightweight adaptation. Time-LLM~\cite{jin2024time} further introduces a reprogramming strategy that maps time-series patch representations into text-prototype-aligned token embeddings, together with prompt-based temporal context for frozen LLM adaptation. In contrast, Chronos~\cite{ansari2024chronos} trains language-model architectures directly on discretized time-series values, turning forecasting into next-token prediction over value bins. Recent analyses also suggest that LLM-based temporal models can benefit from pretraining and sequence modeling, but their gains depend heavily on representation design, prompting, and evaluation settings~\cite{tan2024language}. However, most of these works focus on general forecasting problems instead of generation tasks, which are more challenging, especially when considering the extreme events.

For time-series generation, SDForger~\cite{rousseau2025forging} follows a more generation-oriented paradigm by mapping time series into functional or tabular embeddings, serializing them into text, fine-tuning an autoregressive language model, and decoding the generated embeddings back into time-series samples. This representation-space generation avoids direct token-by-token synthesis over long continuous sequences while leveraging the generative capacity of LLMs. Nevertheless, compact tokens or embeddings are usually optimized to preserve global sequence information, so sparse localized extremes may still be compressed or underemphasized.

\subsection{Extreme-aware Generation}

Extreme-aware generation aims to improve the generation fidelity in rare, high-impact parts that are often underrepresented by average-case objectives. Existing approaches usually address this issue through long-tail sample modeling, extreme-value theory, or heavy-tailed generation dynamics. Long-tailed diffusion methods improve tail coverage by calibrating or transferring knowledge from frequent samples to rare ones~\cite{zhang2024long}, while heavy-tailed generative models modify distributional assumptions to better represent large-magnitude outcomes~\cite{kim2024t3vae,pandey2025heavy,shariatian2025dlpm}. These designs improve tail fidelity from sample-level or value-level perspectives. However, for time-series data, extremes are not only distributional tails but also temporally organized patterns, which makes their preservation different from conventional long-tail generation.

Several studies have incorporated extreme-value modeling into time-series generators. Early explorations like ExGAN~\cite{bhatia2021exgan} combine Extreme Value Theory (EVT) with adversarial learning to generate realistic samples under user-specified tail probabilities. 
FIDE~\cite{galib2024fide} enhances high-frequency components and conditions the generation process on block maxima sampled from a fitted GEV distribution. HeavyDiff~\cite{pandey2025heavy} takes a complementary direction by replacing Gaussian perturbations with Student-\(t\)-based heavy-tailed noise, allowing diffusion models to better capture rare large-magnitude outcomes.

These efforts show that rare and extreme regions can be emphasized by reshaping the sampling objective, conditioning on tail statistics, or modifying the generative noise distribution. Yet most of them still define extremes through sample rarity, peak magnitude, or tail behavior, leaving event-level properties such as location, duration, and temporal shape less directly modeled.

\section{Discussion}
\label{Appendix:limitations}
Although E4GEN shows consistent improvements across diverse time-series datasets, several limitations remain. First, our experiments focus on six representative datasets from climate, healthcare, energy, and transportation domains. While these datasets cover different types of extreme events, broader validation on additional domains and longer time-series settings would further strengthen the generality of the conclusions. Second, the definition of extreme events relies on domain-specific thresholds or task-specific criteria. This design provides flexibility, but the generated extreme-event semantics may depend on the quality of the chosen threshold. Third, E4GEN introduces additional components for semantic prediction and control injection, which may increase training and sampling cost compared with a plain diffusion backbone. Future work can explore more efficient control modules and adaptive threshold selection strategies.

\section{Broader Societal Impacts}
\label{sec:broader_impacts}
Our E4GEN advances the generation of realistic time-series data with a focus on capturing extreme events, offering positive societal impacts in domains like climate disaster preparedness, healthcare (e.g., synthesizing rare ECG anomalies without compromising patient privacy), and energy grid stress-testing. However, we also acknowledge potential negative impacts, including the malicious synthesis of fake sensor data to spoof systems or trigger false alarms, and the risk of over-reliance on synthetic extremes that may not perfectly reflect true underlying physical mechanisms, potentially leading to miscalibrated decisions in safety-critical scenarios. To mitigate these risks, E4GEN-generated data should complement rather than replace real-world observations, and future deployments could incorporate statistical watermarking to prevent the misuse of synthetic datasets.

\section{Data-Driven Analysis} 

\subsection{Discussion About Value-level Extreme-aware Generation} \label{Appendix:Discussion about value-level enhencement}

\begin{figure*}[htbp]
    \centering
    \begin{subfigure}[htbp]{0.42\textwidth}
        \centering
        \includegraphics[width=\linewidth]{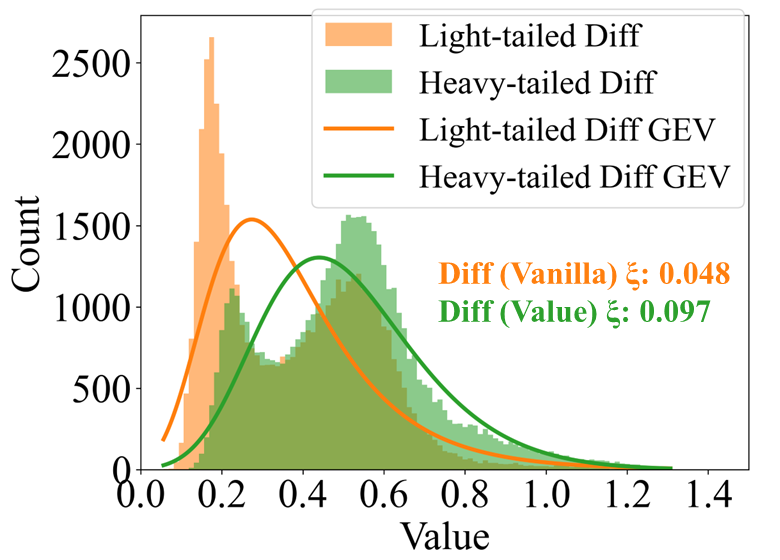}
        \label{fig:synthetic_data_pdf}
    \end{subfigure}
    \hfill
    \begin{subfigure}[htbp]{0.42\textwidth}
        \centering
        \includegraphics[width=\linewidth]{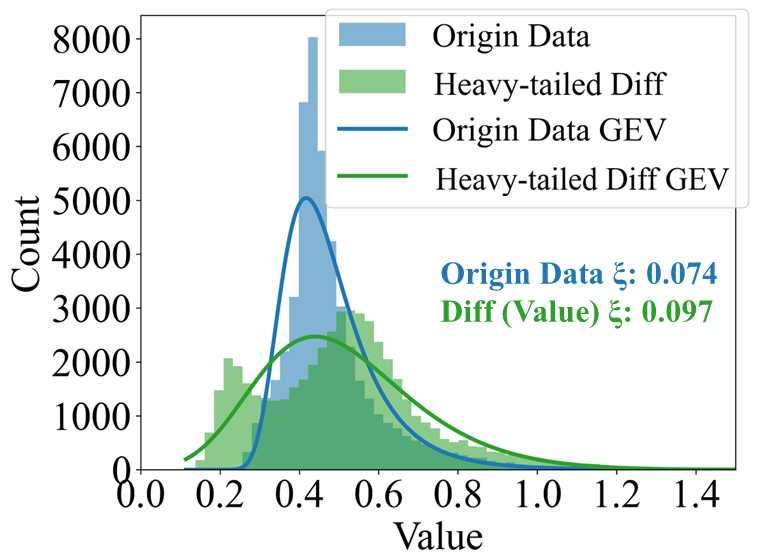}
        \label{fig:synthetic_data_sample}
    \end{subfigure}
    \vspace{-10pt}
    \caption{GEV Distribution of Block Maxima}
    \label{fig:GEV}
\end{figure*}

\vspace{-10pt}

\begin{figure*}[htbp]
    \centering

    \begin{subfigure}[t]{0.32\textwidth}
        \centering
        \includegraphics[width=\linewidth]{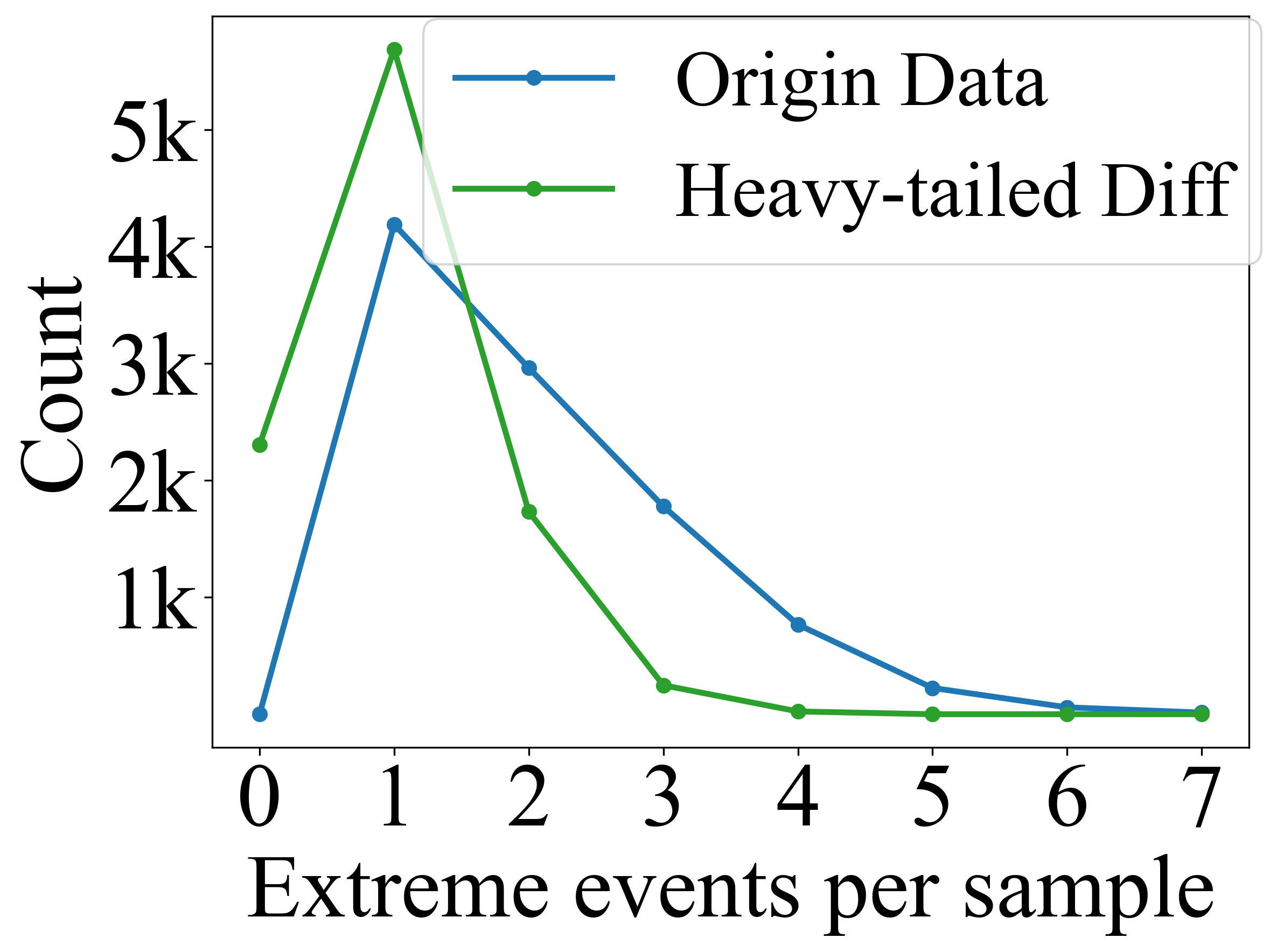}
        \caption{Count}
        \label{fig:extreme_count}
    \end{subfigure}
    \qquad
    \begin{subfigure}[t]{0.32\textwidth}
        \centering
        \includegraphics[width=\linewidth]{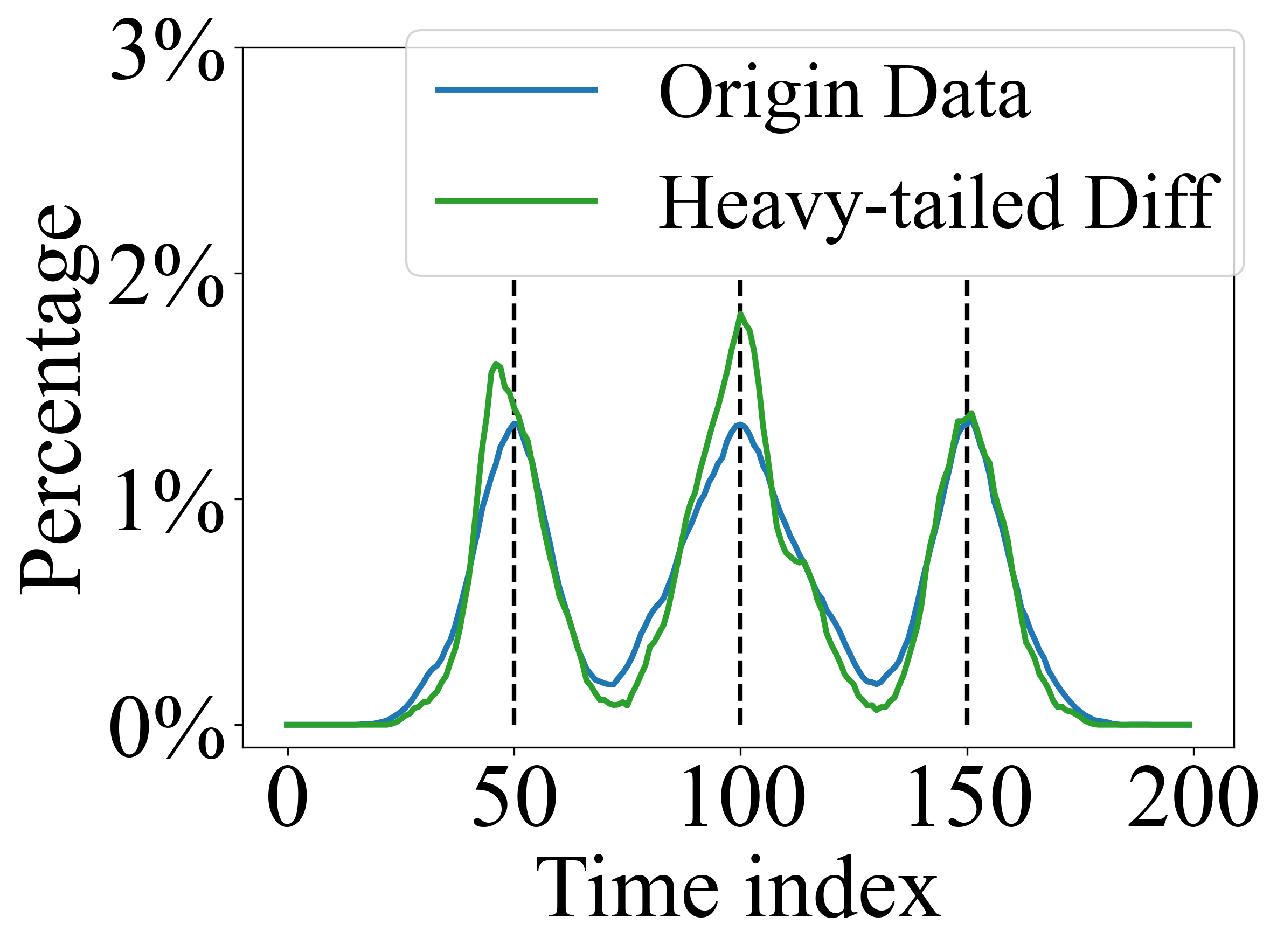}
        \caption{Location}
        \label{fig:extreme_location}
    \end{subfigure}

    \vspace{0.4em}

    \begin{subfigure}[t]{0.32\textwidth}
        \centering
        \includegraphics[width=\linewidth]{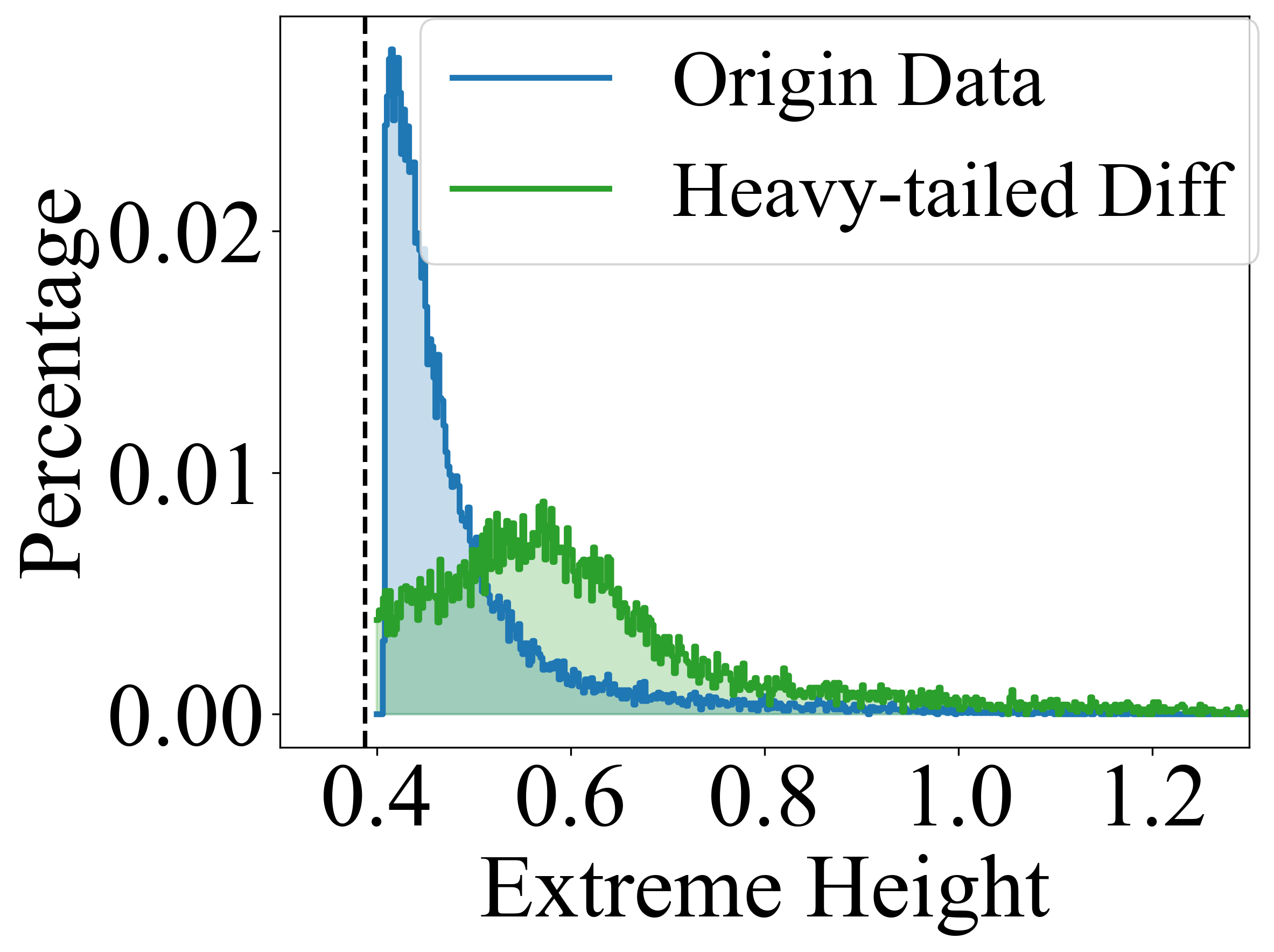}
        \caption{Height}
        \label{fig:extreme_height}
    \end{subfigure}
    \hfill
    \begin{subfigure}[t]{0.32\textwidth}
        \centering
        \includegraphics[width=\linewidth]{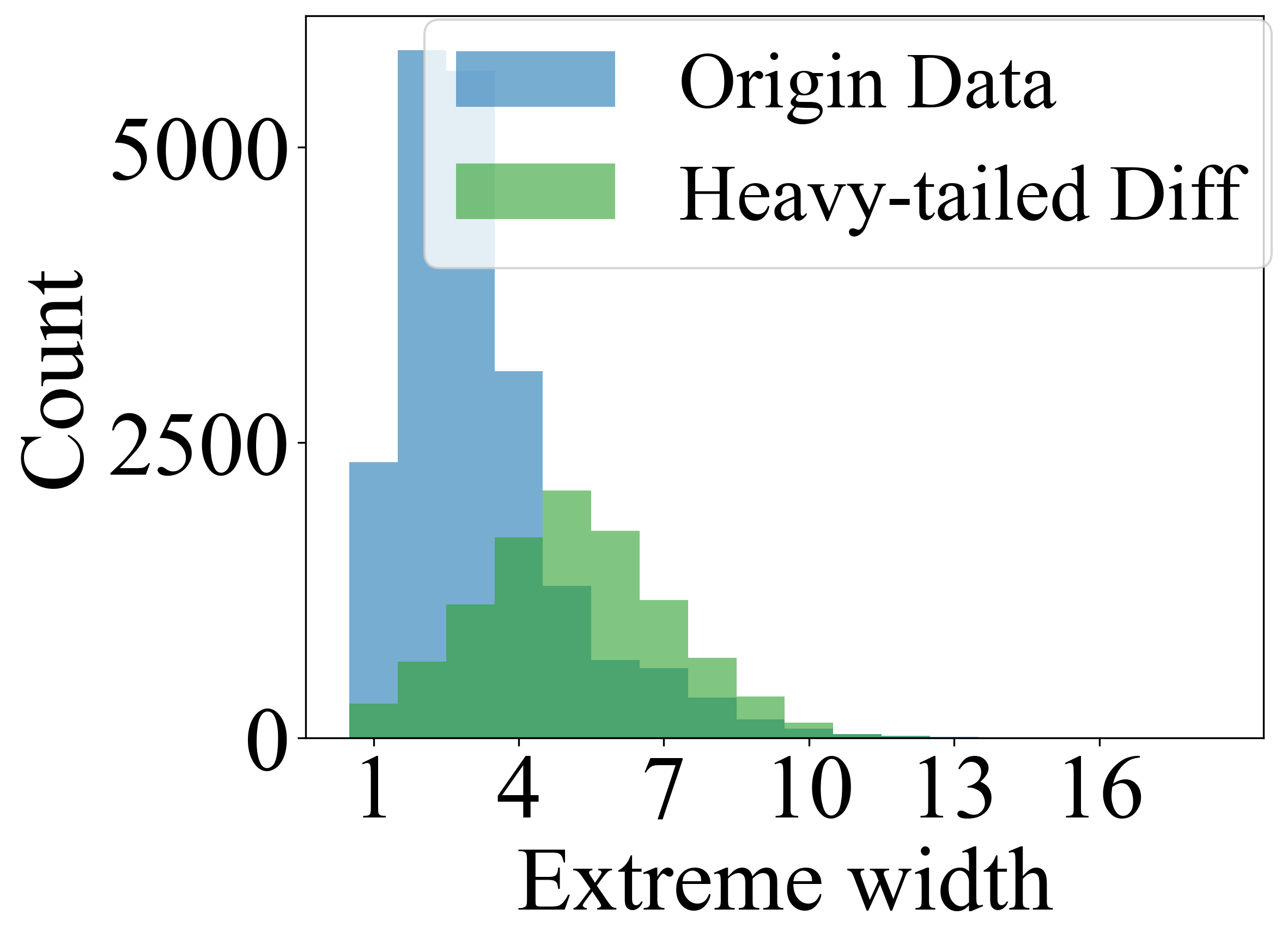}
        \caption{Width}
        \label{fig:extreme_width}
    \end{subfigure}
    \hfill
    \begin{subfigure}[t]{0.32\textwidth}
        \centering
        \includegraphics[width=\linewidth]{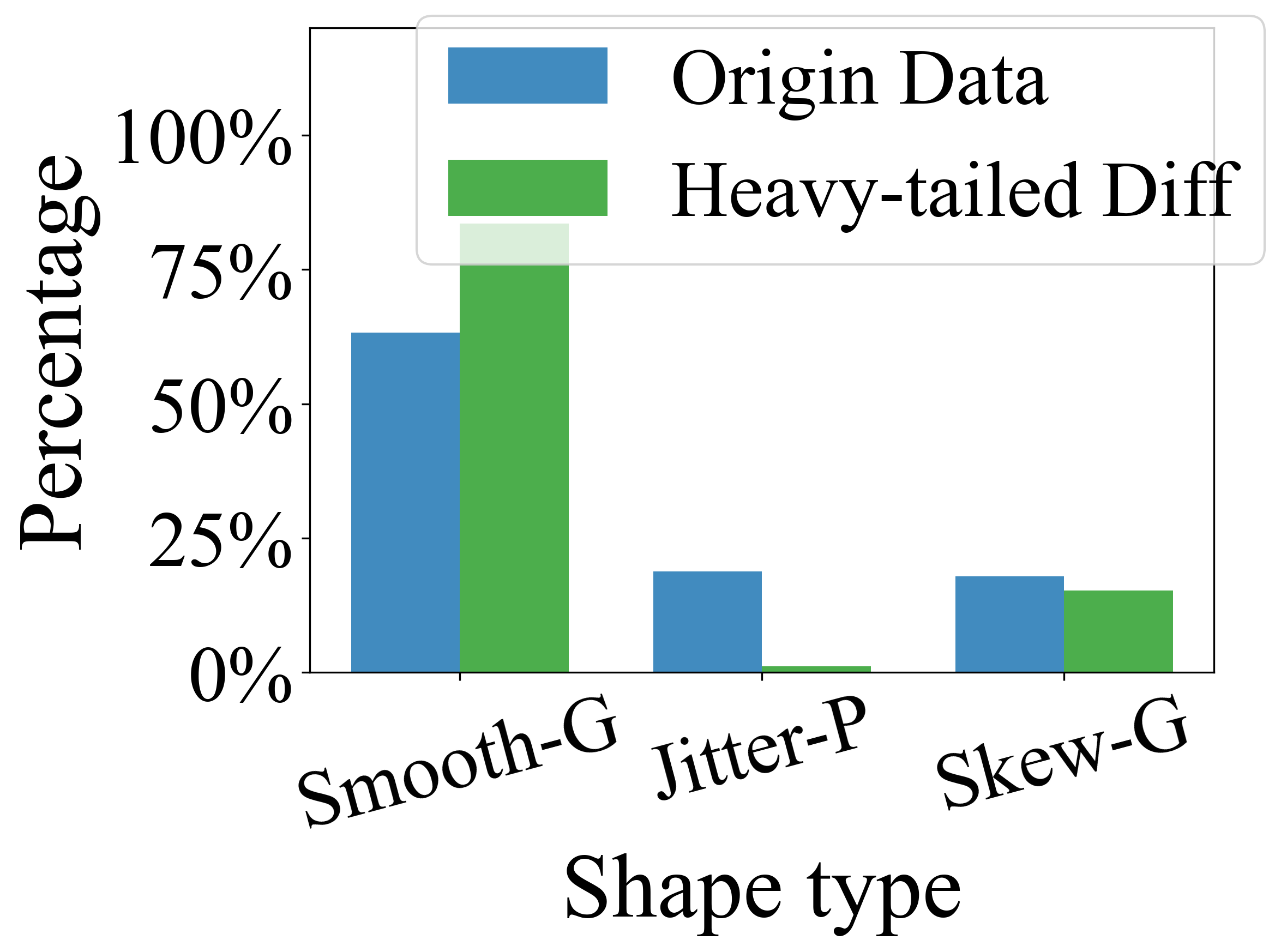}
        \caption{Shape}
        \label{fig:extreme_shape}
    \end{subfigure}
    \caption{Evaluation of extreme-event fidelity across five dimensions.}
    \label{fig:extreme_five_dimension}
    \vspace{-10pt}
\end{figure*}

\textbf{Value-level extreme-aware generation} is a common line of research in existing studies on extreme enhancement~\cite{galib2024fide,huster2021pareto,jaini2020tails,kim2024t3vae,pandey2025heavy}, and it is often framed as \textbf{heavy-tailed generation}. Its core idea is to directly strengthen tail behavior at the value-distribution level, e.g., by replacing light-tailed Gaussian noise with heavy-tailed alternatives such as Student-$t$ noise~\cite{pandey2025heavy}. 

We investigate the advantages and limitations of value-level extreme-aware generation from the perspective of extreme value theory~\cite{de2006extreme}. We use the classical block maxima framework as an example and analyze its corresponding generalized extreme value (GEV) distribution, formulated as:
\begin{equation}
G(z;\mu,\sigma,\xi)
=
\exp\left\{
-\left[1+\xi\left(\frac{z-\mu}{\sigma}\right)\right]^{-1/\xi}
\right\},
\end{equation}
where $\mu \in \mathbb{R}$ is the location parameter, $\sigma > 0$ is the scale parameter, and $\xi \in \mathbb{R}$ is the shape parameter. In general, a larger $\xi$ indicates a heavier tail. Similar to Observation I in Section~\ref{sec:analysis}, we visualize two representative methods, a standard light-tailed Diffusion model~\cite{} implemented with Diffusion-TS and a value-level extreme-enhanced Heavy-tailed Diffusion model implemented with HeavyDiff. Figure~\ref{fig:GEV} presents the overall distribution as the shaded area and the fitted GEV distribution as curves. The left panel compares the generated data of Heavy-tailed Diffusion with Light-tailed Diffusion, while the right panel compares the results of Heavy-tailed Diffusion with the original data. In terms of the shape parameter $\xi$, the generated data of Heavy-tailed Diffusion shows a heavier tail than the original data, whereas the generated data of Light-tailed Diffusion shows a lighter tail. This result indicates that value-level extreme-aware generation can strengthen tail heaviness at the statistical level. However, this raises a further question: does a heavier tail necessarily lead to higher extreme-event fidelity?

\begin{figure*}[htbp]
    \centering
    \includegraphics[width=\textwidth]{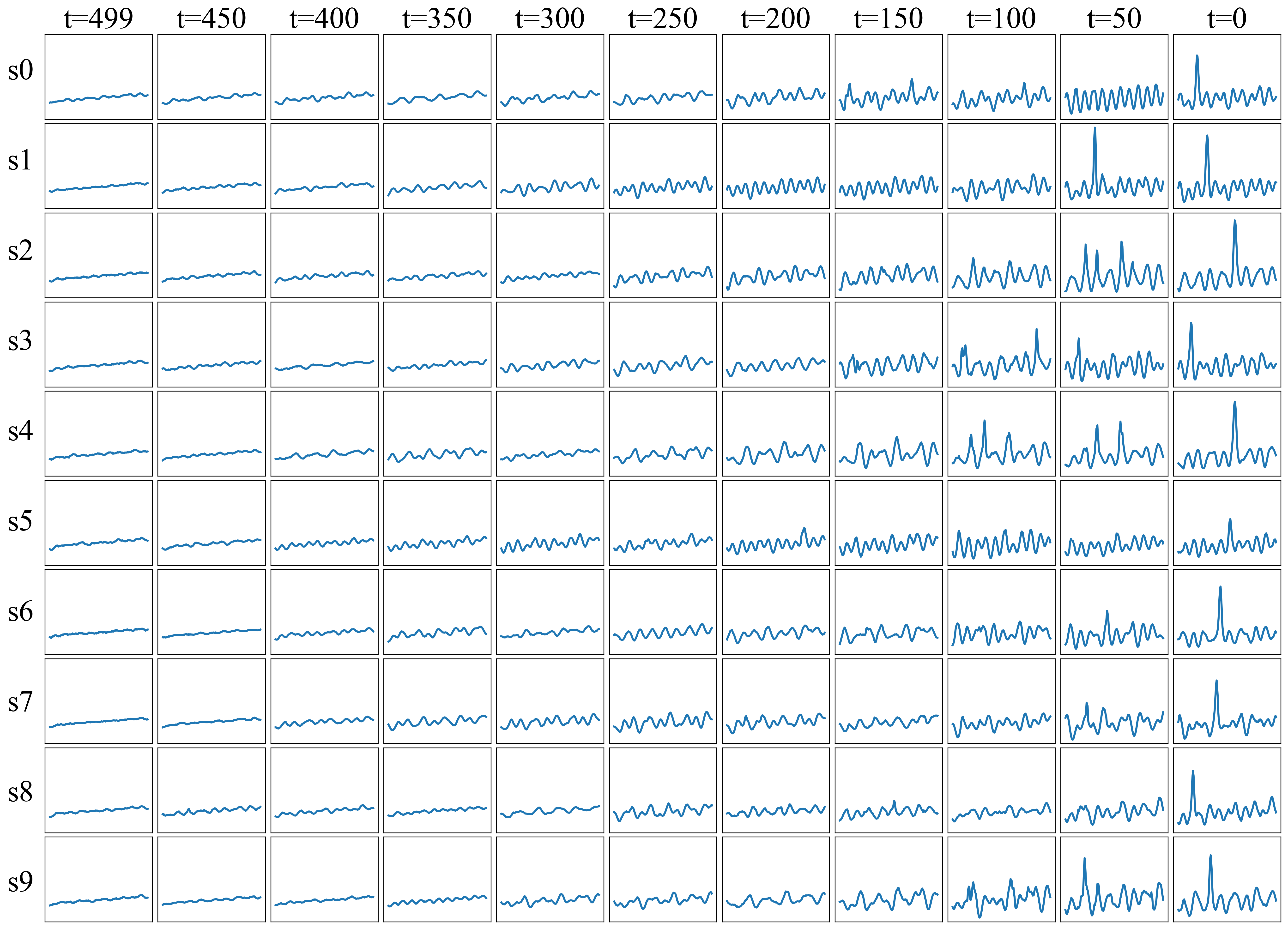}
    \caption{Coarse-to-fine generation trajectories of 10 randomly selected samples from a 500-step diffusion model.}
    \label{fig:full_trajectory_sample}
\end{figure*}

To answer this question, we design a toy example containing extreme events and evaluate it from five dimensions, including count, location, height, width, and shape. These results are shown in Figure~\ref{fig:extreme_five_dimension}. Heavy-tail methods exhibit aggregation effects across all five dimensions, namely a Matthew effect, which increases the proportion of dominant groups while weakening minor ones. This phenomenon is especially clear in count and shape. For example, in Fig.~\ref{fig:extreme_five_dimension}(e) for extreme event shape, the heavy-tailed method generates more Smooth-G events, increasing their proportion from the original 60\% to 80\%, but produces fewer events for the minor types, Jitter-P and Skew-G.

Therefore, the \textbf{advantage} of value-level extreme-aware generation is its ability to increase tail heaviness and produce more extreme values. However, a heavier tail does not necessarily imply higher extreme-event fidelity. Its key \textbf{limitation} is that it enhances generation mainly at the value-distribution level, without modeling the semantic structure of extreme events. As a result, it may amplify dominant patterns while weakening minor ones, thereby limiting its improvement in extreme-event generation fidelity.

\subsection{Additional Analysis of Coarse-to-Fine Generation Dynamics}
\label{Appendix:Additional Visualizations for Coarse-to-fine Generation Process}

In Observation III of Section~\ref{sec:analysis}, we discuss the coarse-to-fine generation process exhibited by diffusion models in time-series generation. The model first focuses on the trend and dominant seasonality during the early sampling stage, and then gradually reconstructs fine-grained details, including extreme events, in the later stage. While Fig.~\ref{fig:trajectory_sample} illustrates this process with one sample from a 500-step diffusion model, Fig.~\ref{fig:full_trajectory_sample} provides a broader visualization using 10 randomly selected samples across different sampling steps.

\subsection{Intrinsic Patterns in Extreme Events}
\label{Appendix:Supporting Analysis for Intrinsic Patterns in Extreme Events}

In Observation II of Section~\ref{sec:analysis}, we argue that extreme events are shaped by underlying physical mechanisms rather than random noise. Therefore, understanding their intrinsic patterns is essential. To support this argument, we perform empirical analyses on real-world datasets from three perspectives: predictability, context dependency, and cross-sample stability.

\begin{figure*}[htbp]
    \centering
    \begin{subfigure}[htbp]{0.45\textwidth}
        \centering
        \includegraphics[width=\linewidth]{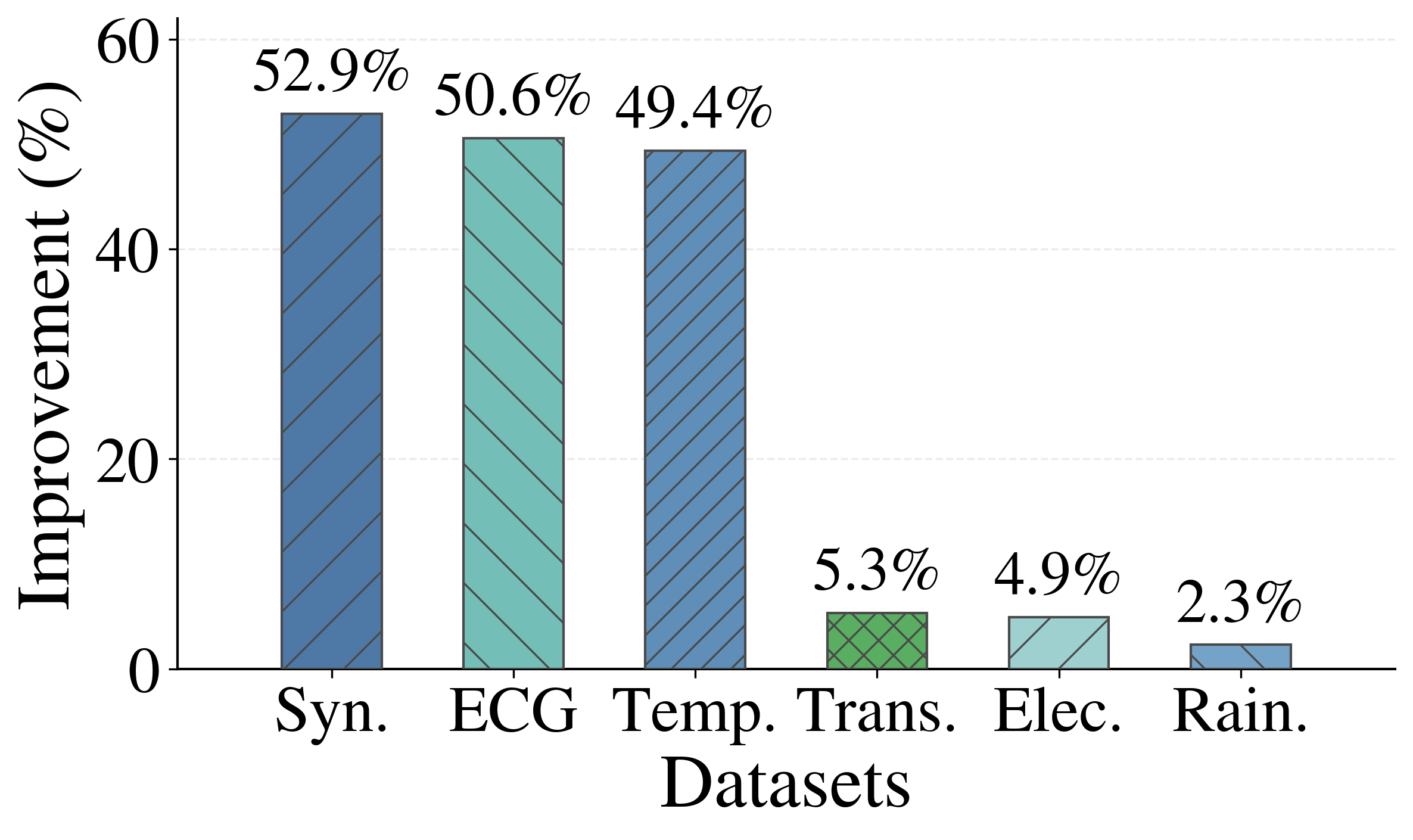}
        \caption{Predictability}
        \label{fig:predictability}
    \end{subfigure}
    \hfill
    \begin{subfigure}[htbp]{0.45\textwidth}
        \centering
        \includegraphics[width=\linewidth]{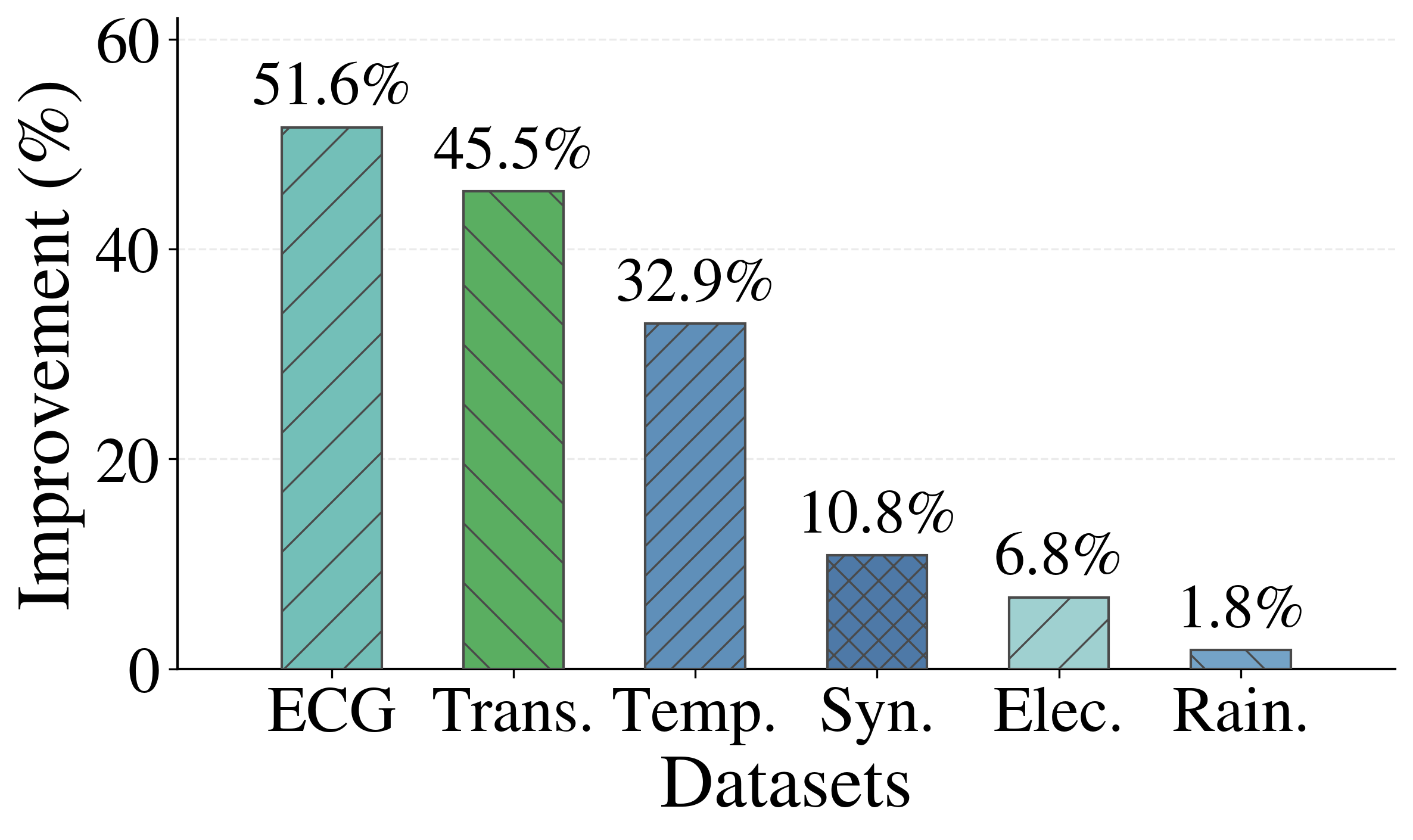}
        \caption{Context dependency}
        \label{fig:context_dependency}
    \end{subfigure}
    \caption{Intrinsic patterns in extreme events.}
    \label{fig:intrinsic_pattern}
\end{figure*}

Fig.~\ref{fig:intrinsic_pattern} visualizes the results for predictability and context dependency across six datasets. For predictability, we examine whether extreme events can be inferred from their preceding history by comparing forecasting performance using the original history with that using shuffled history. As shown in Fig.~\ref{fig:intrinsic_pattern}(a), all six datasets show positive improvements, indicating that the temporal patterns before extreme events provide useful information for predicting their future evolution. The gains are particularly clear on the Synthetic, ECG, and Temperature datasets, reaching 52.9\%, 50.6\%, and 49.4\%, respectively.

For context dependency, we test whether an extreme event is better aligned with its original surrounding context than with shuffled contexts from other samples. As shown in Fig.~\ref{fig:intrinsic_pattern}(b), all six datasets again show positive improvements, with especially strong gains on ECG, Transportation, and Temperature. This consistent trend indicates that extreme events are closely tied to their temporal context. Overall, these two analyses show that extreme events exhibit both temporal predictability and contextual compatibility, providing further evidence that they are structured and non-random.

For cross-sample stability, we refer readers to Fig.~\ref{fig:stability} and the detailed discussion in Section~\ref{sec:analysis}.

\section{Additional Details on Backbone Denoiser Design}
\label{Appendix:Backbone Denoiser}

\subsection{One-step Denoising Transition in the Backbone Denoiser}

Given an intermediate noisy state $\mathbf{x}(t)$, the Backbone Denoiser first predicts the corresponding clean-sample estimate $\hat{\mathbf{x}}_{0}(t)$. This estimate parameterizes the reverse transition from diffusion step $t$ to $t-1$, which constitutes the basic sampling operation that progressively transforms noisy states into clean time series. Under the $\mathbf{x}$-pred diffusion formulation, the one-step denoising tranition is defined as:
\begin{equation}\label{eq:x_update}
p_\theta(\mathbf{x(t-1)}\mid\mathbf{x(t)})
= \frac{1}{1-\bar{\alpha}_t}
\Big(
  \sqrt{\bar{\alpha}_{t-1}}\,\beta_t\,\hat{x}_{0}(t)
  + \sqrt{\alpha_t}\,(1-\bar{\alpha}_{t-1})\,\mathbf{x(t)}
\Big)
+ \tilde{\sigma}_t\,\mathbf{z}_t,
\end{equation}
where $\bar{\alpha}_t=\prod_{s=1}^{t}\alpha_s$, $\alpha_t=1-\beta_t$, and $\mathbf{z}_t\sim\mathcal{N}(\mathbf{0},\mathbf{I})$. Here, $\beta_t$ denotes the noise schedule, and $\tilde{\sigma}_t$ controls the stochasticity of the reverse update. In this way, the Backbone Denoiser provides the clean-sample estimate that guides each denoising step, while the stochastic term preserves the sampling variability of the diffusion process.

\subsection{Trend-Seasonality Design in the Backbone Denoiser}

This section provides additional details on the trend-seasonality design introduced in Section~\ref{sec:backbone denoiser}. This design parameterizes the $\mathbf{x}$-prediction network $\hat{\mathbf{x}}^\theta_{0}(\cdot)$ in the Backbone Denoiser, allowing it to model the global temporal backbone of the sequence through trend and seasonality components.

\textbf{Trend Function.}
The trend function is designed to capture slow, long-term variations through low-frequency information. We avoid directly using the Fourier transform for trend modeling, since long-range trends may be incompletely or inaccurately represented by periodic bases. Instead, inspired by~\cite{desai2021timevae,oreshkin2020nbeats}, we adopt a polynomial regressor and define the trend component as:
\begin{equation}
   \mathrm{trend}(\mathbf{x}(t),\theta_{tr})
= \phi(\tau)^\top \big( W_{tr}\, h(\mathbf{x}(t);a_{tr}) + b_{tr} \big), 
\end{equation}
where $\mathbf{x}(t)$ denotes the noisy time-series state at diffusion step $t$, $\tau$ is the time index, and $\phi(\tau) = [1,\, \tau,\, \tau^2, \ldots, \tau^p]^\top$ is a polynomial basis of degree $p$. The function $h(\mathbf{x}(t);a_{tr})$ is a neural feature extractor parameterized by $a_{tr}$, $W_{tr}$ is a learnable projection matrix that maps the extracted features to polynomial coefficients, and $b_{tr}$ is a bias vector. Together, $W_{tr}$, $a_{tr}$, and $b_{tr}$ form the trend parameters $\theta_{tr}$. In this formulation, the network term $W_{tr}\,h(\mathbf{x}(t);a_{tr}) + b_{tr}$ produces the polynomial coefficients, while the basis vector $\phi(\tau)$ encodes the temporal structure. Their inner product gives the trend value at time index $\tau$. Since the trend mainly captures low-frequency information, $p$ is typically set to a small value, such as $3$.

\textbf{Seasonality Function.}
The seasonality function captures dominant periodic patterns in the frequency domain. Specifically, we apply the discrete Fourier transform (DFT) to $\mathbf{x}(t)$, select the Top-$K$ frequency components with the largest magnitudes, and reconstruct the seasonal component through the inverse discrete Fourier transform (IDFT). The formulation is given by:
\begin{equation}
\mathrm{seasonality}\left(\mathbf{x}(t),\theta_{se}\right)
= \mathcal{IDFT}\!\Big( \mathcal{M}_K\!\big(\mathcal{DFT}(\mathbf{x}(t)),\theta_{se}\big) \Big),
\end{equation}
where $\mathcal{DFT}$ and $\mathcal{IDFT}$ denote the discrete Fourier transform and its inverse, respectively. $\mathcal{M}_K(\cdot,\theta_{se})$ is a masking and reweighting operator that preserves the Top-$K$ frequency components with the largest magnitudes, together with their conjugate-symmetric counterparts. Here, $\theta_{se}$ denotes the learnable parameters of a Transformer network applied to the selected frequency components, which models their interactions and refines their contributions to better reconstruct the seasonal component.

\section{Additional Details for Extreme-Event Semantic Representation}
\label{Appendix:Additional E-Semantic}

\subsection{Handling Threshold Fluctuations in Extreme-Event Definition}\label{Appendix:fluctuate_extreme_event_definition}
We define an extreme event as a localized temporal segment in which consecutive observations exceed a predefined upper threshold (or fall below a predefined lower threshold). We emphasize that an extreme event corresponds to a continuous time-series interval rather than an isolated extreme value. Here, we further discuss a common case: \textit{when extreme values fluctuate over a period of time and briefly fall below the threshold, how should our definition handle such cases?}

\begin{figure*}[htbp]
\vspace{-10pt}
    \centering
    \begin{subfigure}[t]{0.32\textwidth}
        \centering
        \includegraphics[width=\linewidth]{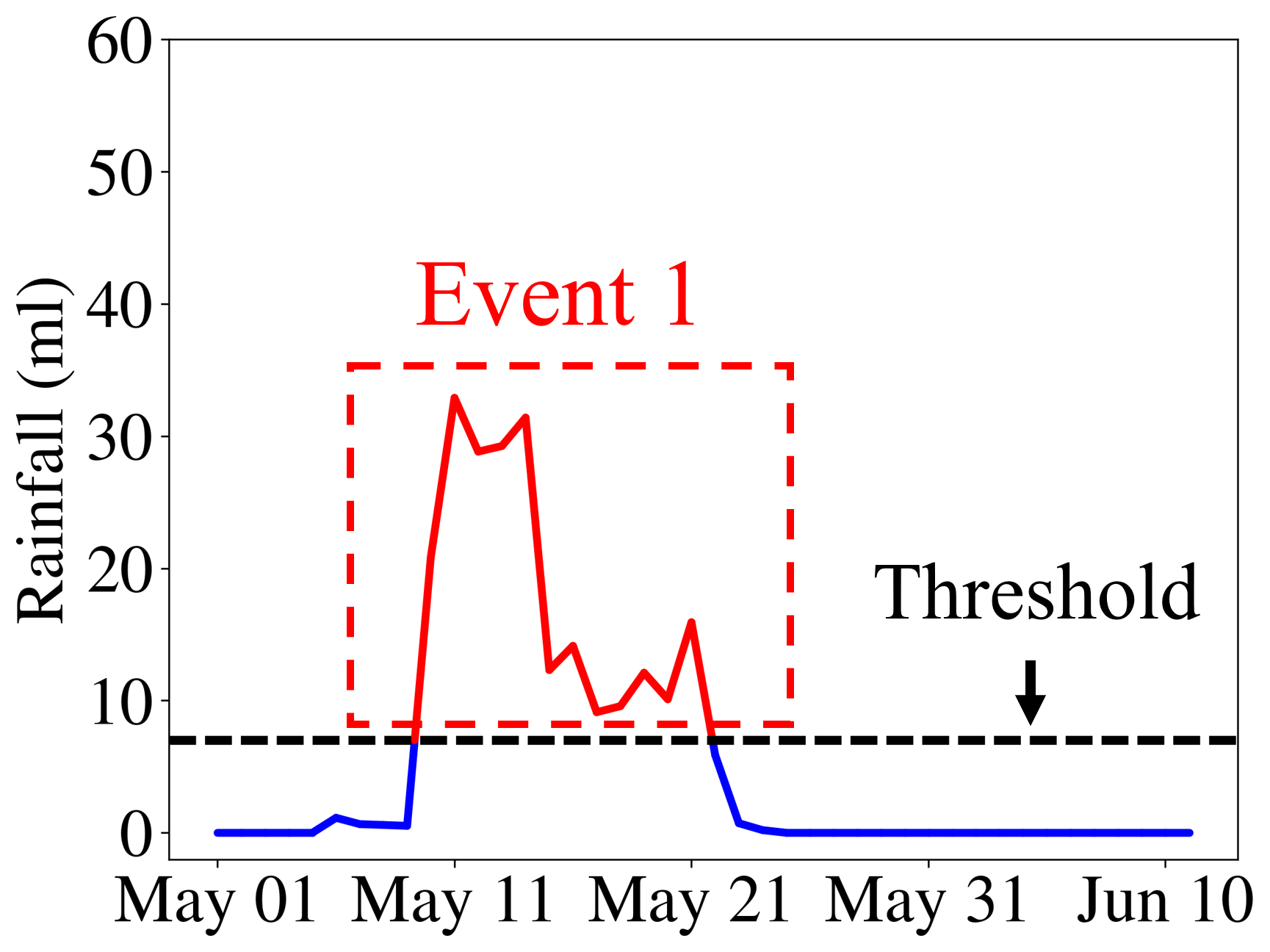}
        \caption{Fully Above-threshold Event}
        \label{fig:ecg_interval}
    \end{subfigure}
    \hfill
    \begin{subfigure}[t]{0.32\textwidth}
        \centering
        \includegraphics[width=\linewidth]{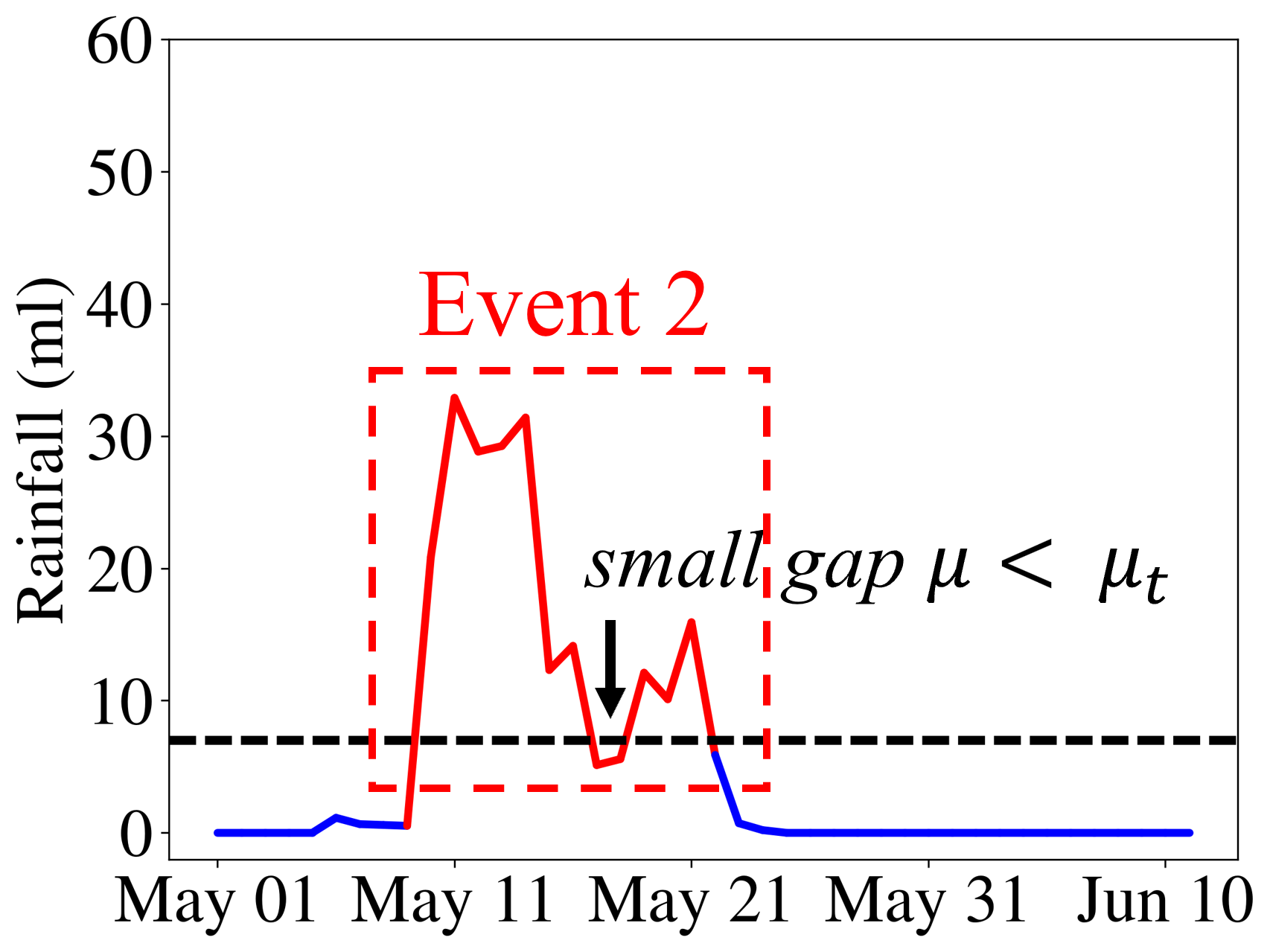}
        \caption{short-gap Interrupted Event}
        \label{fig:electricity_interval}
    \end{subfigure}
    \hfill
    \begin{subfigure}[t]{0.32\textwidth}
        \centering
        \includegraphics[width=\linewidth]{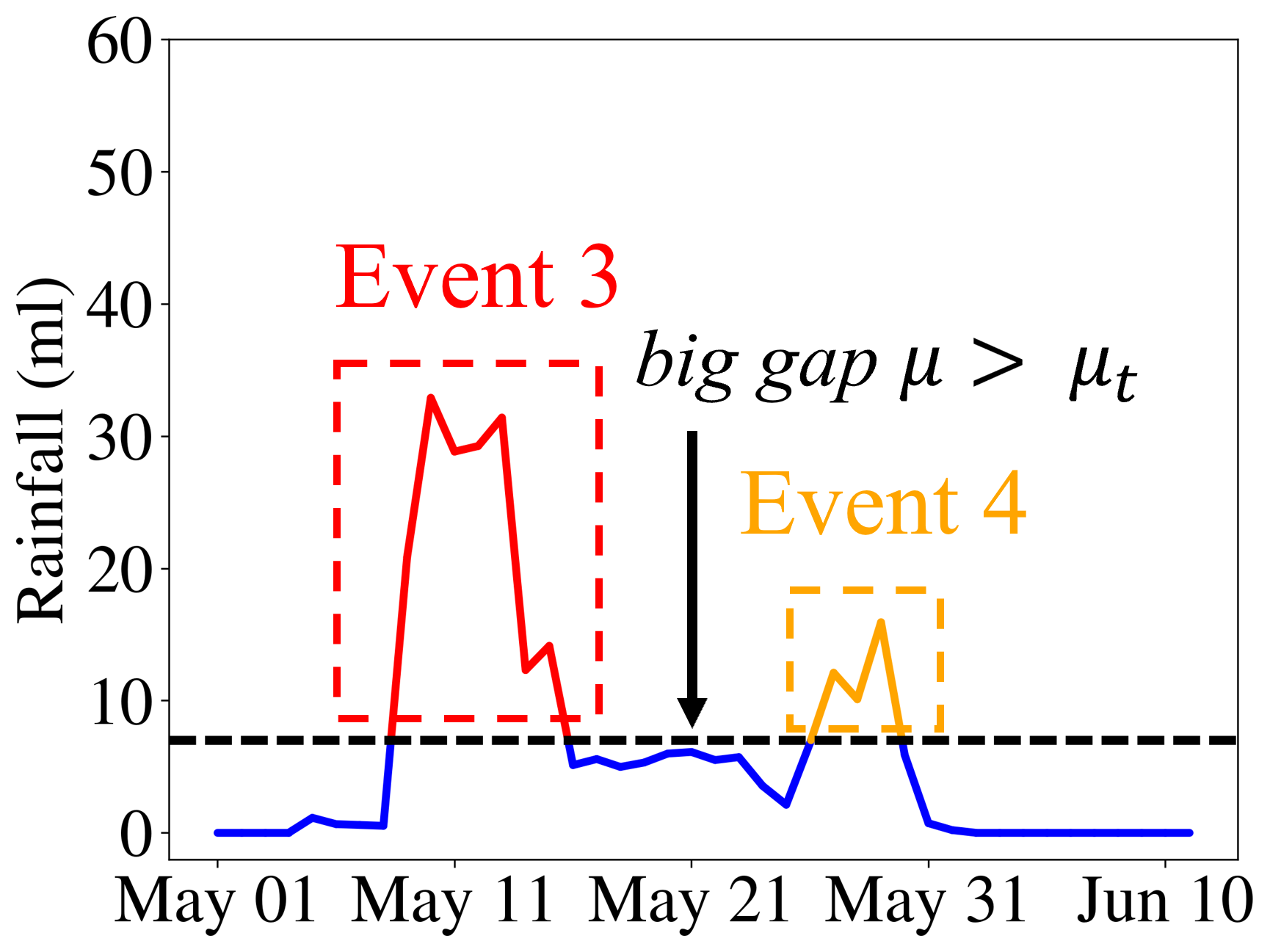}
        \caption{Long-gap Separated Event}
        \label{fig:transportation_interval}
    \end{subfigure}
    \caption{Illustration of Threshold Fluctuations in Extreme-event Definition.}
    \label{fig:threshold flucations}
\end{figure*}

As shown in Figure~\ref{fig:threshold flucations}, we introduce three cases. For a fully above-threshold segment, the entire interval is treated as one extreme event. When a short below-threshold gap appears within an extreme period, we merge the adjacent above-threshold intervals if the gap length is smaller than a tolerance $\mu_t$, thereby preserving the continuity of the same event. In contrast, when the below-threshold gap is larger than $\mu_t$, the intervals are treated as separate events. This rule prevents over-fragmentation caused by short-term fluctuations while separating truly separate extreme events.

\section{Additional Details for E-Activator} \label{Appendix:E-Activation}
In Section~\ref{sec:E-Activate}, we introduce two metrics for E-Activator: Backbone Drift (BD) and Outlier Emergence Score (OES), defined as follows. Both are evaluated on $\hat{\mathbf{x}}_{0}(t)$, the predicted final sample at step $t$.
\begin{equation}
\mathrm{BD}(t)=\mathbb{E}\left[\frac{1}{k}\sum_{j=0}^{k-1}
\frac{\sqrt{\frac{1}{LC}\sum_{\tau=1}^{L}\sum_{\nu=1}^{C}\left(bb_{t-j,\tau,\nu}-bb_{t-j-1,\tau,\nu}\right)^2}}
{\sqrt{\frac{1}{LC}\sum_{\tau=1}^{L}\sum_{\nu=1}^{C}bb_{t-j-1,\tau,\nu}^2}}\right]
\end{equation}
\begin{equation}
\mathrm{OES}(t)=\mathbb{E}\left[\max\left(\left|\frac{\hat{x}_{0,\tau,\nu}(t)-\mathrm{med}_{\tau}(\hat{x}_{0,\tau,\nu}(t))}{1.4826\,\mathrm{mad}_{\tau}(\hat{x}_{0,\tau,\nu}(t))+\varepsilon}\right|-\kappa,\,0\right)\right]
\end{equation}
\noindent where $\mathbb{E}[\cdot]$ denotes empirical averaging. For BD, the expectation is taken over samples in a batch. For OES, the expectation is taken over batch samples, time indices, and feature dimensions. $L$ and $C$ denote the sequence length and the number of feature dimensions, respectively. The parameter $k$ is the temporal smoothing window used to average backbone changes over the previous $k$ sampling steps. The denominator in BD normalizes the backbone change by the energy of the previous low-frequency backbone, making the metric comparable across samples and diffusion steps.

Here, $bb_{t,\tau,\nu}$ denotes the low-frequency backbone of $\hat{\mathbf{x}}_{0}(t)$ at time index $\tau$ and feature dimension $\nu$. It is obtained by first removing the per-series mean from $\hat{\mathbf{x}}_{0}(t)$, then applying rFFT along the temporal axis, retaining the first $M$ non-DC frequency bins, and finally applying inverse rFFT to reconstruct the low-frequency component~\cite{cooley1965algorithm}. This low-pass component approximates the global trend and seasonality structure of the predicted sample.

The term $\hat{x}_{0,\tau,\nu}(t)$ denotes the $(\tau,\nu)$-th entry of $\hat{\mathbf{x}}_{0}(t)$, where $\tau\in[L]$ is the time index and $\nu\in[C]$ is the feature index. For OES, $\mathrm{med}_{\tau}(\cdot)$ and $\mathrm{mad}_{\tau}(\cdot)$ denote the median and median absolute deviation computed along the temporal axis for each sample and feature dimension. Following the standard robust-statistics convention, the constant $1.4826$ rescales $mad_{\tau}$ to be comparable to the standard deviation under a Gaussian distribution. $\kappa$ is the robust outlier threshold, which controls when a normalized deviation is counted as an emerging extreme structure, and $\varepsilon$ is a small constant used for numerical stability.

We further visualize the evolution of BD and OES on several additional datasets in Fig.~\ref{fig:control_activation_window_more_dataset}. In most cases, BD exhibits stable dynamics and gradually decreases in the later sampling stages. By contrast, OES may fluctuate on highly variable datasets, such as rainfall and temperature. For such datasets, a wider control activation window $\mathcal{I}_{\mathrm{AC}}$ may be needed, followed by more fine-grained tuning within the selected range.

\begin{figure*}[htbp]
    \centering
    \begin{subfigure}[t]{0.32\textwidth}
        \centering
        \includegraphics[width=\linewidth]{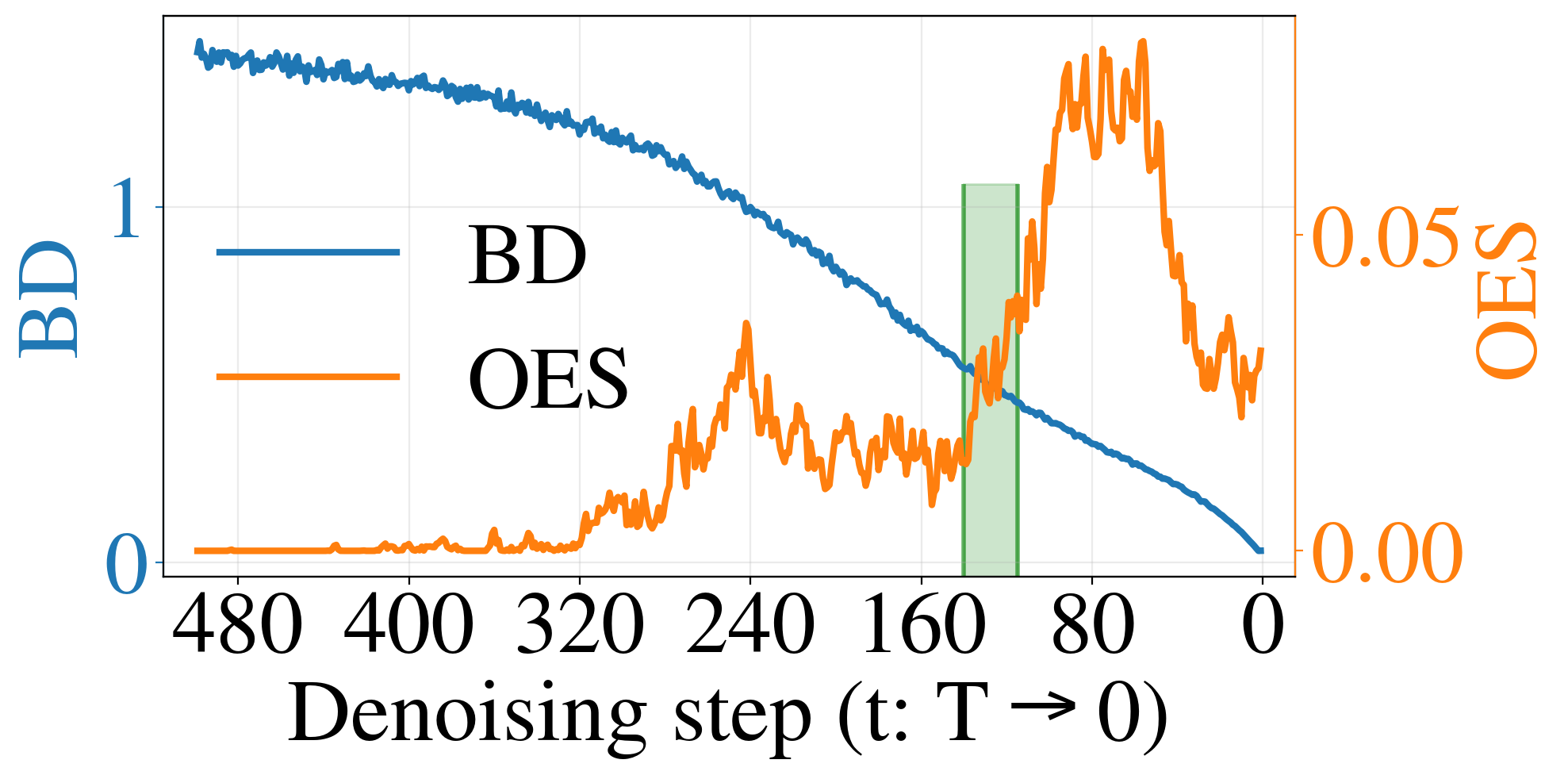}
        \caption{ECG Dataset}
        \label{fig:ecg_interval}
    \end{subfigure}
    \hfill
    \begin{subfigure}[t]{0.32\textwidth}
        \centering
        \includegraphics[width=\linewidth]{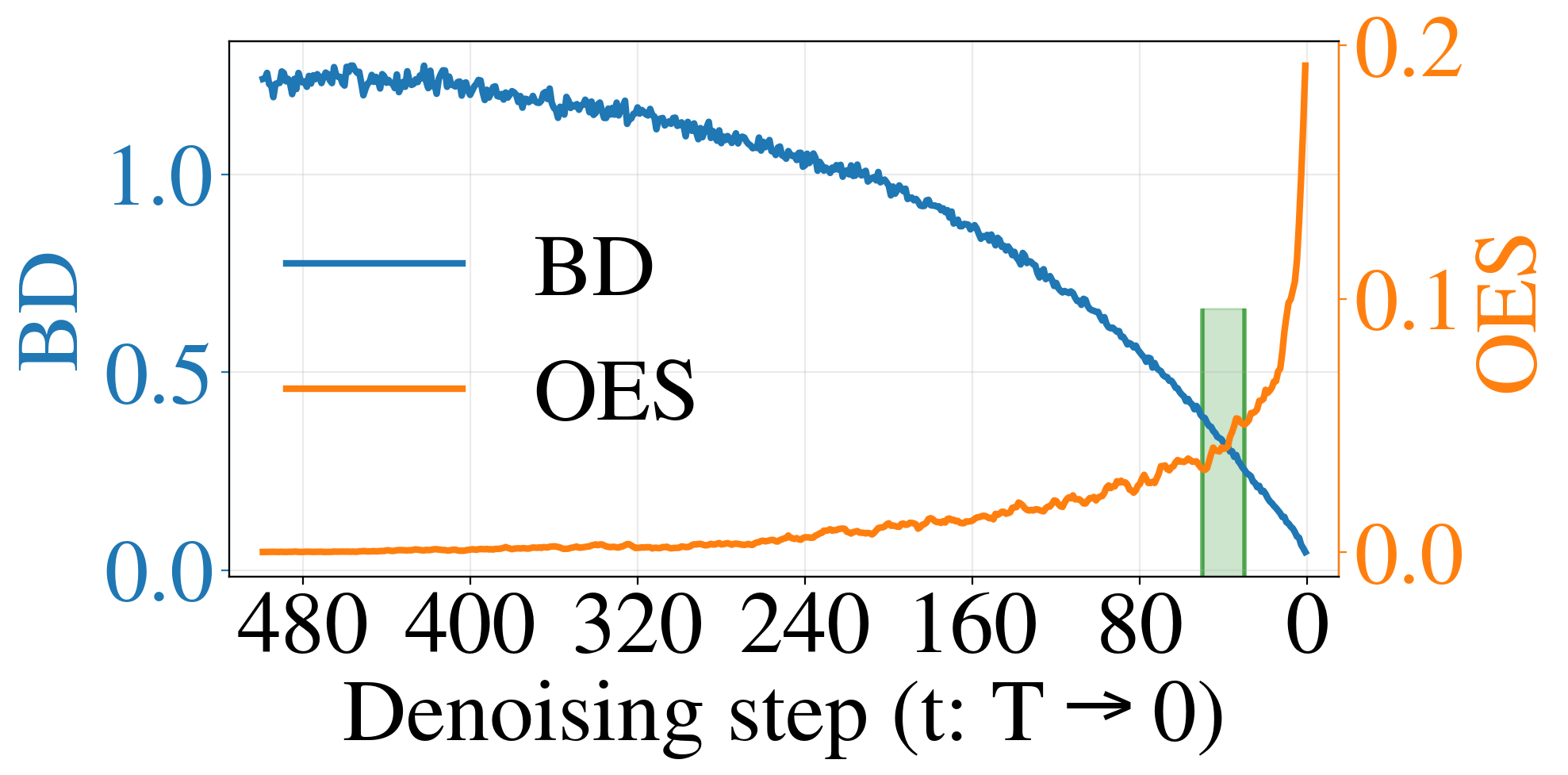}
        \caption{Electricity Dataset}
        \label{fig:electricity_interval}
    \end{subfigure}
    \hfill
    \begin{subfigure}[t]{0.32\textwidth}
        \centering
        \includegraphics[width=\linewidth]{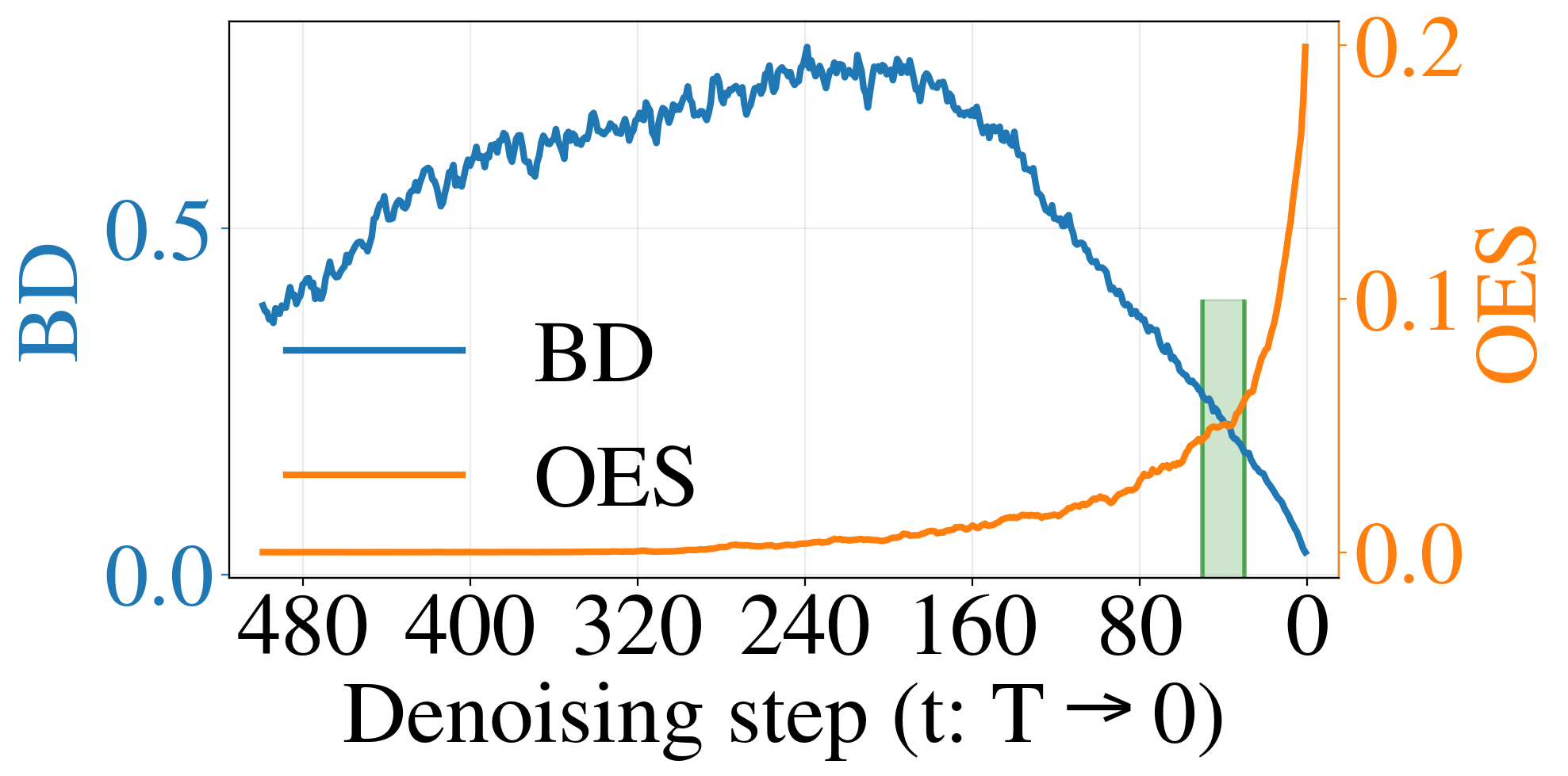}
        \caption{Transportation Dataset}
        \label{fig:transportation_interval}
    \end{subfigure}
    \caption{Control Activation Window via BD and OES in More Datasets}
    \label{fig:control_activation_window_more_dataset}
\end{figure*}

\section{Additional Details for Semantic Predictor $f_\phi$}
\label{Appendix:E-Predictor design}
In this section, we provide the detailed design of the Semantic Predictor $f_\phi$ in E-Predictor introduced in Section~\ref{sec:E-Predictor}.
Given the intermediate denoising state $\mathbf{x}_n(t_{\mathrm{CA}})$, 
$f_\phi$ predicts the extreme-event semantics  
$\smash{\hat{E}_{\mathrm{CA}}}$ as the semantic-based control signal. It consists of three modules: \textit{Transformer-based token encoding}, 
\textit{multi-head semantic prediction}, and \textit{event-level semantic decoding}. 
 
\textbf{Transformer-based Token Encoding.}
The first module encodes the intermediate denoising state 
$\mathbf{x}_n(t_{\mathrm{CA}})$ into contextualized token representations. 
Specifically, before token encoding, we feed $\mathbf{x}_n(t_{\mathrm{CA}})$ into the 
$x$-prediction function in Eq.~\ref{eq:x0-prediction} to predict the final clean sample 
$\hat{\mathbf{x}}_{0}(t_{\mathrm{CA}})
=
\hat{\mathbf{x}}^\theta_{0}\big(\mathbf{x}_n(t_{\mathrm{CA}})\big)$, 
which is then used as the input to the Semantic Predictor.Each time step in $\hat{\mathbf{x}}_{0}(t_{\mathrm{CA}})$ is treated as a temporal token 
and projected into a 128-dimensional hidden embedding. A learnable positional embedding 
is added to each token to preserve temporal order. The token sequence is then fed into 
a Transformer encoder composed of four encoder layers. Each layer contains a 
multi-head self-attention module with four attention heads, followed by a position-wise 
feed-forward network whose hidden width is four times the token dimension. GELU 
activation and dropout with rate 0.1 are used in the feed-forward blocks. Through 
self-attention, the encoder allows each time step to aggregate information from the 
entire sequence, enabling the predictor to capture both local event-related variations 
and long-range temporal dependencies. The resulting contextualized token 
representations serve as the shared feature basis for subsequent semantic prediction.

\textbf{Multi-head Semantic Prediction.}
Based on the contextualized token representations, the second module predicts 
token-wise semantic cues through three lightweight heads. Specifically, given the encoded token sequence, two classification heads estimate an event-region confidence map $\hat{\mathbf{c}}^{\mathrm{ev}}\in[0,1]^L$ and a keypoint confidence map $\hat{\mathbf{c}}^{\mathrm{kp}}\in[0,1]^L$, respectively, while one regression head predicts an auxiliary value sequence $\hat{\mathbf{v}}\in\mathbb{R}^{L}$. The event-region confidence map assigns each time step a probability of belonging to an extreme-event interval, which provides the basis for decoding the event count and event spans in the location semantic $\mathcal{L}$. The keypoint confidence map scores structurally important positions within potential event regions, such as endpoints, extrema, and turning points, which supports the construction of the temporal shape semantic $\mathcal{S}$. The auxiliary value sequence estimates the clean signal value at each time step, and is used to obtain the amplitudes of decoded keypoints as well as intensity-related quantities in $\mathcal{I}$. 
In this way, the multi-head design decomposes extreme-event semantic prediction into complementary token-wise cues, which can be reliably decoded into event-level semantics instead of directly regressing the complete event representation.

\textbf{Event-level Semantic Decoding.}
The third module converts the token-wise predictions into event-level semantics 
$\smash{\hat{E}_{\mathrm{CA}}}=\{\mathcal{L},\mathcal{I},\mathcal{S}\}$. 
We decode these semantics in three steps. First, we compute the location semantic $\mathcal{L}$, including the event count and event spans, from the event-region confidence map $\hat{\mathbf{c}}^{\mathrm{ev}}$. The confidence threshold is selected on the validation set and then fixed during decoding. Time steps whose confidence scores exceed this threshold are regarded as extreme-event positions. Consecutive selected positions are grouped into event intervals, while nearby intervals and very short segments are post-processed to reduce noisy detections.
Second, we compute the intensity semantic $\mathcal{I}$, including peak-related 
quantities such as event peak, prominence, and energy, from the auxiliary value 
sequence $\hat{\mathbf{v}}$. For each decoded event interval, we read the values of $\hat{\mathbf{v}}$ within this interval. Since $\hat{\mathbf{v}}$ estimates the signal value at each time step, the values inside the decoded interval indicate the event strength and are used to compute the intensity attributes.
Third, we compute the temporal shape semantic $\mathcal{S}$, represented by a set of event-level keypoints, using the keypoint confidence map $\hat{\mathbf{c}}^{\mathrm{kp}}$ and the auxiliary value sequence $\hat{\mathbf{v}}$. Within each decoded event interval, we select representative keypoints, including the event endpoints, extrema, and high-confidence structural points. The selected time indices and their corresponding 
values from $\hat{\mathbf{v}}$ form the keypoint-based shape representation.
This validation-calibrated decoding process converts the token-wise outputs into the event-level semantics $\smash{\hat{E}_{\mathrm{CA}}}$.

\section{Theoretical Proof for DCT-NIS}\label{Appendix:Proof}

\setcounter{lemma}{0}

\begin{lemma}[Semantic Consistency]\label{lem:semantic_consistency}
Let $\mathbf{x}_d(t_0)\sim p_{\mathrm{data}}$ and let 
$\mathbf{x}_d(t_{\mathrm{AS}})\sim q(\mathbf{x}_d(t_{\mathrm{AS}})\mid \mathbf{x}_d(t_0))$, 
where $q$ denotes the forward diffusion process. 
Suppose that the denoising trajectory satisfies bounded reconstruction drift and that the event-level semantics is locally stable. 
If the reverse span from the alignment start step $t_{\mathrm{AS}}$ to the control activation step $t_{\mathrm{CA}}$ is sufficiently short, 
then the anchored semantics $EES(\mathbf{x}_d(t_0))$ remains a valid supervision signal for the predictor at 
$\mathbf{x}_d(t_{\mathrm{CA}})$.
\end{lemma}

\begin{proof}[Proof of Lemma~\ref{lem:semantic_consistency}]
Let $\hat{\mathbf{x}}_0^\theta(\mathbf{x}_d(t))$ denote the clean-sample estimate predicted from the intermediate state $\mathbf{x}_d(t)$. 
The goal is to show that the clean estimate associated with $\mathbf{x}_d(t_{\mathrm{CA}})$ remains close to the anchored clean sample $\mathbf{x}_d(t_0)$ when the reverse span from $t_{\mathrm{AS}}$ to $t_{\mathrm{CA}}$ is short.

We first characterize the reconstruction error at the alignment start step. Since 
$\mathbf{x}_d(t_{\mathrm{AS}})$ is obtained by forward noising the anchored clean sample 
$\mathbf{x}_d(t_0)$, the $x$-prediction model provides a clean estimate 
$\hat{\mathbf{x}}_0^\theta(\mathbf{x}_d(t_{\mathrm{AS}}))$. Let $\epsilon_{\mathrm{AS}}$ denote an upper bound on the expected reconstruction error at $t_{\mathrm{AS}}$:
\begin{equation}\label{eq:anchor_recon_error}
\mathbb{E}\Big[
\big\|
\hat{\mathbf{x}}_0^\theta(\mathbf{x}_d(t_{\mathrm{AS}}))
-
\mathbf{x}_d(t_0)
\big\|
\Big]
\le \epsilon_{\mathrm{AS}} .
\end{equation}
Here, $\epsilon_{\mathrm{AS}}$ measures the reconstruction error at the alignment start step. For a well-trained $x$-prediction diffusion model, this error remains bounded when $t_{\mathrm{AS}}$ is sufficiently close to the clean data endpoint.

Next, we characterize the drift caused by the short learned denoising segment from $t_{\mathrm{AS}}$ to $t_{\mathrm{CA}}$. For each 
$t\in\{t_{\mathrm{CA}}+1,\ldots,t_{\mathrm{AS}}\}$, let $\eta_t$ denote an upper bound on the expected one-step change of the clean estimate in the reconstruction space:
\begin{equation}\label{eq:step_recon_drift_bound}
\mathbb{E}\Big[
\big\|
\hat{\mathbf{x}}_0^\theta(\mathbf{x}_d(t-1))
-
\hat{\mathbf{x}}_0^\theta(\mathbf{x}_d(t))
\big\|
\Big]
\le \eta_t .
\end{equation}
Here, $\eta_t$ measures the reconstruction drift introduced by one denoising update from step $t$ to step $t-1$. We use this quantity to characterize the local stability of the denoising trajectory in the reconstruction space; smaller $\eta_t$ indicates that one denoising update causes only a mild change in the predicted clean-sample estimate.

By the triangle inequality, the cumulative reconstruction drift along the reverse span is bounded by
\begin{equation}\label{eq:cum_recon_drift}
\mathbb{E}\Big[
\big\|
\hat{\mathbf{x}}_0^\theta(\mathbf{x}_d(t_{\mathrm{CA}}))
-
\hat{\mathbf{x}}_0^\theta(\mathbf{x}_d(t_{\mathrm{AS}}))
\big\|
\Big]
\le
\sum_{t=t_{\mathrm{CA}}+1}^{t_{\mathrm{AS}}}\eta_t .
\end{equation}
Combining Eqs.~\eqref{eq:anchor_recon_error} and~\eqref{eq:cum_recon_drift}, we obtain
\begin{equation}\label{eq:recon_to_anchor_bound}
\mathbb{E}\Big[
\big\|
\hat{\mathbf{x}}_0^\theta(\mathbf{x}_d(t_{\mathrm{CA}}))
-
\mathbf{x}_d(t_0)
\big\|
\Big]
\le
\epsilon_{\mathrm{AS}}
+
\sum_{t=t_{\mathrm{CA}}+1}^{t_{\mathrm{AS}}}\eta_t .
\end{equation}
This bound shows that the clean estimate at $t_{\mathrm{CA}}$ remains close to the anchored clean sample when the reconstruction error at $t_{\mathrm{AS}}$ is controlled and the reverse span is short.

We then connect reconstruction closeness to semantic consistency. 
Although event-level semantics are not globally invariant to arbitrary perturbations, they are locally stable around clean samples whose extreme-event regions have sufficient margins from the decision threshold. 
We formalize this by assuming that there exists a constant $L_E>0$ such that, for samples within a local neighborhood of $\mathbf{x}_d(t_0)$,
\begin{equation}\label{eq:semantic_lipschitz}
d_E\!\left(EES(\mathbf{x}),EES(\mathbf{x}')\right)
\le
L_E\|\mathbf{x}-\mathbf{x}'\|,
\end{equation}
where $d_E(\cdot,\cdot)$ denotes a suitable discrepancy between event-level semantic representations. 
This condition means that small reconstruction perturbations around a stable clean sample only induce bounded changes in the extracted event-level semantics.

Applying Eq.~\eqref{eq:semantic_lipschitz} to 
$\hat{\mathbf{x}}_0^\theta(\mathbf{x}_d(t_{\mathrm{CA}}))$ and $\mathbf{x}_d(t_0)$ gives
\begin{equation}\label{eq:semantic_error_bound}
\mathbb{E}\Big[
d_E\!\left(
EES\!\left(\hat{\mathbf{x}}_0^\theta(\mathbf{x}_d(t_{\mathrm{CA}}))\right),
EES\!\left(\mathbf{x}_d(t_0)\right)
\right)
\Big]
\le
L_E
\left(
\epsilon_{\mathrm{AS}}
+
\sum_{t=t_{\mathrm{CA}}+1}^{t_{\mathrm{AS}}}\eta_t
\right).
\end{equation}
Therefore, when $t_{\mathrm{AS}}$ is chosen close to $t_{\mathrm{CA}}$, the cumulative drift 
$\sum_{t=t_{\mathrm{CA}}+1}^{t_{\mathrm{AS}}}\eta_t$ is small. As established above, $\epsilon_{\mathrm{AS}}$ remains bounded. Therefore, the semantic discrepancy between the clean estimate at $t_{\mathrm{CA}}$ and the anchored clean sample remains bounded. Hence, 
$EES(\mathbf{x}_d(t_0))$ remains a valid supervision signal for training the semantic predictor at 
$\mathbf{x}_d(t_{\mathrm{CA}})$ in a bounded-error sense.
\end{proof}

\begin{lemma}[Marginal Alignment]\label{lem:marginal_alignment}
The marginal distribution of the data-conditioned state $\mathbf{x}_d(t_{\mathrm{CA}})$ is close to that of the noise-initiated state $\mathbf{x}_n(t_{\mathrm{CA}})$.
\end{lemma}

\begin{proof}[Proof of Lemma~\ref{lem:marginal_alignment}]
Let 
$P_d^{\mathrm{CA}}=\mathcal{L}(\mathbf{x}_d(t_{\mathrm{CA}}))$
denote the marginal distribution of the data-conditioned state, and let
$P_n^{\mathrm{CA}}=\mathcal{L}(\mathbf{x}_n(t_{\mathrm{CA}}))$
denote the marginal distribution of the noise-initiated state at $t_{\mathrm{CA}}$.
Let $q_t$ denote the forward-diffused marginal at step $t$ induced by
$p_{\mathrm{data}}$:
\begin{equation}\label{eq:qt_def}
q_t(\mathbf{x})
\triangleq
\int q(\mathbf{x}(t)\mid \mathbf{x}(t_0))
p_{\mathrm{data}}(\mathbf{x}(t_0))\,d\mathbf{x}(t_0).
\end{equation}
Thus, $q_{t_{\mathrm{CA}}}$ is the reference marginal at the control activation step
under the ideal diffusion dynamics.

We compare $P_d^{\mathrm{CA}}$ and $P_n^{\mathrm{CA}}$ through this common reference
marginal. By the triangle inequality of total variation distance,
\begin{equation}\label{eq:tri_align_tv}
\mathrm{TV}\!\left(P_d^{\mathrm{CA}},P_n^{\mathrm{CA}}\right)
\le
\mathrm{TV}\!\left(P_d^{\mathrm{CA}},q_{t_{\mathrm{CA}}}\right)
+
\mathrm{TV}\!\left(P_n^{\mathrm{CA}},q_{t_{\mathrm{CA}}}\right).
\end{equation}

We first consider the noise-initiated state. During noise-initiated sampling, the
process starts from the prior distribution and applies the learned denoising
transitions until step $t_{\mathrm{CA}}$. Under the ideal reverse diffusion dynamics,
the induced marginal at $t_{\mathrm{CA}}$ is exactly the forward reference marginal
$q_{t_{\mathrm{CA}}}$. For the learned denoising process, we quantify its deviation
from this ideal reference marginal by
\begin{equation}\label{eq:nis_mismatch}
\varepsilon_{\mathrm{nis}}
\triangleq
\mathrm{TV}\!\left(P_n^{\mathrm{CA}},q_{t_{\mathrm{CA}}}\right).
\end{equation}
This quantity measures the residual marginal mismatch of the practical
noise-initiated sampling distribution at the control activation step. Under the
bounded reverse-approximation error assumption for the Backbone Denoiser,
$\varepsilon_{\mathrm{nis}}$ remains bounded.

We next consider the data-conditioned state. In data-conditioned training, we sample
$\mathbf{x}_d(t_0)\sim p_{\mathrm{data}}$ and forward diffuse it to
$\mathbf{x}_d(t_{\mathrm{AS}})\sim q(\mathbf{x}(t_{\mathrm{AS}})\mid\mathbf{x}(t_0))$.
Therefore,
\begin{equation}\label{eq:dct_start_marginal}
\mathcal{L}(\mathbf{x}_d(t_{\mathrm{AS}}))=q_{t_{\mathrm{AS}}}.
\end{equation}
Starting from this forward marginal, the data-conditioned process applies the same
learned denoising transitions over the span from $t_{\mathrm{AS}}$ to
$t_{\mathrm{CA}}$. Under the ideal reverse diffusion dynamics, a chain initialized
from $q_{t_{\mathrm{AS}}}$ has marginal $q_{t_{\mathrm{CA}}}$ at step
$t_{\mathrm{CA}}$. For the learned denoising process, we quantify its deviation from
this ideal reference marginal by
\begin{equation}\label{eq:dct_mismatch}
\varepsilon_{\mathrm{dct}}
\triangleq
\mathrm{TV}\!\left(P_d^{\mathrm{CA}},q_{t_{\mathrm{CA}}}\right).
\end{equation}
This quantity measures the residual marginal mismatch accumulated along the
data-conditioned denoising segment. This mismatch is controlled by the approximation
accuracy of the learned denoising transitions and by the length of the reverse span.
In particular, following Lemma~\ref{lem:semantic_consistency}, the denoising
trajectory over the span from $t_{\mathrm{AS}}$ to $t_{\mathrm{CA}}$ has bounded
reconstruction drift, and this span is chosen to be sufficiently short. Therefore,
only a limited number of approximate denoising updates contribute to the accumulated
marginal mismatch, making $\varepsilon_{\mathrm{dct}}$ bounded and to be
small.

Substituting Eqs.~\eqref{eq:nis_mismatch} and~\eqref{eq:dct_mismatch} into
Eq.~\eqref{eq:tri_align_tv}, we obtain
\begin{equation}\label{eq:marginal_alignment_bound}
\mathrm{TV}\!\left(P_d^{\mathrm{CA}},P_n^{\mathrm{CA}}\right)
\le
\varepsilon_{\mathrm{dct}}+\varepsilon_{\mathrm{nis}} .
\end{equation}
Therefore, when the residual marginal mismatches $\varepsilon_{\mathrm{dct}}$ and
$\varepsilon_{\mathrm{nis}}$ are small, the marginal distribution of
$\mathbf{x}_d(t_{\mathrm{CA}})$ is close to that of
$\mathbf{x}_n(t_{\mathrm{CA}})$. This completes the proof.
\end{proof}

\section{Additional Details for E-Control}
\label{Appendix:E-Control design}

This appendix details the implementation of E-Control in Section~\ref{sec:E-control}, which injects the predicted control signal $\hat{E}_{\mathrm{CA}}$ into the Frozen Backbone Denoiser through a trainable Extreme Control Network. Inspired by the control-based diffusion paradigm~\cite{zhang2023adding}, E-Control is designed for event-level semantic control in time-series diffusion. The Extreme Control Network consists of three components: Semantic Encoder, Control Branch, and Control Gate.

\textbf{Semantic Encoder.}
The Semantic Encoder takes the predicted control signal $\hat{E}_{\mathrm{CA}}$ as input and maps it into semantic features for the Control Branch. As defined in Section~\ref{sec:extreme definition}, $\hat{E}_{\mathrm{CA}}$ follows the event-level semantic representation $EES=\{\mathcal{L},\mathcal{I},\mathcal{S}\}$, where each event is described by its location, intensity, and temporal shape. For each event, we embed its semantic tuple $(\mathcal{L}^{m},\mathcal{I}^{m},\mathcal{S}^{m})$ into an event-aware token $\mathbf{e}_m$ using a shared projection module:
\begin{equation}\label{eq:SE_event_emb}
\mathbf{e}_{m}
=
\mathrm{Emb}(\mathcal{L}^{m},\mathcal{I}^{m},\mathcal{S}^{m}),
\quad m=1,\ldots,n_{\mathrm{ev}}.
\end{equation}
We then aggregate all event tokens with a lightweight set encoder to obtain the semantic control representation:
\begin{equation}\label{eq:SE_highlevel}
\mathbf{C}_{\mathrm{sem}}(t_{\mathrm{CA}})
=
\mathrm{SetEnc}(\{\mathbf{e}_{m}\}_{m=1}^{n_{\mathrm{ev}}}).
\end{equation}
Here, $\mathrm{Emb}(\cdot)$ denotes a shared embedding function for event-level semantics, and $\mathrm{SetEnc}(\cdot)$ denotes a set-attention encoder. This design produces a compact, permutation-invariant representation that captures both per-event attributes and relations among events. The resulting $\mathbf{C}_{\mathrm{sem}}(t_{\mathrm{CA}})$ serves as the event-level conditioning representation for layer-aligned control.

\textbf{Control Branch.}
The Control Branch transforms the semantic features 
$\mathbf{C}_{\mathrm{sem}}(t_{\mathrm{CA}})$ into layer-aligned control signals 
$\{\mathbf{c}_\ell(t_{\mathrm{CA}})\}_{\ell=1}^{L_c}$ for the frozen Backbone Denoiser. 
To match the injection sites, it adopts a lightweight trainable structure aligned with the core blocks of the Backbone Denoiser, so that its intermediate features correspond to the hidden representations where control is injected. Starting from the current denoising state $\mathbf{x}_n(t_{\mathrm{CA}})$, each control block incorporates the time-step embedding $\mathrm{TE}(t_{\mathrm{CA}})$ and attends to the semantic features $\mathbf{C}_{\mathrm{sem}}(t_{\mathrm{CA}})$, producing the control signal for the corresponding injection layer:
\begin{equation}\label{eq:cb_block}
\mathbf{c}_\ell(t_{\mathrm{CA}})
=
\mathrm{CB}_\ell\big(
\mathbf{c}_{\ell-1}(t_{\mathrm{CA}}),
\mathrm{TE}(t_{\mathrm{CA}}),
\mathbf{C}_{\mathrm{sem}}(t_{\mathrm{CA}})
\big),
\quad \ell=1,\ldots,L_c .
\end{equation}
Here, $\mathrm{CB}_\ell(\cdot)$ denotes the $\ell$-th control block, 
$\mathrm{TE}(\cdot)$ denotes the time-step embedding function, and 
$\mathbf{c}_\ell(t_{\mathrm{CA}})$ is the control signal aligned with the $\ell$-th injection layer. In practice, semantic conditioning is implemented through cross-attention, where intermediate control features serve as queries and semantic features serve as keys and values.

\textbf{Control Gate.}
Given the layer-aligned control signals 
$\{\mathbf{c}_\ell(t_{\mathrm{CA}})\}_{\ell=1}^{L_c}$ from the Control Branch, 
the Control Gate converts them into gated residual offsets for the frozen Backbone Denoiser. 
At each predefined injection site, the gate produces an offset that matches the shape of the corresponding hidden representation or output prediction, and adaptively regulates its magnitude:
\begin{equation}\label{eq:control_gate}
\Delta_\ell(t_{\mathrm{CA}})
=
\mathrm{CG}_\ell(\mathbf{c}_\ell(t_{\mathrm{CA}}))
=
\sigma\!\big(\mathbf{W}^{(g)}_\ell \mathbf{c}_\ell(t_{\mathrm{CA}})\big)
\odot
\big(\mathbf{W}^{(\Delta)}_\ell \mathbf{c}_\ell(t_{\mathrm{CA}})\big),
\quad \ell=1,\ldots,L_c .
\end{equation}
Here, $\mathrm{CG}_\ell(\cdot)$ denotes the $\ell$-th gating module, 
$\sigma(\cdot)$ is the sigmoid function, $\odot$ denotes element-wise product, 
and $\mathbf{W}^{(g)}_\ell$ and $\mathbf{W}^{(\Delta)}_\ell$ project the control feature to the gate value and residual offset, respectively. 
The resulting offsets are injected as residual updates at the corresponding sites of the Backbone Denoiser. 
For example, when injecting into a hidden representation or the output prediction, we have
\begin{equation}\label{eq:delta_h_delta_0}
\tilde{\mathbf{x}}_h(t_{\mathrm{CA}})
=
\mathbf{x}_h(t_{\mathrm{CA}})+\Delta_h(t_{\mathrm{CA}}),
\quad
\tilde{\hat{\mathbf{x}}}_0(t_{\mathrm{CA}})
=
\hat{\mathbf{x}}_0(t_{\mathrm{CA}})+\Delta_0(t_{\mathrm{CA}}),
\end{equation}
where $\Delta_h(t_{\mathrm{CA}})$ and $\Delta_0(t_{\mathrm{CA}})$ are selected from 
$\{\Delta_\ell(t_{\mathrm{CA}})\}_{\ell=1}^{L_c}$ according to the predefined injection layers. 
This gated residual design adaptively strengthens extreme-event formation while preserving background temporal patterns.

\section{IMPLEMENTATION DETAILS}

\subsection{Dataset Description} \label{Appendix:Dataset}
In this section, we provide a detailed description of the six datasets used in this work. Our evaluation is conducted on one synthetic dataset and five real-world datasets spanning diverse domains, including Temperature, Precipitation, ECG, Electricity, and Transportation. An overview of these datasets is presented in Table~\ref{tab:dataset_statistics}.

\begin{table*}[htbp]
\centering
\caption{Dataset statistics.}
\label{tab:dataset_statistics}
\renewcommand{\arraystretch}{1.1}
\resizebox{\textwidth}{!}{
\begin{tabular}{lccccc}
\toprule
Datasets & \# of Samples & Time Series Length & Categories & Extreme Events & Link \\
\midrule
Syn-Data          & 1000 & 200 & Synthetic      & Synthetic Extreme Event Sequence & \href{https://anonymous.4open.science/r/E4GEN}{E4GEN/Data} \\
Wea-Temp          & 1520 & 168 & Climate        & Cold-Temperature Event & \href{https://docs.deweydata.io/docs/custom-weather-weather-data}{Dewey/CustomWeather} \\
Wea-Prec          & 1056 & 90  & Climate        & Heavy-Rainfall Event & \href{https://docs.deweydata.io/docs/custom-weather-weather-data}{Dewey/CustomWeather} \\
LTST-ECG              & 1000 & 250 & Healthcare     & ST-Segment Elevation Event & \href{https://www.physionet.org/physiobank/database/ltstdb/}{PhysioNet LTSTDB} \\
HH-Power  & 1096 & 144 & Energy         & High Electricity Consumption Event & \href{https://archive.ics.uci.edu/dataset/235/individual+household+electric+power+consumption}{UCI Household Power} \\
PEMS-SF               & 1320 & 144 & Transportation & Traffic Congestion Event & \href{https://archive.ics.uci.edu/ml/datasets/pems-sf}{UCI PEMS-SF} \\
\bottomrule
\end{tabular}
}
\end{table*}

\begin{itemize}
    \item \textbf{Syn-Data}: For the synthetic dataset, we design a toy example to explicitly illustrate the generation mechanism of extreme events in time series. Here, an extreme event is defined as a contiguous temporal segment whose values remain above a predefined threshold. Each sequence consists of four components: a trend component, a seasonal component, a noise component, and an extreme-event component. To facilitate separate analysis, we further introduce a threshold so that the extreme-event component can be clearly identified and isolated. The full sequence is defined as
    \begin{equation}
    x_i(t)=\text{trend}_i(t)+\text{season}_i(t)+\text{noise}_i(t)+\text{extreme}_i(t), \quad t=0,\dots,L-1.
    \end{equation}
    
    The trend component is modeled as
    \begin{equation}
    \text{trend}_i(t)=b_i+a_i\frac{t}{L-1},
    \end{equation}
    where \(a_i\sim U(0.05,0.1)\) and \(b_i\sim U(-0.02,0.02)\). The seasonal component is defined as the sum of three sinusoidal waves with fixed frequencies \(k\in\{5,8,11\}\):
    \begin{equation}
    \text{season}_i(t)=\sum_{k\in\{5,8,11\}} A_{i,k}\sin\left(2\pi k\frac{t}{L}+\phi_{i,k}\right).
    \end{equation}
    Here, \(A_{i,5}\in\{0.04,0.06,0.08,0.10\}\), \(A_{i,8}\in\{0.03,0.05,0.07,0.09\}\), \(A_{i,11}\in\{0.02,0.04,0.06,0.08\}\), and \(\phi_{i,k}\in\left\{0,\frac{\pi}{2},\pi,\frac{3\pi}{2}\right\}\). The noise component is sampled from a Gaussian distribution:
    \begin{equation}
    \text{noise}_i(t)\sim\mathcal{N}(0,0.01^2).
    \end{equation}

    For the extreme-event component, we first define the background signal as the sum of the trend, seasonal, and noise components:
    \begin{equation}
    \text{base}_i(t)=\text{trend}_i(t)+\text{season}_i(t)+\text{noise}_i(t).
    \end{equation}
       We further introduce the threshold $\theta =0.3711$ to distinguish extreme events from normal background fluctuation, and the signal with the extreme-event component satisfies: \(\text{base}_i(t)+\text{extreme}_i(t) > u\). As a result, the threshold can cleanly separate normal background behavior from extreme events.
    
    To make extreme events follow a predefined intrinsic pattern, we align their occurrence locations with specific phase combinations of \(\text{base}_i(t)\). In this way, the generated extremes are structurally coupled with the underlying temporal pattern rather than inserted at arbitrary positions. Each extreme event is modeled as a smooth Gaussian-shaped spike, whose amplitude is adaptively determined by the local background level. The final sequence is therefore constructed as
    \begin{equation}
    x_i(t)=\text{base}_i(t)+\sum_{j=1}^{n_i}\text{spike}_{i,j}(t).
    \end{equation}

    For our experiments, we set the sequence length to \(200\) and generate \(1000\) samples, resulting in a final dataset of shape \((1000,200,1)\). This toy dataset provides a controlled testbed for studying the event-level generation mechanism of extremes. Notably, this design induces a visible concentration of values near the threshold region in the tail of the overall distribution, producing a local bump rather than a purely conventional long-tailed shape. As a result, the synthetic distribution exhibits a structured tail pattern, making it particularly suitable for evaluating whether generation models can capture tail-specific enhancement beyond simple heavy-tail behavior. The estimated probability density function (PDF) and sample visualizations are shown in Fig.~\ref{fig:synthetic_data_pdf} and Fig.~\ref{fig:synthetic_data_sample}, respectively. In particular, three representative samples are displayed in yellow, blue, and green. The red dashed line denotes the threshold, while the red dots indicate the extreme-value points. All subsequent figures across all datasets follow the same sample setting.
    
\begin{figure*}[htbp]
    \centering
    \begin{subfigure}[htbp]{0.40\textwidth}
        \centering
        \includegraphics[width=\linewidth]{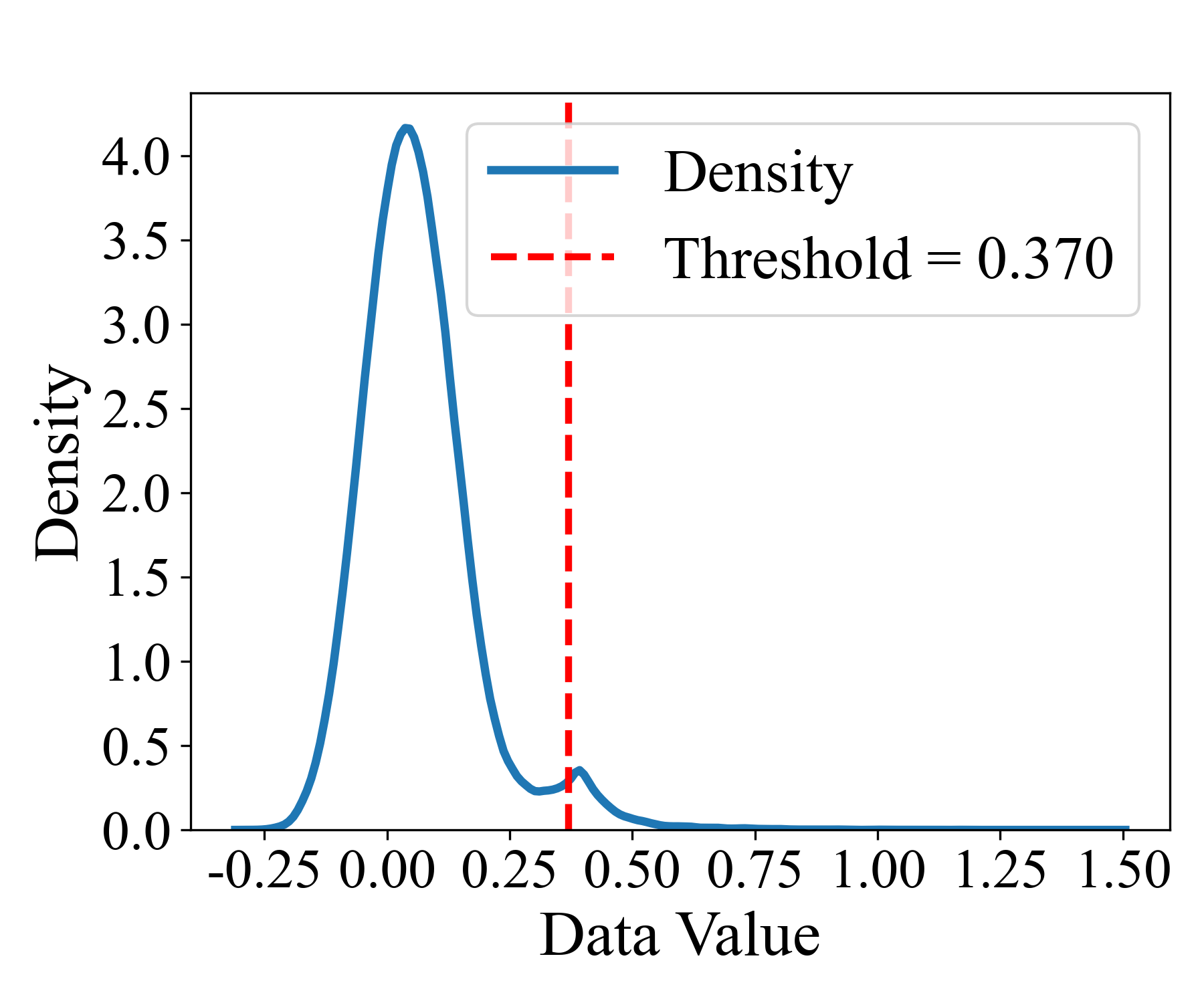}
        \caption{PDF of Syn-Data}
        \label{fig:synthetic_data_pdf}
    \end{subfigure}
    \hspace{0.05\textwidth}
    \begin{subfigure}[htbp]{0.40\textwidth}
        \centering
        \includegraphics[width=\linewidth]{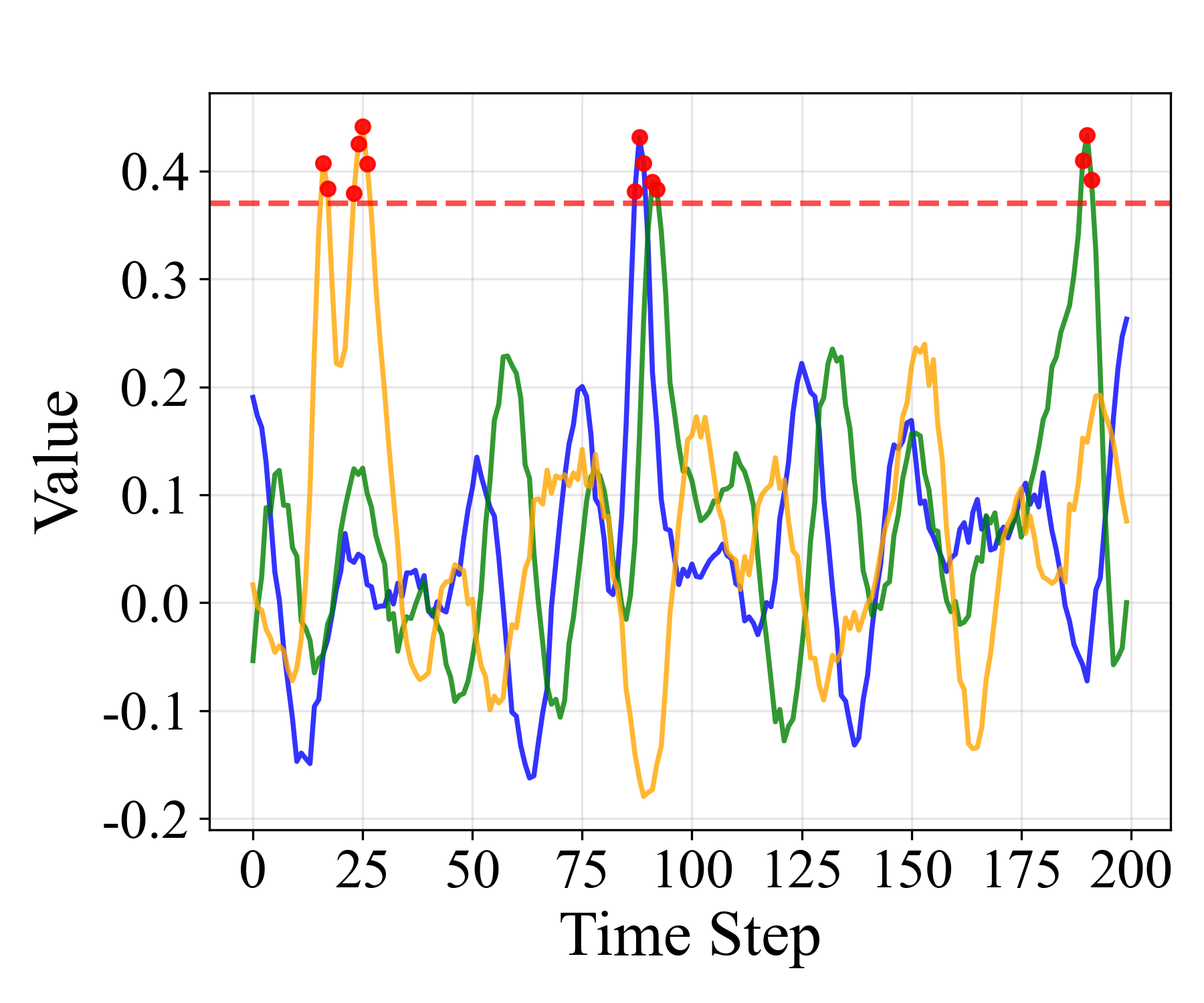}
        \caption{Data Samples in Syn-Data}
        \label{fig:synthetic_data_sample}
    \end{subfigure}
    \caption{Visualizations for Syn-Data Dataset}
    \label{fig:two_panel}
\end{figure*}

    \item \textbf{Wea-Temp}: We use the \emph{Hourly Weather Data} provided by Dewey~\cite{customweather2022hourly}, which contains hourly climate observations across the United States since 2018. Here, an extreme event refers to a low-temperature event, represented by a consecutive period with temperature below a predefined threshold. We select five nearby stations around Jacksonville in northeastern Florida, a region frequently affected by extreme weather such as heavy rainfall and hurricanes. The temperature records are divided into weekly samples, each containing 168 hourly observations \((24 \times 7)\). Using data from these five stations over the period from January 1, 2021 to December 31, 2025, we obtain 1,520 samples, resulting in a dataset of shape \((1520,168,1)\). The extreme threshold is defined as the lower 5\% quantile of the temperature distribution, corresponding to \(8.3^\circ\mathrm{C}\). Temperatures below this value are treated as anomalously low-temperature extremes. Unlike the other datasets, this benchmark focuses on lower-tail extremes, which further demonstrates the flexibility of our model across different extreme-event types. This dataset captures both long-term temporal variation and local spatial variation across nearby stations. We visualize the temperature series of the five stations over one month in Fig.~\ref{fig:monthly_temperature}. The PDF and sample visualizations are shown in Fig.~\ref{fig:temperature_data_pdf} and Fig.~\ref{fig:temperature_data_sample}, respectively.

    \begin{figure*}[h]
        \centering
        \includegraphics[width=\linewidth]{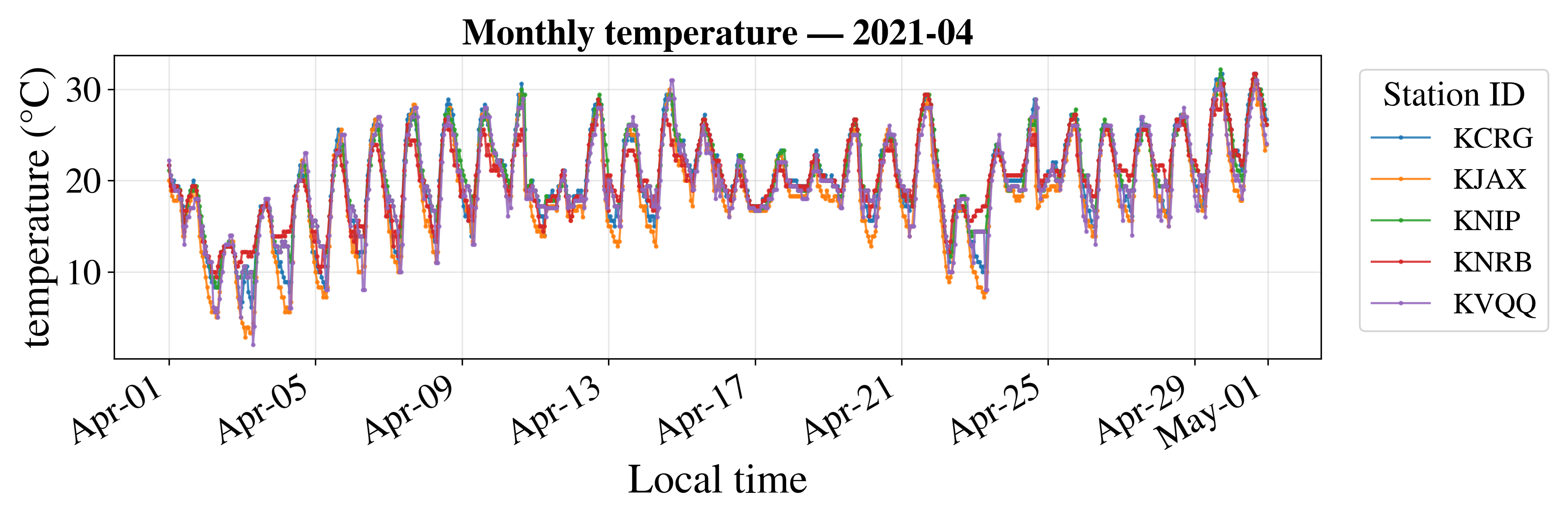}
        \caption{Temperature series for five stations in April 2021.}
        \label{fig:monthly_temperature}
    \end{figure*}

\begin{figure*}[htbp]
    \centering
    \begin{subfigure}[htbp]{0.40\textwidth}
        \centering
        \includegraphics[width=\linewidth]{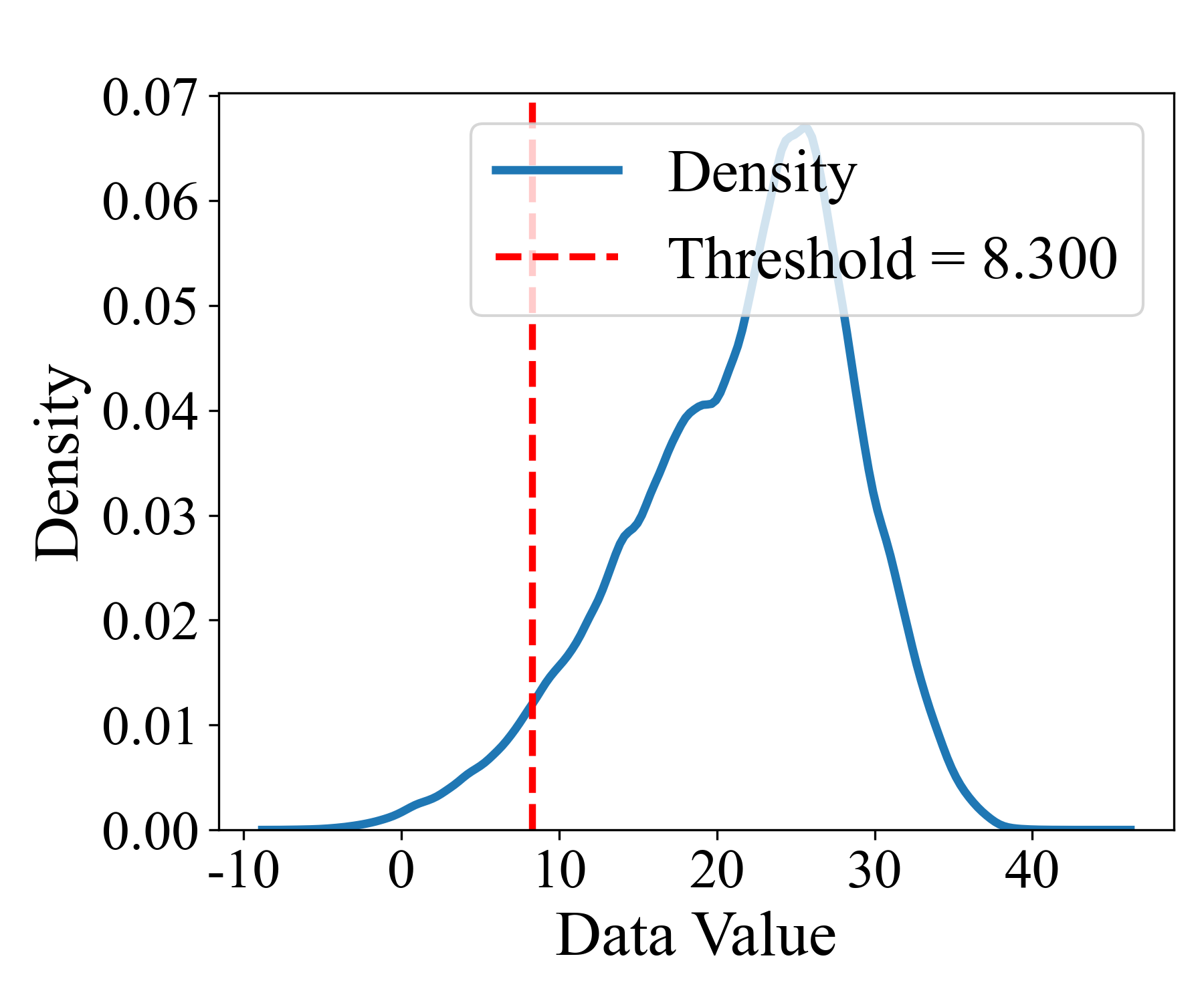}
        \caption{PDF of Wea-Temp Dataset}
        \label{fig:temperature_data_pdf}
    \end{subfigure}
    \hspace{0.05\textwidth}
    \begin{subfigure}[htbp]{0.40\textwidth}
        \centering
        \includegraphics[width=\linewidth]{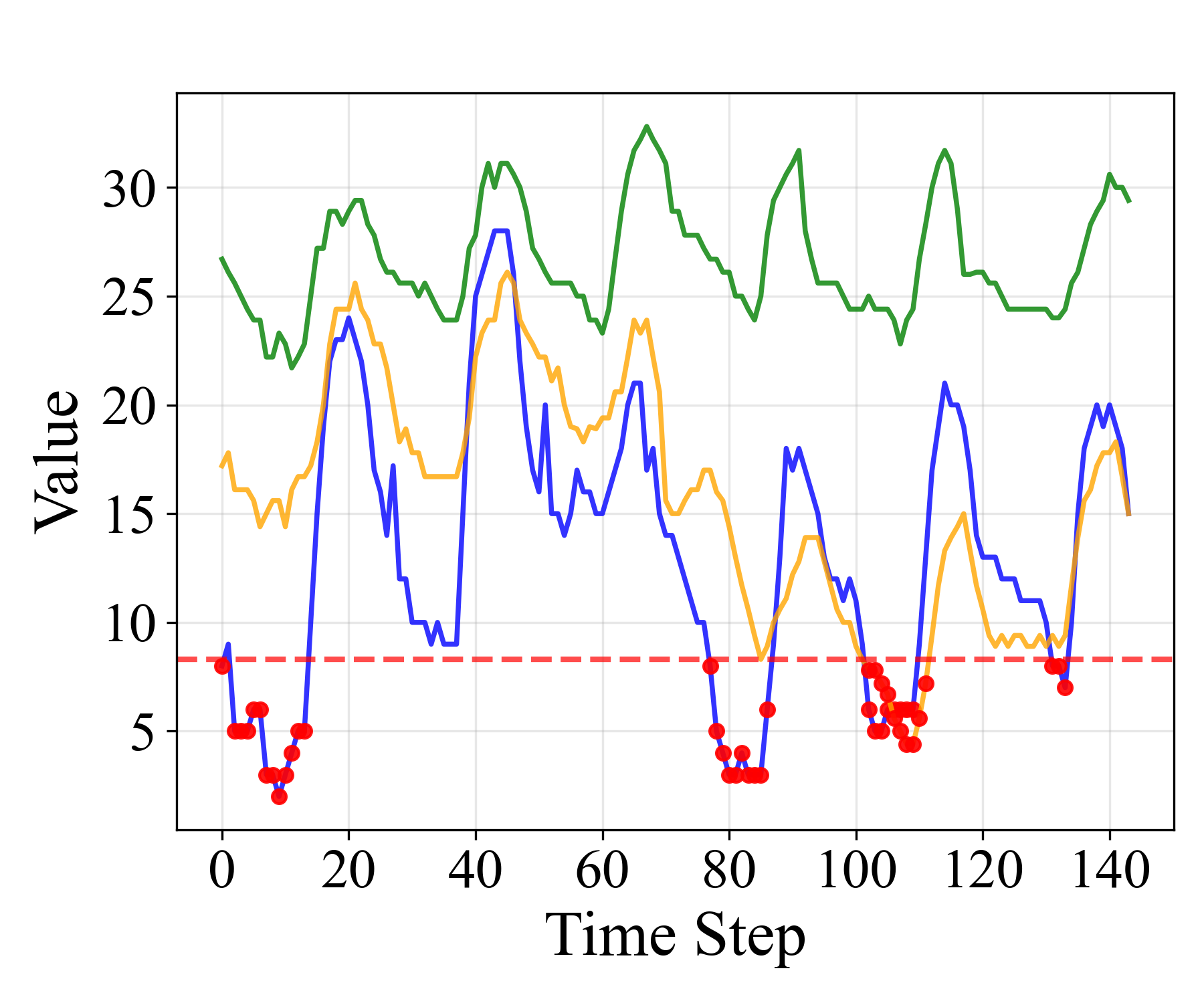}
        \caption{Data Samples in Wea-Temp Dataset}
        \label{fig:temperature_data_sample}
    \end{subfigure}
    \caption{Visualizations for Wea-Temp Dataset}
    \label{fig:two_panel}
\end{figure*}

    \item \textbf{Wea-Prec}: We use the \emph{Daily Weather Data} provided by Dewey~\cite{customweather2022daily}. In this dataset, an extreme event refers to an extreme rainfall event, characterized by a consecutive period with daily precipitation above a predefined threshold. We select 88 observation stations across Florida and record the daily total precipitation at each station. Daily precipitation is used instead of hourly precipitation because the hourly records contain substantial missing values. The data are partitioned into 90-day samples. Using records from January 1, 2023 to December 31, 2025, we construct a dataset of shape \((1056,90,1)\). The extreme threshold is set to \(1015\), corresponding to the top 10\% of the data. The PDF and sample visualizations are shown in Fig.~\ref{fig:precipitation_data_pdf} and Fig.~\ref{fig:precipitation_data_sample}, respectively.

    \begin{figure*}[htbp]
    \centering
    \begin{subfigure}[htbp]{0.40\textwidth}
        \centering
        \includegraphics[width=\linewidth]{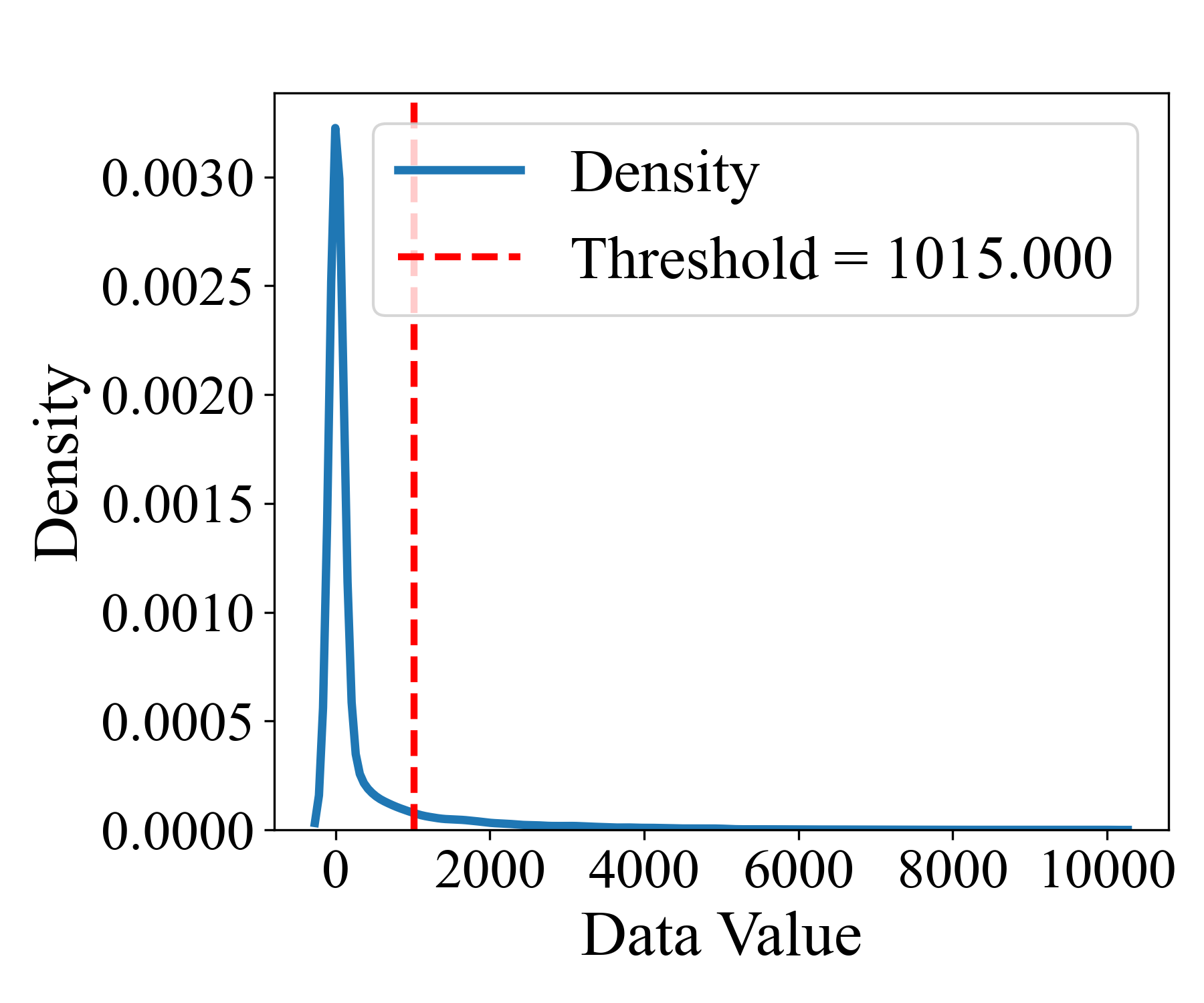}
        \caption{PDF of Wea-Prec Dataset}
        \label{fig:precipitation_data_pdf}
    \end{subfigure}
    \hspace{0.05\textwidth}
    \begin{subfigure}[htbp]{0.40\textwidth}
        \centering
        \includegraphics[width=\linewidth]{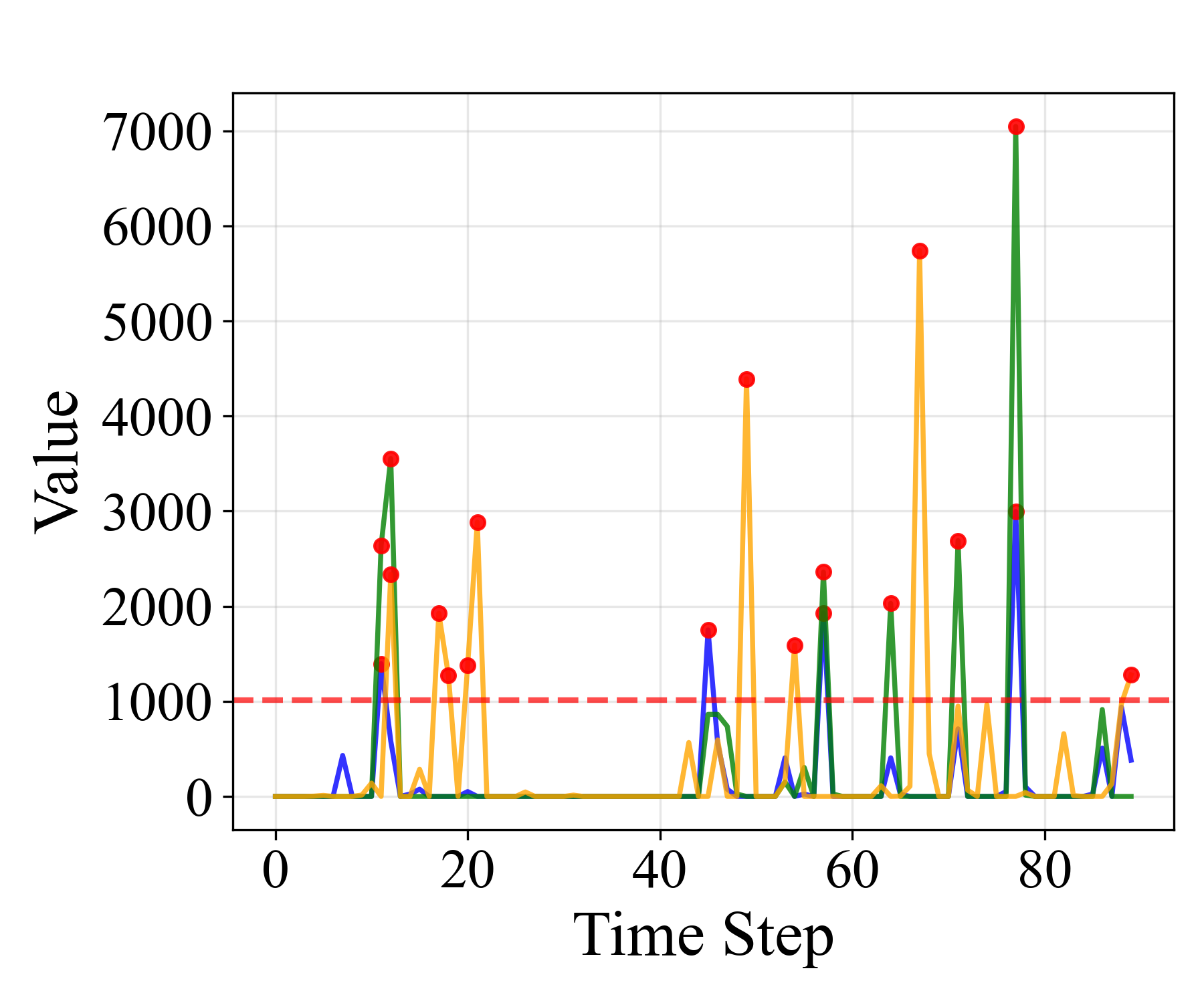}
        \caption{Data Samples in Wea-Prec Dataset}
        \label{fig:precipitation_data_sample}
    \end{subfigure}
    \caption{Visualizations for Wea-Prec Dataset}
    \label{fig:two_panel}
\end{figure*}
    
    \item \textbf{LTST-ECG}: We use the \emph{Long-Term ST Database} (LTST) from PhysioNet \cite{jager2003ltstdb,goldberger2000physionet}, which contains long-duration ambulatory ECG recordings with expert-provided ST-related annotations. In this dataset, we first identify abnormal ST-related intervals based on the provided annotations, and then define extreme events within these intervals as consecutive periods during which the signal values remain above a predefined threshold. During preprocessing, we load the raw ECG signals and annotation files, and extract the first ECG lead as a one-dimensional time series. The annotations are used to identify abnormal ST-related intervals, which provide clinically meaningful references for locating informative regions. Based on these regions, we segment the signal into fixed-length windows of 250 time steps, corresponding to one second at 250 Hz, and generate candidate samples using a sliding-window strategy. To obtain representative windows with stronger signal intensity, we rank candidates by the mean absolute signal amplitude within each window and randomly select 1,000 samples from the high-score set, yielding a final dataset of shape \((1000,250,1)\). The extreme threshold is set to \(0.3205\), corresponding to the top 10\% of the data. The PDF and sample visualizations are shown in Fig.~\ref{fig:ECG_data_pdf} and Fig.~\ref{fig:ECG_data_sample}, respectively.

    \begin{figure*}[htbp]
    \centering
    \begin{subfigure}[htbp]{0.40\textwidth}
        \centering
        \includegraphics[width=\linewidth]{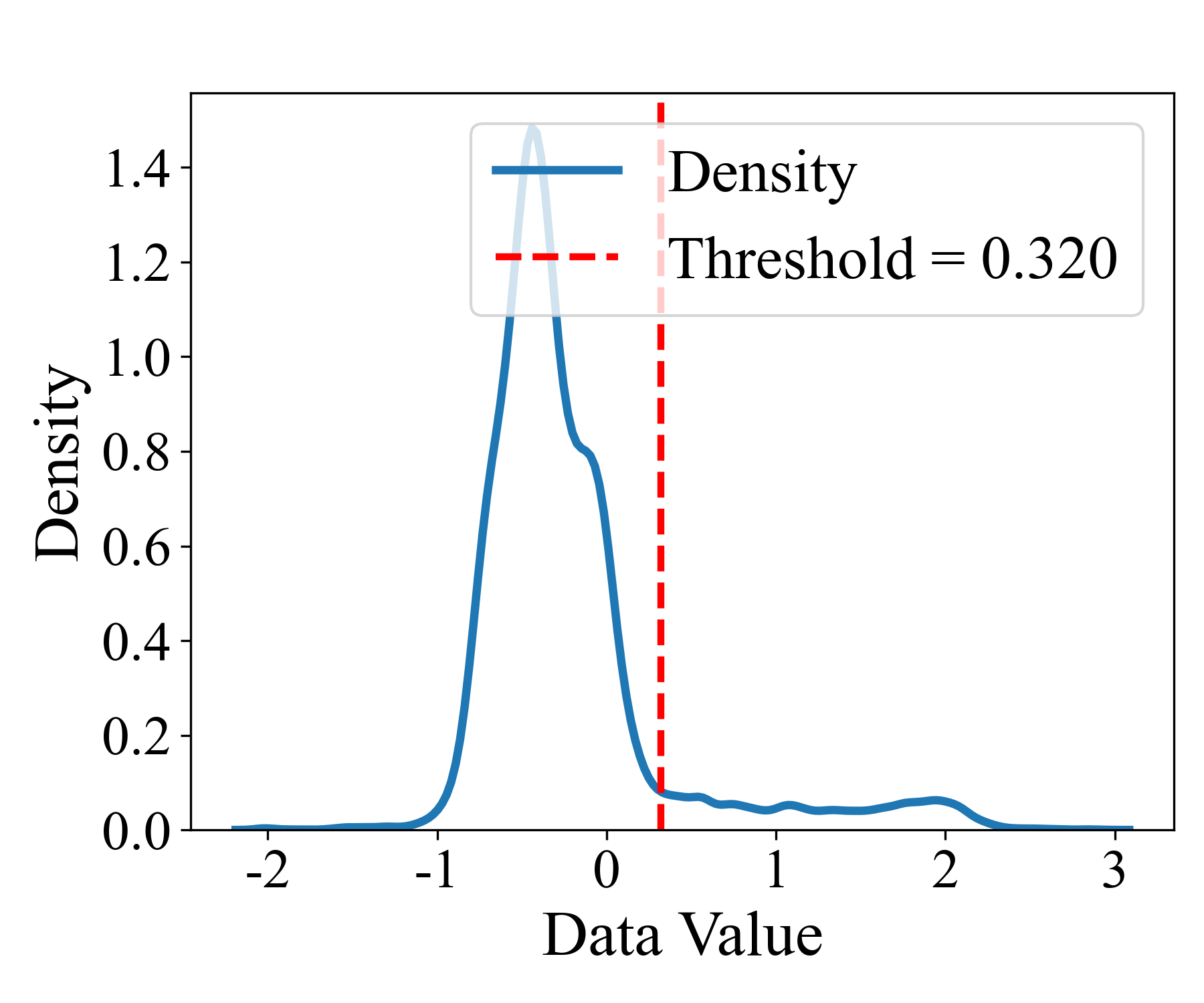}
        \caption{PDF of LTST-ECG Dataset}
        \label{fig:ECG_data_pdf}
    \end{subfigure}
    \hspace{0.05\textwidth}
    \begin{subfigure}[htbp]{0.40\textwidth}
        \centering
        \includegraphics[width=\linewidth]{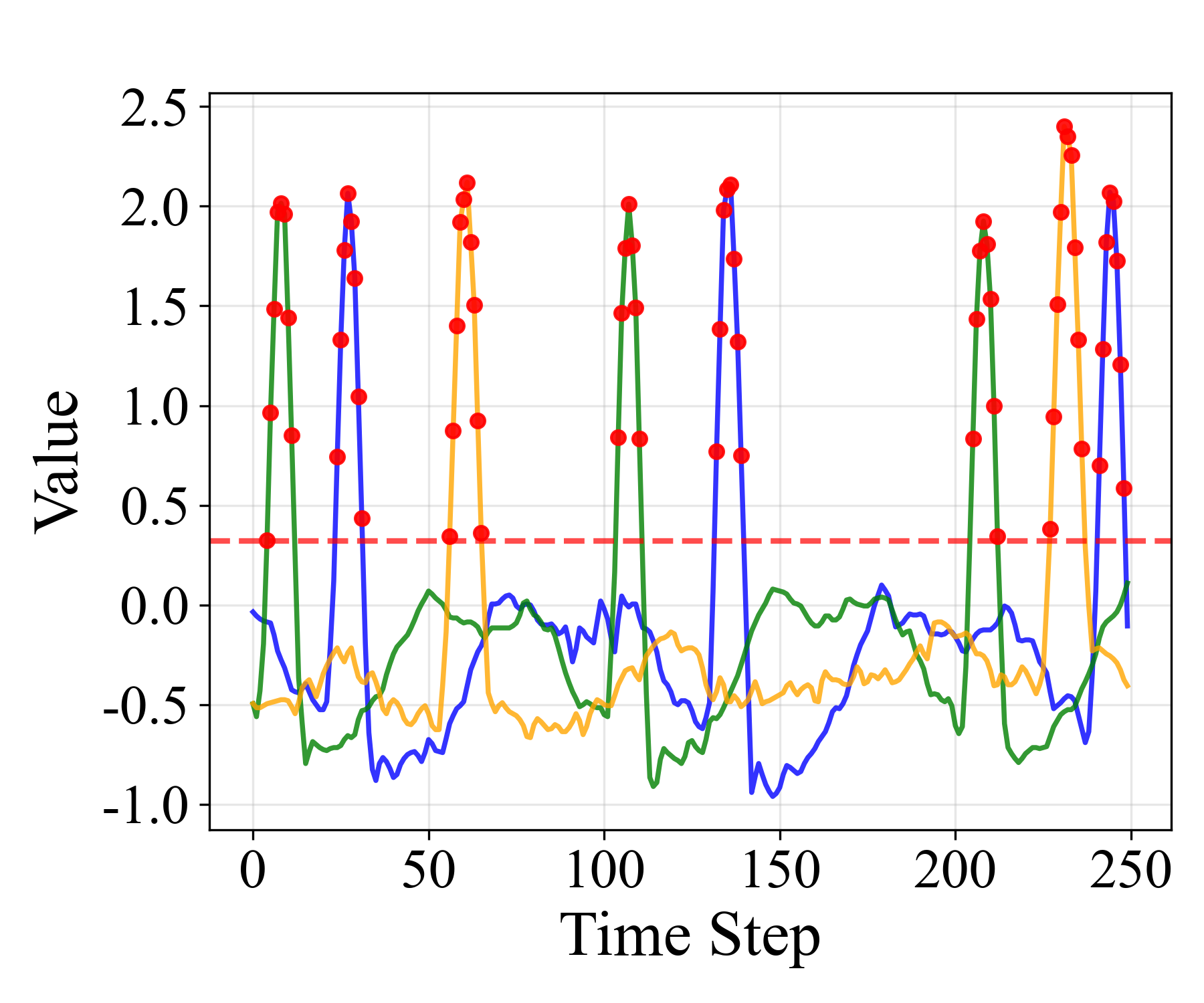}
        \caption{Data Samples in LTST-ECG Dataset}
        \label{fig:ECG_data_sample}
    \end{subfigure}
    \caption{Visualizations for LTST-ECG Dataset}
    \label{fig:two_panel}
\end{figure*}
    \item \textbf{HH-Power }: We use the \emph{UCI Individual Household Electric Power Consumption Dataset} \cite{hebrail2006individual}, which records minute-level household electricity usage over an extended period. In this dataset, an extreme event refers to a high-consumption electricity event, characterized by a consecutive period during which the power values remain above a predefined threshold. We focus on the \texttt{Global\_active\_power} variable and select a continuous three-year period from January 1, 2007 to December 31, 2009. After temporal alignment and missing-value imputation, the data are resampled at 10-minute intervals, so that each day contains 144 observations. We then group the data into daily samples and retain only complete sequences, resulting in a dataset of shape \((1096,144,1)\). The extreme threshold is set to \(2.4905\,\mathrm{kW}\), corresponding to the top 10\% of the data. This dataset provides a realistic benchmark for modeling daily electricity consumption patterns and their extreme behaviors in a single-household setting. The PDF and sample visualizations are shown in Fig.~\ref{fig:Electricity_data_pdf} and Fig.~\ref{fig:Electricity_data_sample}, respectively.
    
    \item \textbf{PEMS-SF}: We use the \emph{PEMS-SF} dataset derived from the California Performance Measurement System (PeMS)~\cite{cuturi2011pems,caltranspems}, which records freeway traffic occupancy rates from monitoring stations in the San Francisco Bay Area. In this dataset, an extreme event refers to a high-occupancy traffic event, characterized by a consecutive period during which the occupancy values remain above a predefined threshold. The occupancy value ranges from 0 to 1 and indicates the proportion of time that a sensor is occupied by vehicles within a given interval. The data are sampled every 10 minutes, so each day is represented by a sequence of length 144. Following the provided train-test split, we combine the \texttt{PEMS\_train} and \texttt{PEMS\_test} sets, obtaining 440 daily sequences in total. To build our benchmark, we focus on three nearby monitoring stations and extract their daily occupancy sequences separately, treating each station-day pair as one sample. This results in a dataset of shape \((1320,144,1)\). The extreme threshold is set to \(0.097\), corresponding to the top 10\% of the occupancy values. This construction captures realistic daily traffic dynamics while also reflecting local spatial variation across nearby sensor locations. The PDF and sample visualizations are shown in Fig.~\ref{fig:Transportation_data_pdf} and Fig.~\ref{fig:Transportation_data_sample}, respectively.
\end{itemize}

\begin{figure*}[htbp]
    \centering
    \begin{subfigure}[htbp]{0.40\textwidth}
        \centering
        \includegraphics[width=\linewidth]{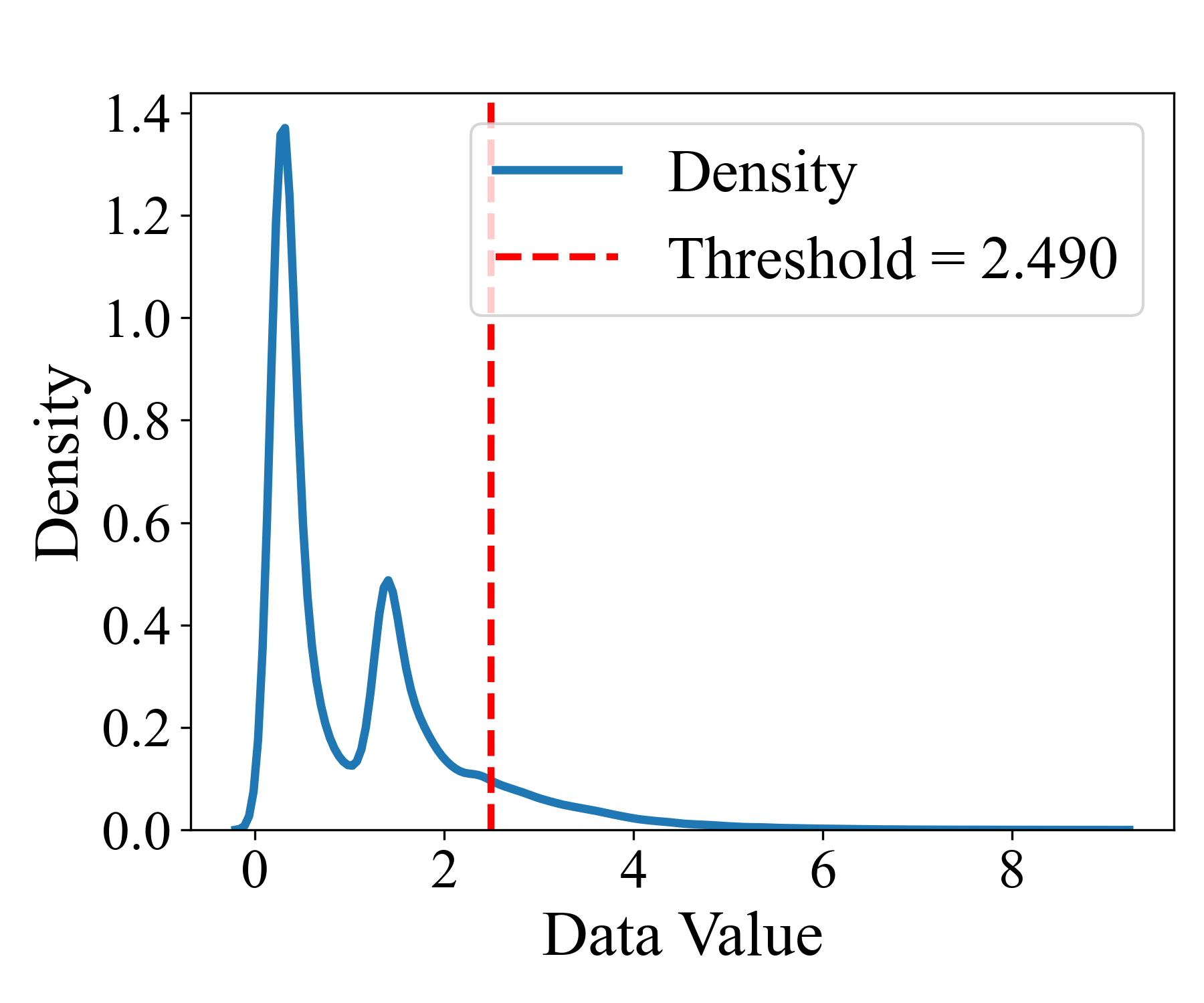}
        \caption{PDF of HH-Power Dataset}
        \label{fig:Electricity_data_pdf}
    \end{subfigure}
    \hspace{0.05\textwidth}
    \begin{subfigure}[htbp]{0.40\textwidth}
        \centering
        \includegraphics[width=\linewidth]{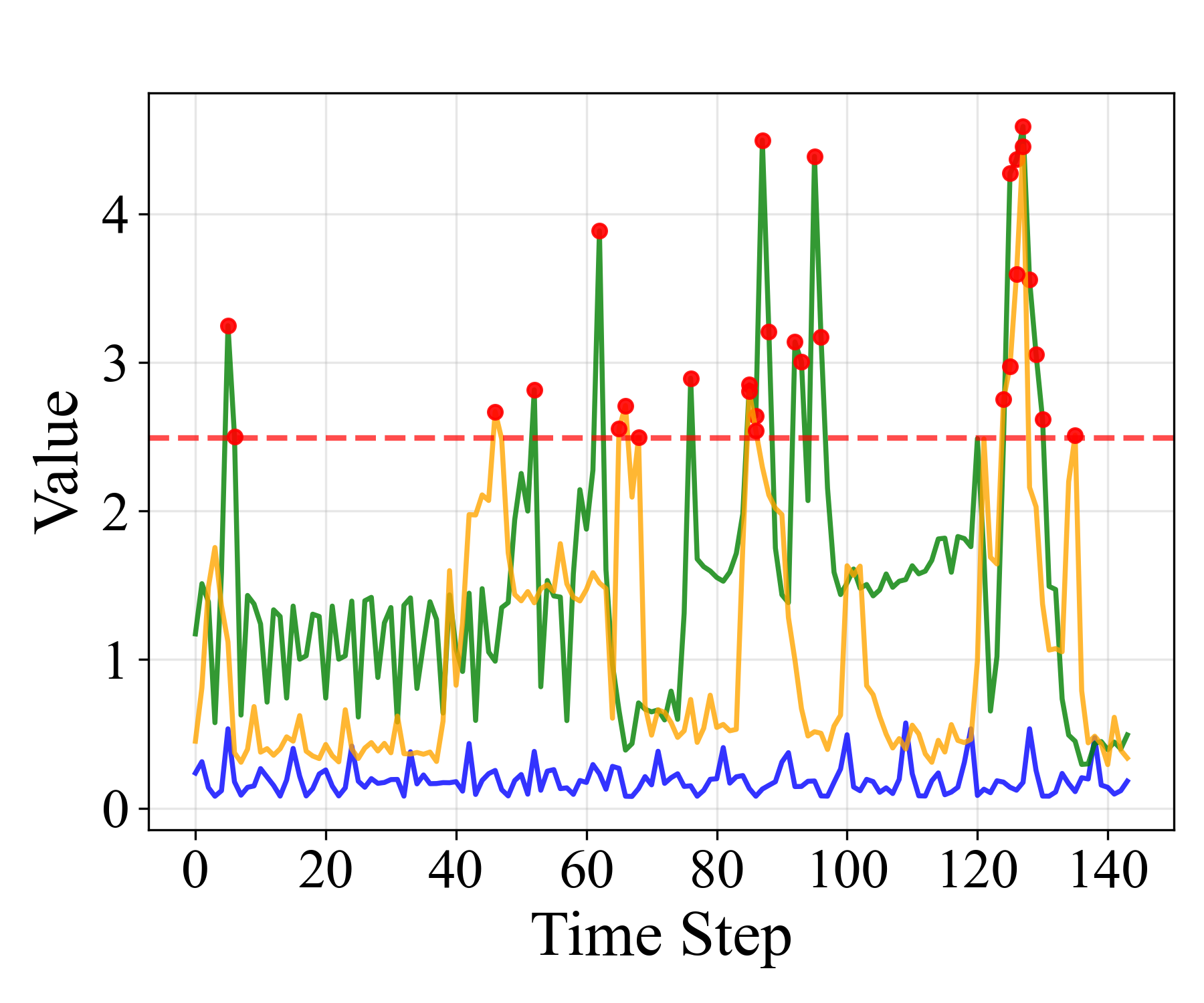}
        \caption{Data Samples in HH-Power Dataset}
        \label{fig:Electricity_data_sample}
    \end{subfigure}
    \caption{Visualizations for HH-Power Dataset}
    \label{fig:two_panel}
\end{figure*}

\begin{figure*}[htbp]
    \centering
    \begin{subfigure}[htbp]{0.40\textwidth}
        \centering
        \includegraphics[width=\linewidth]{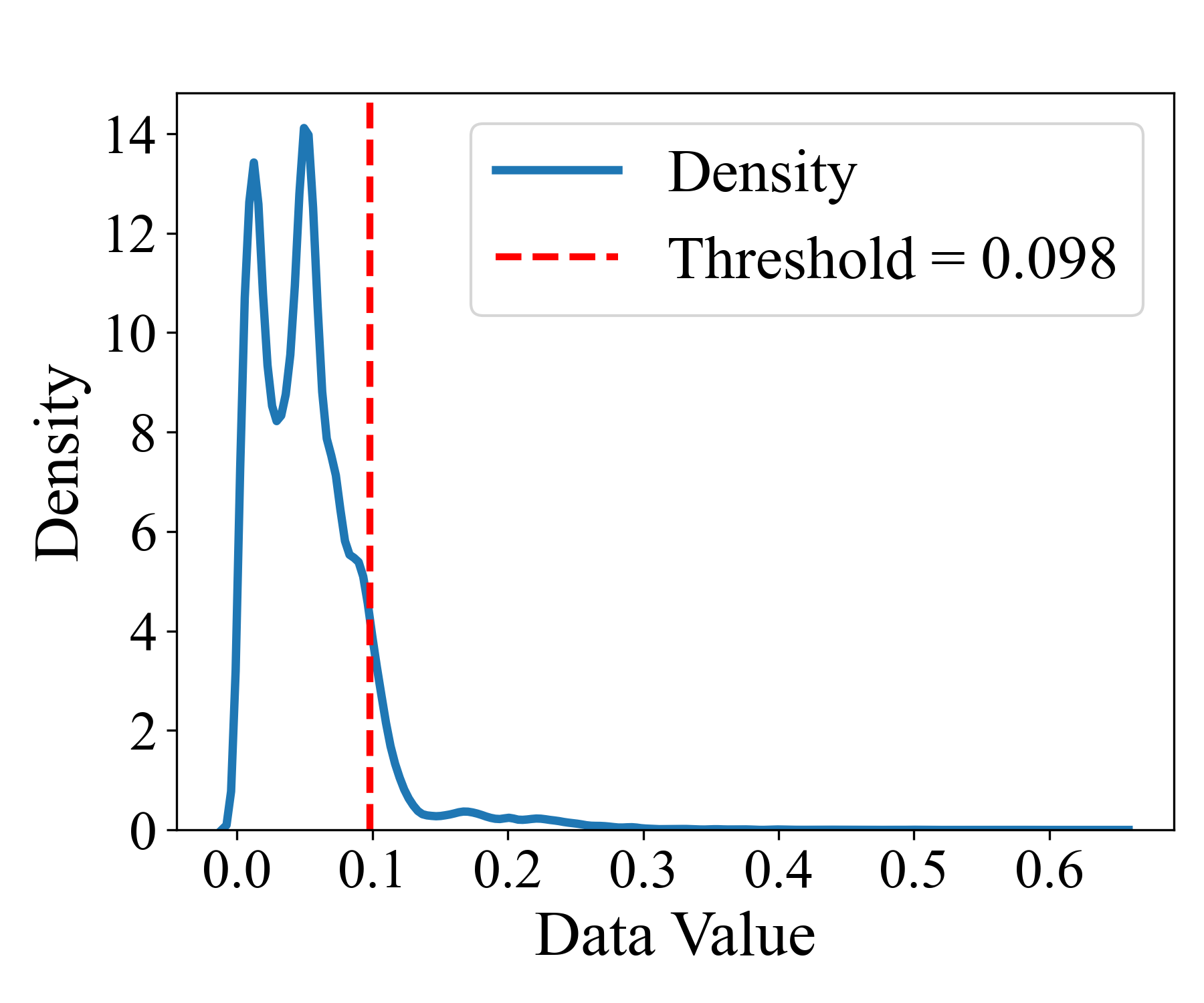}
        \caption{PDF of PEMS-SF Dataset}
        \label{fig:Transportation_data_pdf}
    \end{subfigure}
    \hspace{0.05\textwidth}
    \begin{subfigure}[htbp]{0.40\textwidth}
        \centering
        \includegraphics[width=\linewidth]{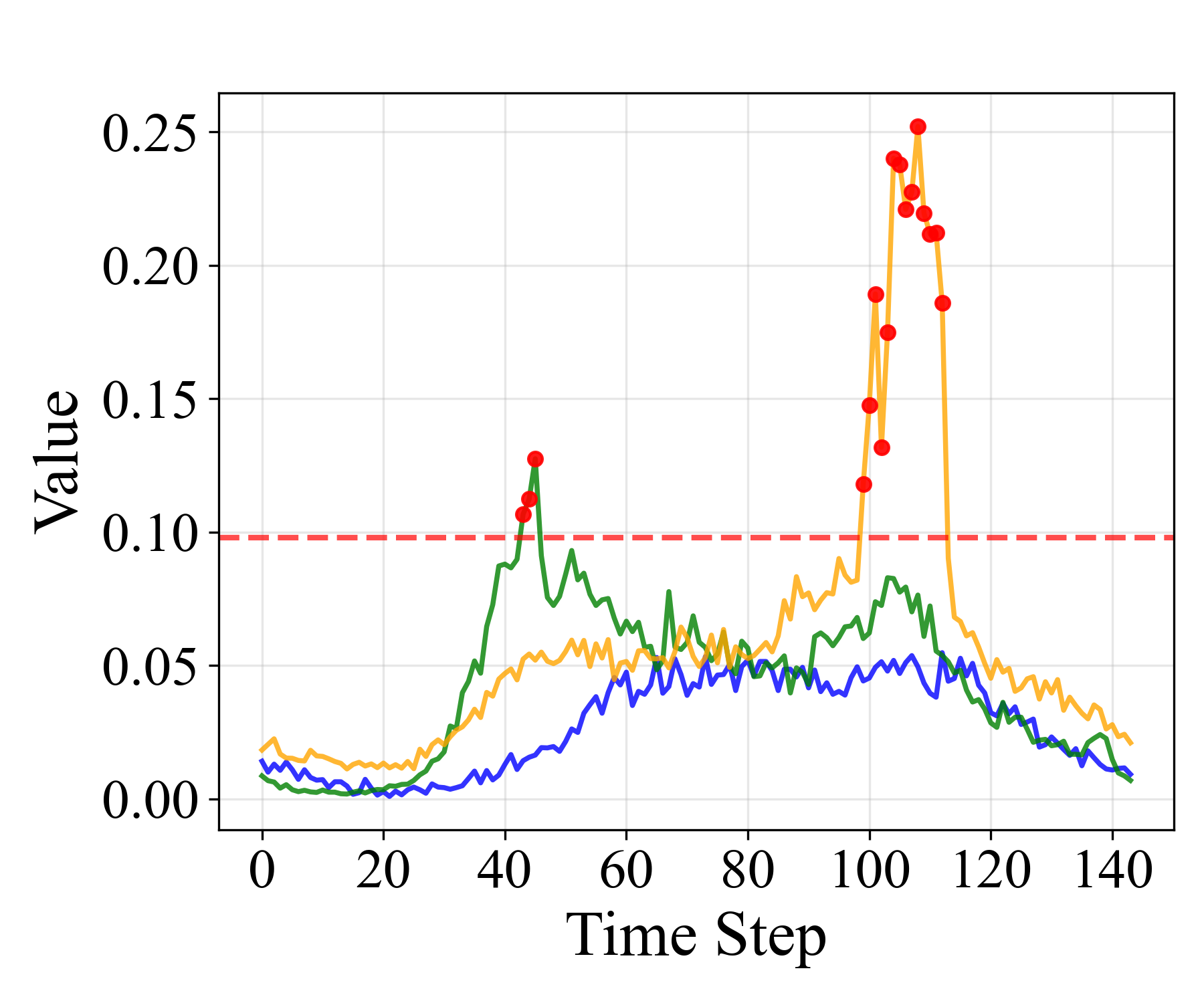}
        \caption{Data Samples in PEMS-SF Dataset}
        \label{fig:Transportation_data_sample}
    \end{subfigure}
    \caption{Visualizations for PEMS-SF Dataset}
    \label{fig:two_panel}
\end{figure*}

\FloatBarrier
\subsection{Descriptions of Metrics} \label{Appendix:Metric}

To comprehensively evaluate the performance of our generation model, we utilize seventeen distinct metrics assessed from two primary perspectives: overall generation quality and extreme-event generation quality. Below are the detailed definitions and implementations of each metric.

For \textbf{overall generation quality}, we employ eight metrics to assess the general characteristics of the generated time series, which are further divided into fidelity metrics and downstream utility metrics. Following TSGBench~\cite{ang2023tsgbench}, we use five fidelity metrics to measure the distributional and temporal similarity between the generated and real data:
\begin{itemize}
    \item \textbf{Wasserstein Distance (Wass.)~\cite{villani2009optimal}:} Measures the distance between two probability distributions by quantifying the minimum cost to transform $P$ into $Q$. We use the 1-Wasserstein distance:
    \begin{equation}
    W_1(P, Q) =
    \inf_{\gamma \in \Pi(P, Q)}
    \mathbb{E}_{(\mathbf{x}, \mathbf{y}) \sim \gamma}
    \left[\|\mathbf{x} - \mathbf{y}\|\right],
    \end{equation}
    where $\Pi(P, Q)$ denotes the set of all joint distributions with marginals $P$ and $Q$.
    \item \textbf{Kolmogorov-Smirnov Test (KS)~\cite{smirnov1939estimation}:} Computes the maximum absolute distance between the empirical cumulative distribution functions (CDFs) of the real dataset ($F_{\text{real}}$) and the generated dataset ($F_{\text{gen}}$):
    \begin{equation}
    D_{\mathrm{KS}} =
    \sup_{x}
    \left|F_{\text{real}}(x) - F_{\text{gen}}(x)\right|.
    \end{equation}
    \item \textbf{Jensen-Shannon Divergence (JS)~\cite{lin2002divergence}:} A symmetric and smoothed version of the Kullback-Leibler (KL) divergence, defined as:
    \begin{equation}
    JS(P \parallel Q) =
    \frac{1}{2} D_{\mathrm{KL}}(P \parallel M)
    +
    \frac{1}{2} D_{\mathrm{KL}}(Q \parallel M),
    \end{equation}
    where $M = \frac{1}{2}(P+Q)$.
\item \textbf{Maximum Mean Discrepancy (MMD)~\cite{gretton2012kernel}:} Evaluates the distance between distributions mapped into a Reproducing Kernel Hilbert Space (RKHS) using a kernel $k(\cdot, \cdot)$:
    \begin{equation}
    \begin{aligned}
    \mathrm{MMD}^2(P, Q)
    =&\ \mathbb{E}_{\mathbf{x}, \mathbf{x}' \sim P}
    \left[k(\mathbf{x}, \mathbf{x}')\right]
    +
    \mathbb{E}_{\mathbf{y}, \mathbf{y}' \sim Q}
    \left[k(\mathbf{y}, \mathbf{y}')\right] \\
    &\ - 2\mathbb{E}_{\mathbf{x} \sim P, \mathbf{y} \sim Q}
    \left[k(\mathbf{x}, \mathbf{y})\right].
    \end{aligned}
    \end{equation}
    \item \textbf{Autocorrelation Discrepancy (ACD)~\cite{mirylenka2017data}:} Assesses the preservation of temporal dynamics by computing the absolute difference between the autocorrelation functions:
    \begin{equation}
    \mathrm{ACD} =
    \frac{1}{L}
    \sum_{l=1}^{L}
    \left|\rho_{\text{real}}(l) - \rho_{\text{gen}}(l)\right|,
    \end{equation}
    where $\rho(l)$ is the autocorrelation at lag $l$.
\end{itemize}

Furthermore, we evaluate the general downstream utility of the generated data through three Train-on-Synthetic, Test-on-Real (TSTR) tasks:
\begin{itemize}
    \item \textbf{Context-FID (\cite{jeha2022psa}):} Measures the Fréchet distance between the contextual representations of real and generated time series. We first train a TS2Vec~\cite{yue2022ts2vec} encoder on the real data and then encode each real and generated sequence into a full-series representation. Let $\mu_r, \Sigma_r$ and $\mu_g, \Sigma_g$ denote the mean vectors and covariance matrices of the real and generated representations, respectively:
    \begin{equation}
    \mathrm{Context\text{-}FID} =
    \|\mu_r - \mu_g\|_2^2
    +
    \mathrm{Tr}
    \left(
    \Sigma_r + \Sigma_g
    - 2(\Sigma_r \Sigma_g)^{1/2}
    \right).
    \end{equation}
    \item \textbf{Predictive Score (Pred~\cite{yoon2019time}):} 
    Evaluates the prediction utility of generated time series. A forecasting model $f_\theta$ is trained only on generated data and tested on real data. For each channel $d$ and prediction horizon $h$, the model predicts the future segment from a fixed-length lookback window:
    \begin{equation}
    \hat{\mathbf{x}}_{t:t+h,d}
    =
    f_\theta(\mathbf{x}_{t-L_b:t,d}).
    \end{equation}
    The score is computed as the average prediction error on real data:
    \begin{equation}
    \mathrm{Pred}
    =
    \frac{1}{|\mathcal{D}|}
    \sum_{(d,h)}
    \frac{1}{m_{d,h}h}
    \sum_{i=1}^{m_{d,h}}
    \sum_{j=1}^{h}
    \left|
    \hat{x}_{i,j}^{(d,h)} - x_{i,j}^{(d,h)}
    \right|,
    \end{equation}
    where $m_{d,h}$ is the number of valid test windows for channel $d$ and horizon $h$. Lower values indicate that the generated data better preserves predictive temporal dynamics.
    \item \textbf{Reconstructive Score (Recon~\cite{malhotra2016lstm}):} 
    Evaluates reconstruction-based utility. An autoencoder, consisting of an encoder $E_\phi$ and a decoder $D_\psi$, is trained on generated data and tested on real data. For each real sequence $\mathbf{x}^{(i)}\in\mathbb{R}^{T\times D}$, the reconstruction is obtained as
    \begin{equation}
    \hat{\mathbf{x}}^{(i)}
    =
    D_\psi\left(E_\phi(\mathbf{x}^{(i)})\right).
    \end{equation}
    The score is the mean squared reconstruction error over all test sequences, time steps, and channels:
    \begin{equation}
    \mathrm{Recon}
    =
    \frac{1}{NTD}
    \sum_{i=1}^{N}
    \sum_{t=1}^{T}
    \sum_{d=1}^{D}
    \left(
    \hat{x}^{(i)}_{t,d}
    -
    x^{(i)}_{t,d}
    \right)^2.
    \end{equation}
\end{itemize}

For \textbf{extreme-event generation quality}, we specifically evaluate the model's ability to capture and synthesize extreme phenomena. Let $\mathcal{E}_{\text{real}}$ and $\mathcal{E}_{\text{gen}}$ denote the sets of extracted extreme events from the real and generated time series, respectively. Here, $P(\cdot)$ denotes the corresponding distribution. We propose seven event-level fidelity metrics:
\begin{itemize}
    \item \textbf{EMean-W1:} The Wasserstein-1 distance between the distributions of event-level mean values from real and generated extreme events:
    \begin{equation}
    \mathrm{EMean\text{-}W1}
    =
    W_1\left(
    P(\mathrm{Mean}(\mathcal{E}_{\text{real}})),
    P(\mathrm{Mean}(\mathcal{E}_{\text{gen}}))
    \right),
    \end{equation}
    where $\mathrm{Mean}(\mathcal{E})$ denotes the collection of mean values computed over the temporal span of each event in $\mathcal{E}$.
    \item \textbf{ECount-Diff:} The absolute difference between the numbers of extracted extreme events in real and generated time series:
    \begin{equation}
    \mathrm{ECount\text{-}Diff}
    =
    \left|
    |\mathcal{E}_{\text{real}}|
    -
    |\mathcal{E}_{\text{gen}}|
    \right|.
    \end{equation}
    where $|\mathcal{E}|$ denotes the number of extracted extreme events in the event set $\mathcal{E}$.
    \item \textbf{EDur-W1:} The Wasserstein-1 distance between the duration distributions of real and generated extreme events:
    \begin{equation}
    \mathrm{EDur\text{-}W1}
    =
    W_1\left(
    P(\mathrm{Dur}(\mathcal{E}_{\text{real}})),
    P(\mathrm{Dur}(\mathcal{E}_{\text{gen}}))
    \right),
    \end{equation}
    where $\mathrm{Dur}(\mathcal{E})$ denotes the collection of durations of all events in $\mathcal{E}$.
    \item \textbf{EPeak-W1:} The Wasserstein-1 distance between the peak-value distributions of real and generated extreme events:
    \begin{equation}
    \mathrm{EPeak\text{-}W1}
    =
    W_1\left(
    P(\mathrm{Peak}(\mathcal{E}_{\text{real}})),
    P(\mathrm{Peak}(\mathcal{E}_{\text{gen}}))
    \right),
    \end{equation}
    where $\mathrm{Peak}(\mathcal{E})$ denotes the collection of peak values of all events in $\mathcal{E}$.
    \item \textbf{EWass (\cite{villani2009optimal}):} 
    The Wasserstein-1 distance between the distributions of pooled extreme values from real and generated extreme events:
    \begin{equation}
    \mathrm{EWass}
    =
    W_1\left(
    P(\mathcal{V}(\mathcal{E}_{\text{real}})),
    P(\mathcal{V}(\mathcal{E}_{\text{gen}}))
    \right),
    \end{equation}
    where $\mathcal{V}(\mathcal{E})$ denotes the collection of all values at the extreme time points within the event set $\mathcal{E}$.
    \item \textbf{EJS (\cite{lin2002divergence}):} 
    The Jensen-Shannon distance between the distributions of pooled extreme values from real and generated extreme events:
    \begin{equation}
    \mathrm{EJS}
    =
    \sqrt{
    JS\left(
    P(\mathcal{V}(\mathcal{E}_{\text{real}}))
    \parallel
    P(\mathcal{V}(\mathcal{E}_{\text{gen}}))
    \right)
    }.
    \end{equation}
    \item \textbf{EMMD (\cite{gretton2012kernel}):} 
    The RBF-kernel Maximum Mean Discrepancy between real and generated extreme-event snippet distributions. Each snippet is a fixed-length segment centered at the peak of an identified extreme event. The final score further includes a count penalty for the mismatch in the number of extracted snippets:
    \begin{equation}
    \mathrm{EMMD}
    =
    \mathrm{MMD}_{\mathrm{RBF}}
    \left(
    P(\mathcal{S}_{\text{real}}),
    P(\mathcal{S}_{\text{gen}})
    \right)
    +
    \lambda_{\mathrm{cnt}}\Delta_{\mathrm{cnt}}.
    \end{equation}
\end{itemize}

Finally, we assess the extreme-oriented downstream utility to evaluate the practical value of the generated extreme events for critical downstream tasks:
\begin{itemize}
\item \textbf{EPred:} 
Evaluates whether generated data can support prediction under extreme-event contexts. A forecasting model is trained on generated windows whose prediction target corresponds to an extreme time point, and is then tested on real windows selected by the same rule. Specifically, the model predicts future values from 1 to 5 steps ahead using a fixed-length lookback window:
\begin{equation}
\hat{x}_{t+h,d}
=
f_\theta^{(h)}(\mathbf{x}_{t-L_b:t,d}),
\quad h=1,\ldots,5,
\end{equation}
where $d$ denotes the target channel, $L_b$ is the lookback length, and $h$ denotes the prediction horizon. The score is computed as the average prediction error over real extreme-target windows and prediction horizons:
\begin{equation}
\mathrm{EPred}
=
\frac{1}{5|\mathcal{W}_{\mathrm{ext}}|}
\sum_{h=1}^{5}
\sum_{(t,d)\in \mathcal{W}_{\mathrm{ext}}}
\ell\left(
\hat{x}_{t+h,d},
x_{t+h,d}
\right),
\end{equation}
where $\mathcal{W}_{\mathrm{ext}}$ denotes the set of real test windows whose prediction targets are extreme time points, and $\ell(\cdot,\cdot)$ is the prediction loss, such as MAE or MSE. Lower values indicate that the generated data better preserves local predictive dynamics around extreme events.
\item \textbf{ERecon:} 
Evaluates whether generated data preserves reconstructive patterns around extreme-event regions. An autoencoder, consisting of an encoder $E_\phi$ and a decoder $D_\psi$, is trained on generated windows that contain at least one extreme time point. It is then tested on real sequences, with the reconstruction error evaluated only around real extreme-event regions. For each real sequence, the reconstruction is obtained as:
\begin{equation}
\hat{\mathbf{x}}^{(i)}
=
D_\psi\left(E_\phi(\mathbf{x}^{(i)})\right).
\end{equation}
The score is computed as the average reconstruction error over real extreme-event regions:
\begin{equation}
\mathrm{ERecon}
=
\frac{1}{|\mathcal{E}_{\text{real}}|}
\sum_{e \in \mathcal{E}_{\text{real}}}
\mathrm{ReconErr}(e),
\end{equation}
where $\mathrm{ReconErr}(e)$ denotes the mean squared reconstruction error over the temporal region of extreme event $e$. Lower values indicate better reconstruction of extreme-event patterns.
\end{itemize}

\subsection{Baseline Description} \label{Appendix:Baseline}
In E4GEN, we consider nine baseline methods from six different categories. An overview of these baselines is provided in Table~\ref{tab:baseline_statistics}. Their implementations with detailed parameter settings are included in the \texttt{Baseline} directory of our codebase. The details are summarized as follows:

\begin{table*}[htbp]
\centering
\caption{Overview of baseline methods compared in this paper.}
\label{tab:baseline_statistics}
\renewcommand{\arraystretch}{1.12}
\footnotesize
\begin{tabular}{llll}
\toprule
Baseline & Category & Venue & Link \\
\midrule
TimeGAN~\cite{yoon2019time} & GAN & NeurIPS'19 &
\href{https://github.com/AlexanderVNikitin/tsgm/blob/main/tsgm/models/timeGAN.py}{Official Implementation} \\

\arrayrulecolor{gray!60}\cmidrule(lr){1-4}\arrayrulecolor{black}

TimeVAE~\cite{desai2021timevae} & VAE & ArXiv'21 & 
\href{https://github.com/wangyz1999/timeVAE-pytorch}{Official Implementation} \\
koVAE~\cite{naimangenerative} & VAE & ICLR'24 & 
\href{https://github.com/azencot-group/KoVAE}{Official Implementation} \\

\arrayrulecolor{gray!60}\cmidrule(lr){1-4}\arrayrulecolor{black}

F-Flow~\cite{alaa2021generative} & Flow & ICLR'21 & 
\href{https://github.com/ahmedmalaa/Fourier-flows}{Official Implementation} \\

\arrayrulecolor{gray!60}\cmidrule(lr){1-4}\arrayrulecolor{black}

DiffWave~\cite{kongdiffwave} & Diffusion & ICLR'21 & 
\href{https://github.com/lmnt-com/diffwave}{Official Implementation} \\
Diffusion-TS~\cite{yuan2024diffusionts} & Diffusion & ICLR'24 & 
\href{https://github.com/Y-debug-sys/Diffusion-TS}{Official Implementation} \\

\arrayrulecolor{gray!60}\cmidrule(lr){1-4}\arrayrulecolor{black}

SDForger~\cite{rousseau2025forging} & LLM & NeurIPS'25 & 
\href{https://github.com/IBM/fms-dgt/tree/main/fms_dgt/public/databuilders/time_series}{Official Implementation} \\

\arrayrulecolor{gray!60}\cmidrule(lr){1-4}\arrayrulecolor{black}

FIDE~\cite{galib2024fide} & Extreme-Aware & NeurIPS'24 & 
\href{https://github.com/galib19/FIDE}{Official Implementation} \\
HeavyDiff~\cite{pandey2025heavy} & Extreme-Aware & ICLR'25 & 
\href{https://github.com/Y-debug-sys/Diffusion-TS}{Official Implementation} \\

\bottomrule
\end{tabular}
\renewcommand{\arraystretch}{1.0}
\end{table*}

\begin{itemize}
    \item \textbf{TimeGAN}~\cite{yoon2019time}: For the TimeGAN baseline, we adopt the implementation provided by the TSGM library~\cite{nikitin2024tsgm}, a unified framework for synthetic time-series generation and evaluation. This choice is mainly motivated by reproducibility and engineering compatibility, since the original TimeGAN codebase is built in an early TensorFlow style and imposes strict constraints on software versions, dependencies, and training pipeline organization.
    
    In our implementation, the TimeGAN model is built using the TSGM library through \texttt{TimeGAN(...)}. The sequence length $T$ and feature dimension $C$ are automatically inferred from the shape of the input data. The model adopts a GRU-based architecture with hidden dimension 24, 3 recurrent layers, batch size 128, and $\gamma=1.0$. After initialization, the model is configured through the \texttt{compile} interface, where separate Adam optimizers with learning rate $10^{-3}$ are assigned to the discriminator, generator, embedding module, supervised component, and autoencoder component. The loss design follows the standard TimeGAN formulation, using mean squared error for embedding and reconstruction objectives and binary cross-entropy for adversarial discrimination.
    
    TSGM encapsulates the multi-stage optimization strategy of TimeGAN, and the actual training process proceeds through embedding network training, supervised training, and joint adversarial training for 100 epochs. Once training is completed, the generator weights are stored as the reusable component for later synthesis. During generation, random noise sampled from a uniform distribution is fed into the generator batch by batch, which supports arbitrary output sample sizes without requiring exact batch divisibility. For the sample-only setting, the generator graph is first built through a zero-epoch initialization step before loading pretrained weights, after which synthetic sequences are produced directly. The generated outputs are finally mapped back to the original scale through inverse min-max normalization and saved in the same shape as the input series.

    \item \textbf{TimeVAE}~\cite{desai2021timevae}: For the TimeVAE baseline, we use a PyTorch reimplementation of the original TimeVAE model proposed in~\cite{desai2021timevae}. The corresponding code link is provided in Table~\ref{tab:baseline_statistics}. Similar to TimeGAN, we do not use the original TensorFlow implementation and instead adopt a more recent PyTorch version in our experiments.
    
    In our implementation, TimeVAE is used as a variational autoencoder for time-series generation. The encoder maps each input sequence into a latent representation through stacked 1D convolutional layers, while the decoder reconstructs the sequence through an additive decomposition including a global level term, a polynomial trend term, and a residual branch for non-structured variations. The model is trained with a weighted VAE objective consisting of the reconstruction loss and the standard KL regularization. During sampling, latent variables are drawn from a standard Gaussian prior and decoded into synthetic sequences, which are then transformed back to the original data scale.
    
    For the parameter settings, each dataset is organized as sliding-window samples of shape \((N,T,C)\) in \texttt{.npz} format. We use a 90\%/10\% train-validation split and fit MinMax normalization only on the training set. The latent dimension is set to 8, the hidden dimensions are \([50,100,200]\), the reconstruction weight is 3.0, and the batch size is 16. We enable the residual connection and use a second-order polynomial trend term, while disabling the seasonal component. The model is trained for 200 epochs with Adam and a fixed random seed of 123, and the number of generated samples is matched to the size of the training set.

    \item \textbf{koVAE}~\cite{naimangenerative}: For the koVAE baseline, we reproduce the model using the authors' Official Implementation. KoVAE is a variational autoencoder for time-series generation that replaces the conventional static latent prior with a Koopman-inspired dynamical prior. Specifically, it models the latent prior evolution through a linear transition map, so that temporal dynamics in latent space follow a structured linear system. This design enables koVAE to better capture temporal regularities in time series and supports both regular and irregularly sampled data. By combining the VAE framework with Koopman-inspired latent dynamics, koVAE provides a principled approach to realistic time-series generation.
    
    For the model configuration, following the Official Implementation, all channels are independently normalized to \([0,1]\) using Min--Max scaling before training, and the generated samples are mapped back to the original scale afterward. We adopt the regular setting of koVAE, in which both the encoder and decoder are implemented as bidirectional GRUs. The encoder uses 3 layers with hidden dimension 20 and maps inputs to a latent space of dimension 16, while the decoder uses a linear output head with sigmoid activation to match the normalized target space. Batch normalization is enabled by default. The training objective consists of a reconstruction term, a KL regularization term, and a Koopman prior prediction term, with weights \(w_{\mathrm{rec}}=1.0\), \(w_{\mathrm{kl}}=0.007\), and \(w_{\mathrm{pred\_prior}}=0.005\), respectively. The prior prediction loss is computed with 1-step rollout, and the Koopman operator is estimated using the default QR-based solver. For optimization, we train the model with Adam for 100 epochs using a learning rate of \(7\times10^{-4}\), zero weight decay, a batch size of 64, and random seed 10. During generation, the trained model samples latent trajectories and decodes them into synthetic sequences with the same sample size as the training set.

    \item \textbf{F-Flow}~\cite{alaa2021generative}: For the F-Flow baseline, we reproduce the model using the authors' Official Implementation. F-Flow is a normalizing-flow-based model for time-series generation that operates in the frequency domain instead of directly modeling raw sequences in the time domain. Specifically, it first transforms each time series into its Fourier representation, and then learns an explicit likelihood over the resulting spectral coefficients using invertible flows. By modeling data in the spectral domain, F-Flow is better suited to capturing global temporal patterns and periodic structures in time series. This design provides a principled explicit-density framework for realistic time-series generation.
    
    For the model configuration, following the Official Implementation, all channels are independently normalized to \([0,1]\) using Min--Max scaling before training. Each time series is then flattened into a one-dimensional vector and transformed into the frequency domain, where F-Flow performs density estimation over the spectral representation. To satisfy the model requirement for the Fourier transform size, an additional zero is prepended when the flattened sequence length is even. We adopt the default Fourier Flow setting, with 5 flow layers and hidden dimension 200. During training, the spectral coefficients are standardized using dataset-level mean and standard deviation statistics. For optimization, we train the model with Adam for 1000 epochs using a learning rate of \(10^{-3}\), together with an exponential learning rate decay. During generation, samples are first drawn from a standard Gaussian distribution, then mapped through the inverse flow and inverse Fourier transform back to the time domain, and finally reshaped to the original sequence format. We generate the same number of synthetic samples as in the training set for downstream evaluation.

    \item \textbf{DiffWave}~\cite{kongdiffwave}: For the DiffWave baseline, we adapt the authors' Official Implementation from raw speech waveform synthesis to unconditional univariate time-series generation. DiffWave is a diffusion-based generative model that progressively transforms Gaussian noise into realistic sequences through a predefined reverse diffusion process, using a one-dimensional convolutional denoising network as its backbone. Since the original implementation is designed for speech data, we replace the audio-based pipeline with direct \texttt{.npy} loading and use samples of shape \((N,T,1)\) as both input and output. With these changes, DiffWave serves as a diffusion-based baseline for unconditional time-series generation in our benchmark.
    
    For the model configuration, the training data are stored in a single \texttt{.npy} file and normalized using dataset-level mean and standard deviation statistics computed over all values. During training, a random subsequence of length \(T\) is extracted from each series, while shorter sequences are right-padded with zeros before cropping. We retain the original DiffWave backbone, including 30 residual layers, 64 residual channels, a dilation cycle length of 10, and the default diffusion noise schedule with 50 training steps. The model is trained in the unconditional setting with the spectrogram input fixed to \texttt{None}. For optimization, we use Adam with a learning rate of \(2\times10^{-4}\) and an \(L_1\) loss for noise prediction. During generation, samples are initialized from standard Gaussian noise, refined through reverse diffusion, and then transformed back to the original scale using the stored normalization statistics. For downstream evaluation, we generate the same number of synthetic samples as in the training set.

    \item \textbf{Diffusion-TS}~\cite{yuan2024diffusionts}: For the Diffusion-TS baseline, we adopt the official open-source implementation and retain its original diffusion training and sampling procedure. In our codebase, the model is instantiated directly from the YAML configuration and trained on sequences normalized to $[-1,1]$. The denoising network follows the interpretable decomposition design of Diffusion-TS, where the model reconstructs the clean target $x_0$ from noisy inputs and explicitly decomposes the prediction into trend and seasonality components. During training, a timestep is randomly sampled for each batch, noise is injected according to the forward diffusion process, and the network is optimized to recover the target sequence under the corresponding denoising objective. For generation, we use the EMA model to perform the reverse diffusion process and synthesize time series samples.
    
    For the model configuration, the sequence length and feature dimension are specified according to the input data. The Transformer backbone is built with 2 encoder layers and 2 decoder layers, a hidden dimension of 64, and 4 attention heads. The diffusion process uses 1000 timesteps with a cosine noise schedule, and the same number of steps is used during sampling. We adopt the $L_1$ loss as the reconstruction objective. For optimization, the model is trained with Adam using a base learning rate of $10^{-5}$, a batch size of 64, and gradient accumulation every 2 steps. To further improve training stability, we employ EMA with decay 0.995 and update interval 10, together with a warmup-based ReduceLROnPlateau scheduler, where the warmup learning rate is set to $8\times10^{-4}$ for the first 500 steps. The total training budget is 10{,}000 optimization steps.

    \item \textbf{SDForger}~\cite{rousseau2025forging}: For the SDForger baseline, we follow its original framework for synthetic time-series generation with large language models. Its central idea is to map raw time series into compact functional embeddings, cast sequence generation as structured text generation, and then decode the generated embeddings back into time-series samples. By operating in a lower-dimensional representation space, SDForger avoids directly generating long raw sequences and instead leverages the generative capability of language models for time-series synthesis.
    
    In our implementation, the input sequences are processed under the multisample setting and embedded using FastICA, where the embedding dimension is determined automatically under a variance retention target of 0.7. These embeddings are then used to fine-tune an autoregressive language model, and the generated embeddings are mapped back to the original time-series space through inverse embedding and inverse normalization. The number of generated samples is set to match the size of the training set, and \texttt{norms\_diversity\_threshold} is set to 0.

    \item \textbf{FIDE}~\cite{galib2024fide}: For the FIDE baseline, we follow its original extreme-aware diffusion framework for time-series generation. Its core idea is to improve tail preservation through frequency inflation in the frequency domain and further enhance extreme-event modeling by conditioning generation on block maxima. Since the original FIDE codebase is not directly runnable on some datasets, we make a lightweight revision to adapt it to our benchmark.
    
    In our implementation, FIDE is used as a conditional diffusion model, where each training window is paired with its block maximum and a GEV model is fitted on the observed maxima to characterize the extreme-value distribution. During training, the model takes the noisy sequence together with the corresponding block-maximum condition as input, and is optimized by the DDPM loss with an additional GEV-based regularization term. During sampling, new block-maxima conditions are drawn from the fitted GEV model to guide reverse diffusion, so that the generated sequences better preserve realistic peak behavior.
    
    For the parameter settings, we train the model for 400 epochs with batch size 2000, using hidden dimension 64 on univariate sequences. The diffusion process uses 100 steps with a linear noise schedule from $\beta_{\mathrm{start}}=10^{-4}$ to $\beta_{\mathrm{end}}=0.2$, together with correlated Gaussian-process noise with $\sigma=0.05$. We also retain the original frequency-enhancement strategy, where the top 20\% high-frequency components are amplified by a factor of 1.1 before training and inversely transformed after generation.
    
    \item \textbf{HeavyDiff}~\cite{pandey2025heavy}: As no Official Implementation of HeavyDiff is publicly available, and the original method is not specifically designed for time-series generation, we adapt its core idea within Diffusion-TS, a state-of-the-art framework for general time-series generation. Specifically, following the key idea of HeavyDiff, we replace the commonly used light-tailed Gaussian distribution in diffusion models with a heavy-tailed Student-\(t\) distribution in Diffusion-TS, thereby obtaining our HeavyDiff baseline.
    
    In our implementation, the heavy-tailed Student-\(t\) noise is applied consistently in both the forward diffusion and reverse sampling processes. Instead of directly sampling Gaussian noise, we construct the injected noise through a Gaussian scale mixture:
    \[
    \epsilon = \frac{z}{\sqrt{\kappa / \nu}}.
    \]
    Here, \(\epsilon\) denotes the sampled heavy-tailed noise, \(z \sim \mathcal{N}(0,I)\) is a standard Gaussian random variable, \(\kappa \sim \chi^2_{\nu}\) is a chi-square random variable, and \(\nu\) is the degree of freedom controlling the tail heaviness of the Student-\(t\) distribution, where a smaller \(\nu\) produces heavier tails. To reduce the variance shift introduced by heavy-tailed noise, we further apply an optional variance correction:
    \[
    \epsilon \leftarrow \epsilon \cdot \sqrt{\frac{\nu - 2}{\nu}}, \qquad \nu > 2.
    \]
    This correction rescales the sampled noise so that its variance is numerically closer to that of the standard Gaussian setting, which improves training stability and comparability. The same Student-\(t\)-based noise mechanism is used during reverse stochastic sampling to maintain consistency between training and inference.
    
    For the parameter settings, we enable heavy-tailed diffusion in the main experiments by setting the degree of freedom of the Student-\(t\) distribution to \(\nu = 2.5\), and apply variance correction by default.
        
\end{itemize}

\subsection{Interpretable Generation Dynamics}
\label{Appendix:Interpretable_Generation_Dynamics}
In this section, we provide an interpretable analysis of how extreme events emerge and are controlled during the denoising process, which answers RQ4 in Section~\ref{sec:evaluation}. To provide an intuitive illustration of the effect of the control signal in E4GEN, we show an example and trace the generation trajectory. Specifically, we examine a denoising process in E4GEN in which the control signal is activated at \textbf{step 100}, and visualize how the overall distribution and extreme-value statistics evolve across representative denoising steps. As shown in Fig.~\ref{fig:interpretable_generation_dynamics}, the trajectory covers both the stages before control activation and the stages after the control signal is introduced, including step 101 immediately before activation and step 100 as the activation step. The results show that the generated distribution is progressively transformed from the initial Gaussian distribution toward the target data distribution, while the extreme values are enhanced after the control signal is activated.

We further analyze how extreme values evolve during the denoising process. For intuitive visualization, we report the proportion of extreme-value points along the denoising trajectory. Although this statistic does not fully capture the formation of complete extreme events, it provides a clear view of how the extreme-event control signal guides generation in E4GEN. As shown in Fig.~\ref{fig:E4GEN_extreme_proportion_P}, after the control signal is activated at $t=100$, the proportion of extreme-value points increases rapidly and moves toward the target extreme-value level, indicating that E4GEN can effectively amplify extreme regions during the refinement stage. In contrast, Fig.~\ref{fig:Diffusion_extreme_proportion} shows that, without the extreme-event control signal, the extreme-value proportion increases only gradually and remains consistently lower than the target distribution. This comparison demonstrates that the proposed control signal provides explicit and reliable guidance for extreme-aware generation, enabling E4GEN to better follow the target extreme-event distribution.

\begin{figure}[htbp]
    \centering
    \begin{subfigure}{0.9\linewidth}
        \centering
        \includegraphics[width=\linewidth]{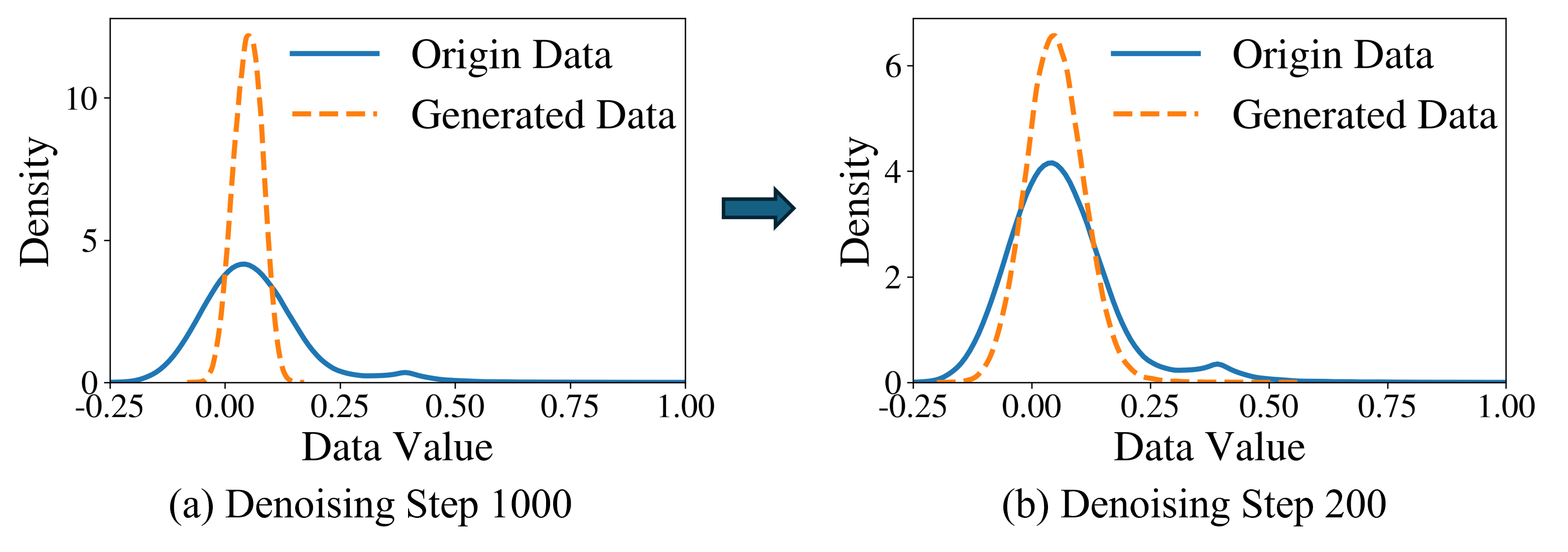}
    \end{subfigure}
    
    \vspace{4pt}
    
    \begin{subfigure}{0.9\linewidth}
        \centering
        \includegraphics[width=\linewidth]{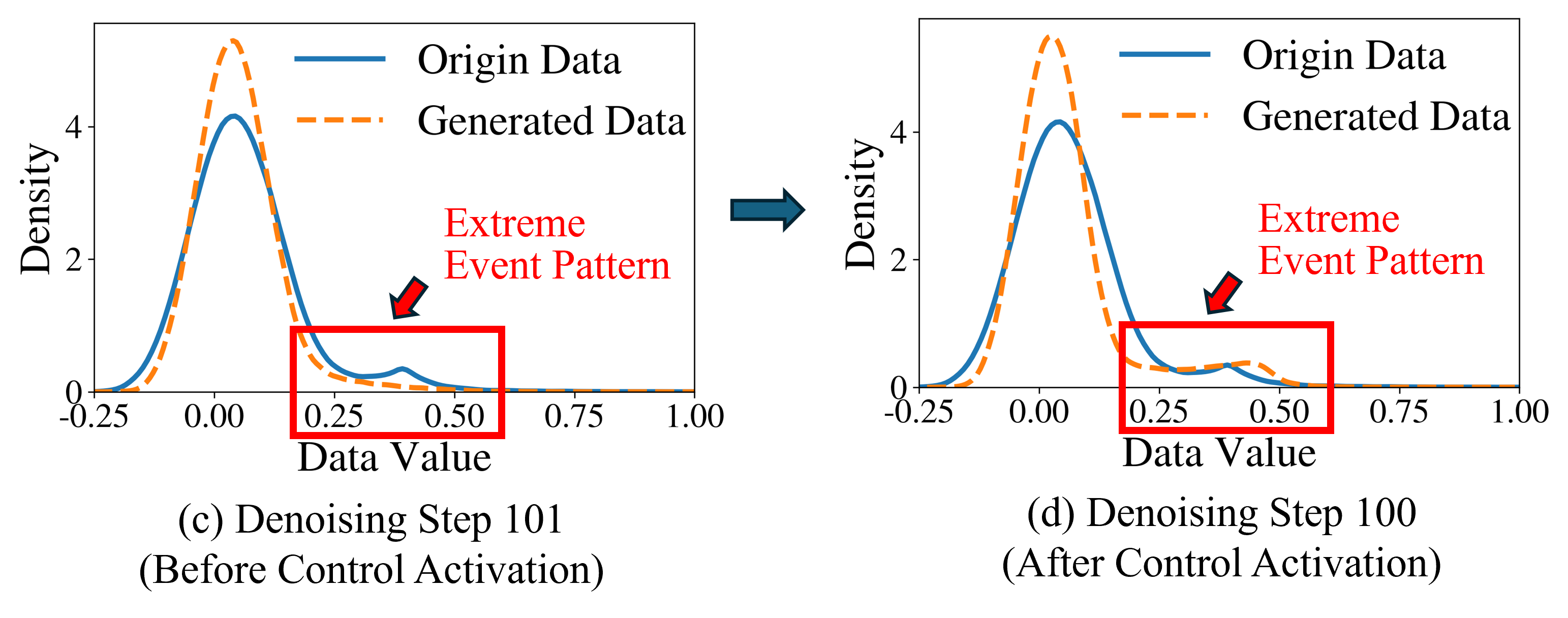}
    \end{subfigure}
    
    \vspace{4pt}
    
    \begin{subfigure}{0.9\linewidth}
        \centering
        \includegraphics[width=\linewidth]{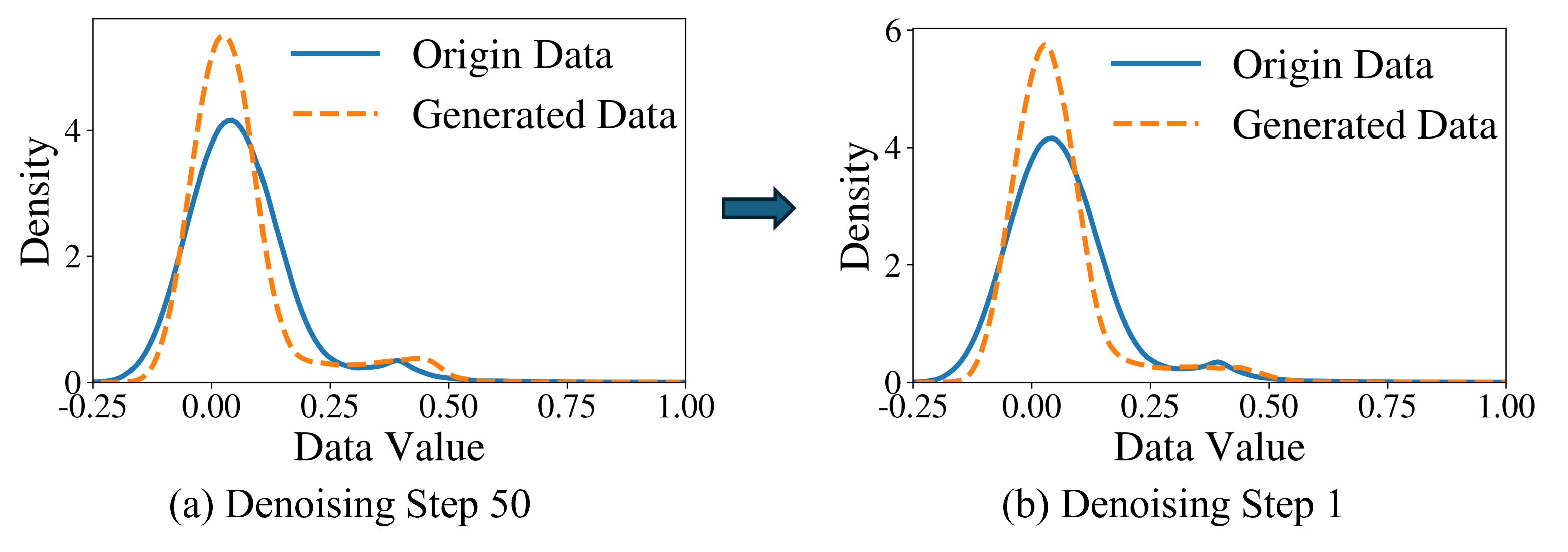}
    \end{subfigure}
    
    \caption{Example for interpretable generation dynamics. The six figures visualize the evolution of the overall distribution and extreme-value statistics along the denoising trajectory at steps 1000, 200, 101, 100, 50, and 1. The control signal is activated at step 100, where step 101 denotes the stage immediately before control activation and step 100 denotes the activation step. Overall, the trajectory shows how the generated distribution is gradually transformed from the initial Gaussian distribution toward the target data distribution, while extreme values are enhanced after control activation.}
    \label{fig:interpretable_generation_dynamics}
\end{figure}

\begin{figure}[htbp]
    \centering
    \begin{subfigure}{0.45\linewidth}
        \centering
        \includegraphics[width=\linewidth]{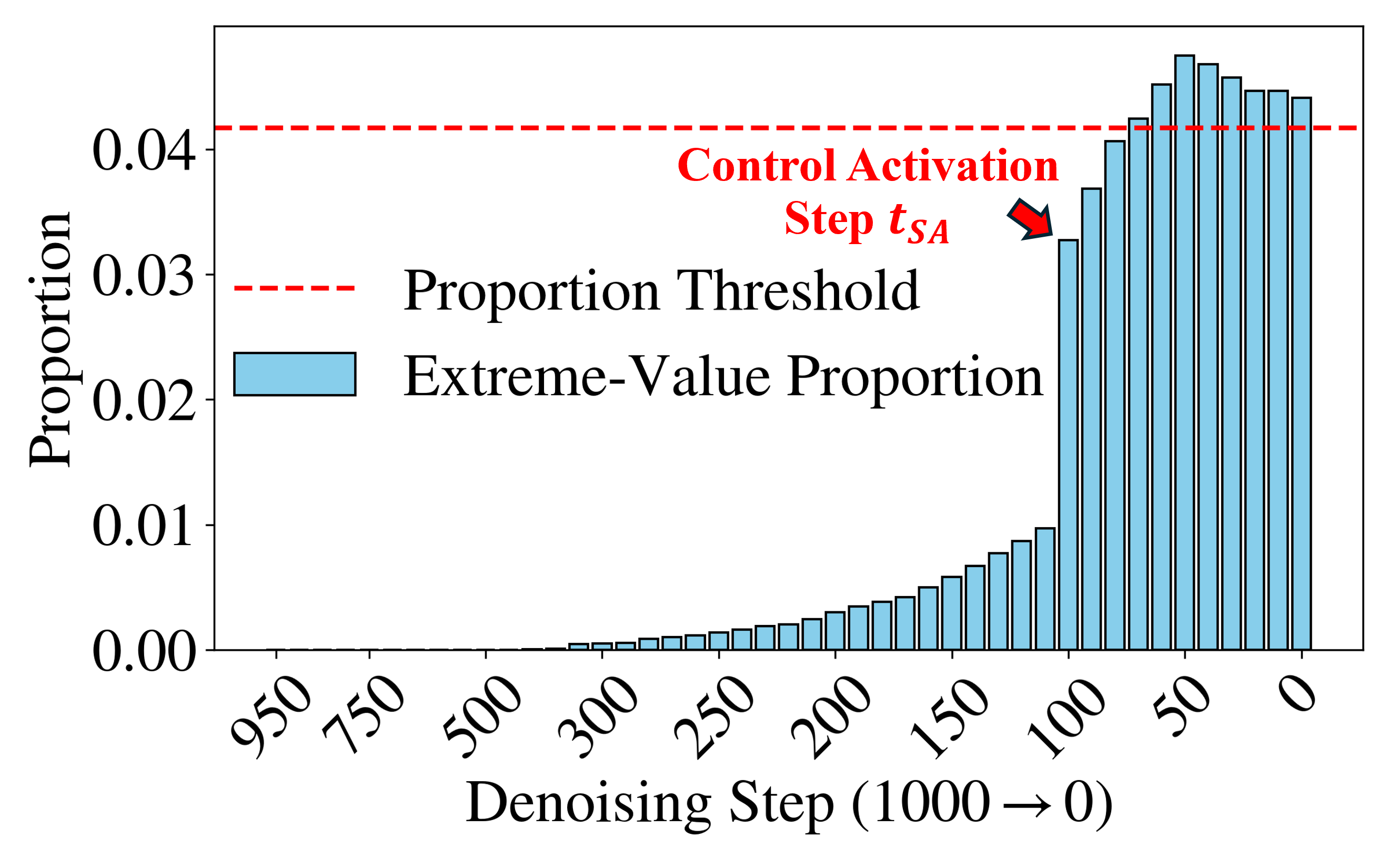}
        \caption{Evolution of the extreme-value point proportion with the extreme-event control signal. The curve shows that E4GEN progressively increases the proportion of extreme-value points during denoising, illustrating the guidance effect of the proposed control signal. After the control signal is activated at $t=100$, the extreme-value proportion rises rapidly, demonstrating the effectiveness and reliability of E4GEN in guiding extreme-event generation.}
        \label{fig:E4GEN_extreme_proportion_P}
    \end{subfigure}
    \hspace{10pt}
    \begin{subfigure}{0.45\linewidth}
        \centering
        \includegraphics[width=\linewidth]{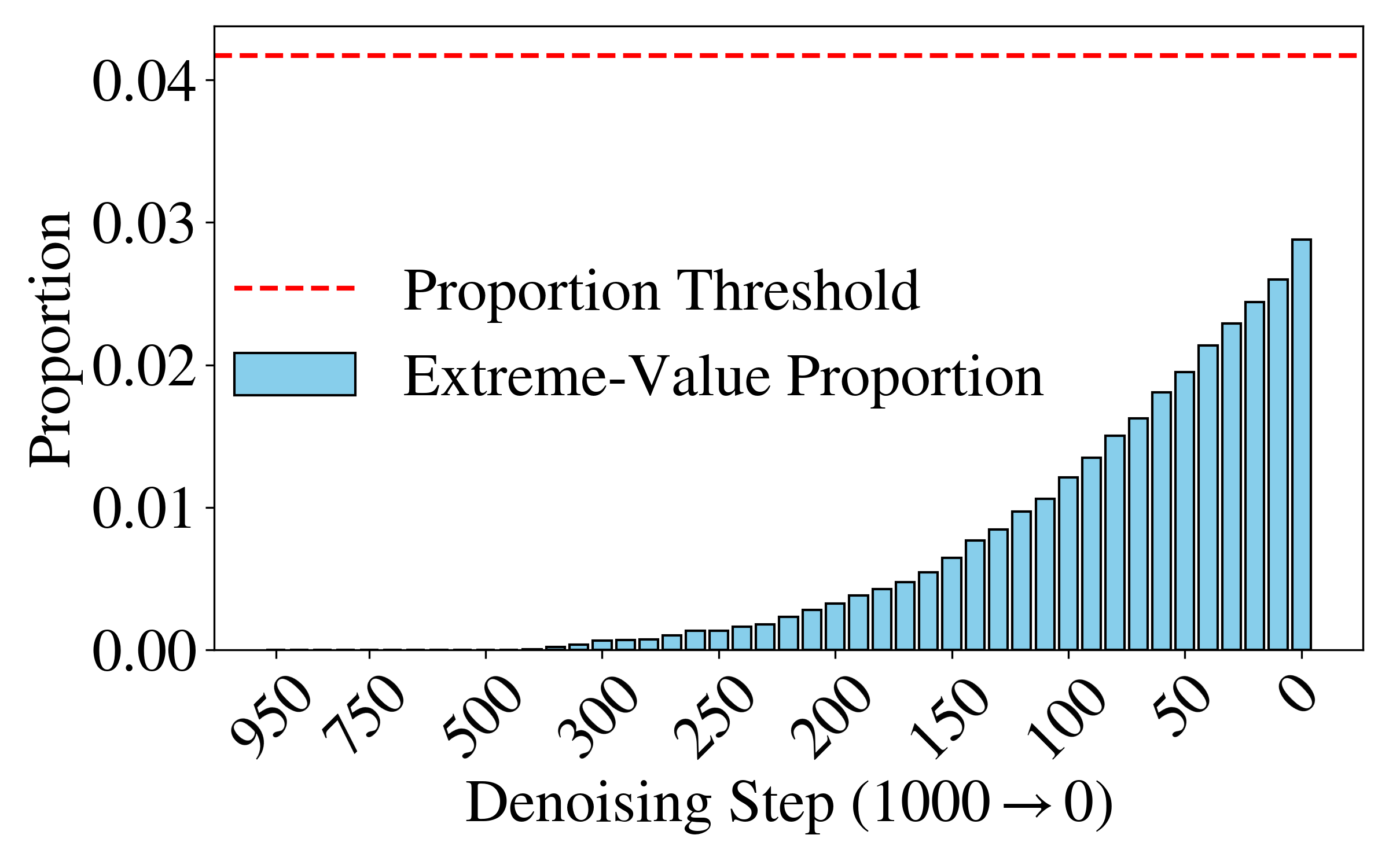}
        \caption{Evolution of the extreme-value point proportion without the extreme-event control signal. The curve shows that the proportion of extreme-value points increases gradually during denoising, reflecting the model's inherent tendency to recover extreme values. However, without external extreme-event control, the extreme-value proportion remains consistently suppressed, indicating that existing methods struggle to follow the target extreme-event distribution.}
        \label{fig:Diffusion_extreme_proportion}
    \end{subfigure}
    
    \caption{Comparison of the evolution of extreme-value point proportions during the denoising process, with and without the extreme-event control signal.}
    \label{fig:extreme_proportion_comparison}
\end{figure}

\subsection{E-Predictor Intermediate Output Analysis}
\label{Appendix:E_Predictor_Intermediate_Output_Analysis}

To answer RQ5, this part analyzes the accuracy of E-Predictor in predicting extreme-event semantics during generation. E-Predictor bridges semantic characterization and control injection by converting intermediate denoising states into explicit extreme-event semantics, which are then used by E-Control as control signals. Therefore, we examine whether E-Predictor produces accurate and semantically meaningful intermediate outputs, thereby validating the reliability of the predicted control signals for extreme-event generation.

Figure~\ref{fig:EPredictor_Visualization} provides qualitative examples of E-Predictor outputs on three representative datasets, namely Syn-Data, LTST-ECG, and HH-Power, with one time-series sample visualized for each dataset. In each example, the black curve denotes the ground-truth signal, red points indicate the ground-truth extreme-event regions, blue shaded areas represent the predicted extreme-event regions, and blue triangles mark the predicted peak positions. Overall, the results show that E-Predictor can effectively capture the presence of extreme events across different datasets. When an extreme event is successfully detected, the predicted event span and peak position are generally well aligned with the ground truth, indicating accurate recovery of the location and intensity semantics in EES. The predicted regions also preserve the main temporal extent of the event, reflecting meaningful temporal-shape semantics for subsequent control injection. Meanwhile, a small number of false negatives and false positives can still be observed, suggesting that E-Predictor is not error-free. Nevertheless, the overall alignment between predicted and ground-truth event semantics supports the reliability of the predicted EES as control signals for E-Control.

\begin{figure}[htbp]
    \centering
    
    \begin{subfigure}{0.8\linewidth}
        \centering
        \includegraphics[width=0.9\linewidth]{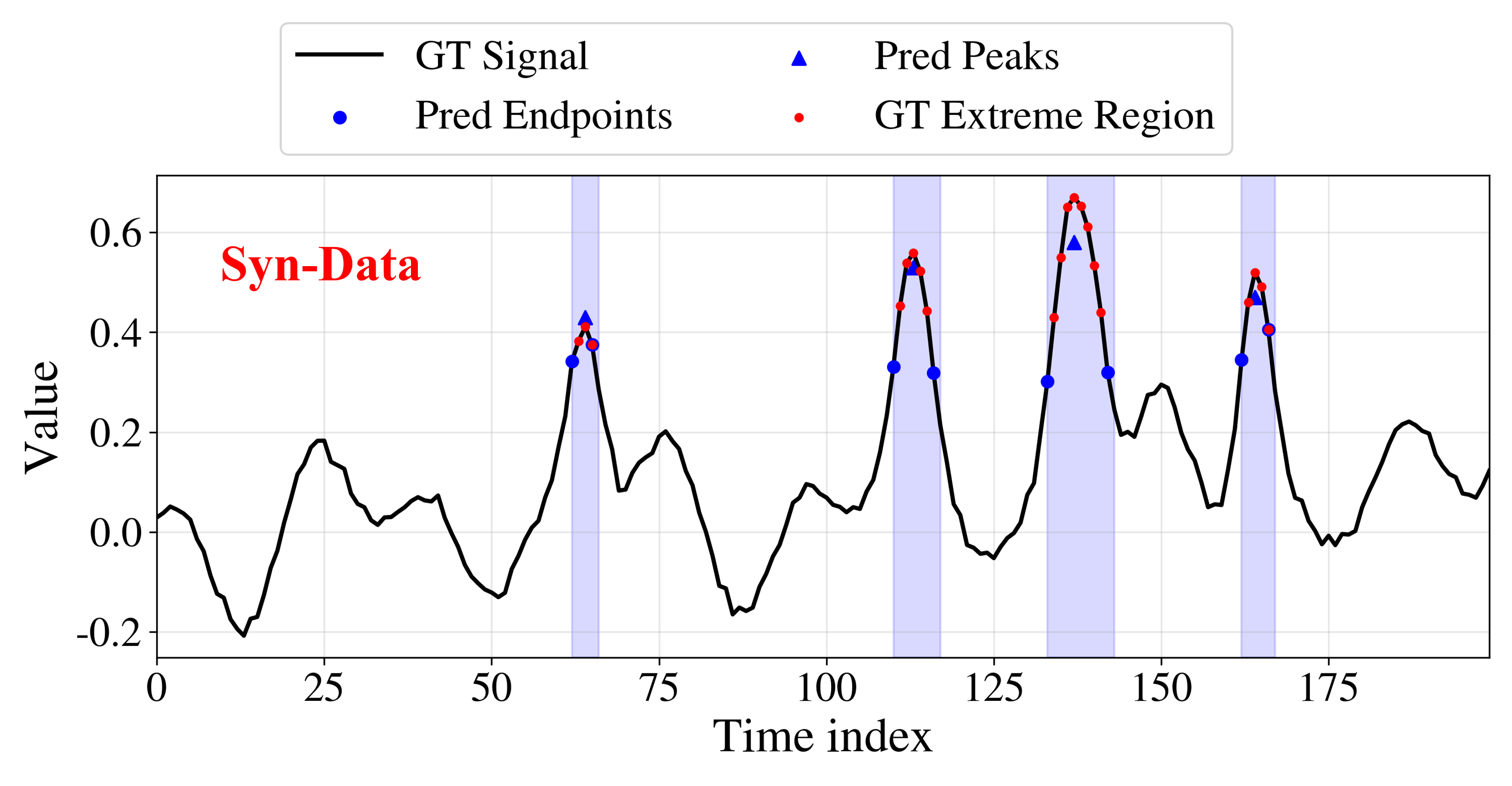}
        \label{fig:EPredictor_Syn}
    \end{subfigure}
    
    \vspace{4pt}
    
    \begin{subfigure}{0.8\linewidth}
        \centering
        \includegraphics[width=0.9\linewidth]{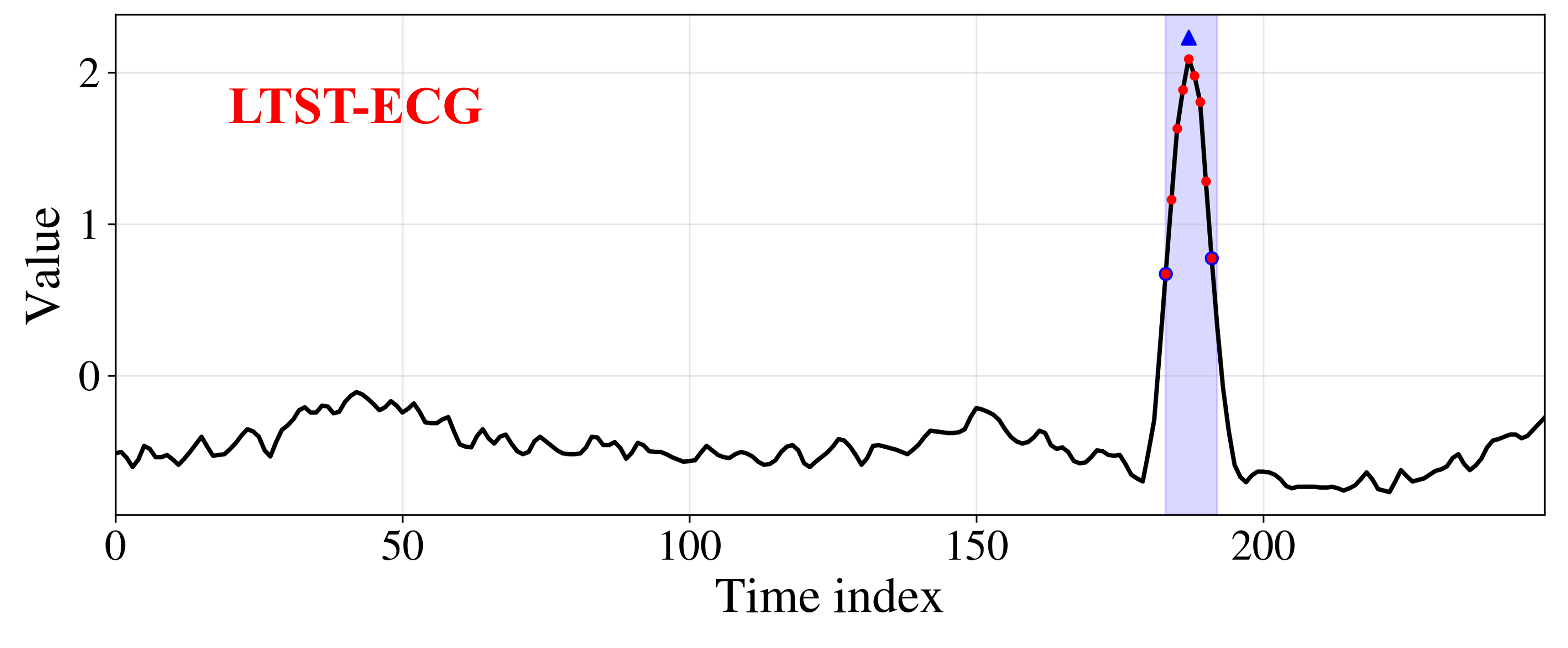}
        \label{fig:EPredictor_ECG}
    \end{subfigure}
    
    \vspace{4pt}
    
    \begin{subfigure}{0.8\linewidth}
        \centering
        \includegraphics[width=0.9\linewidth]{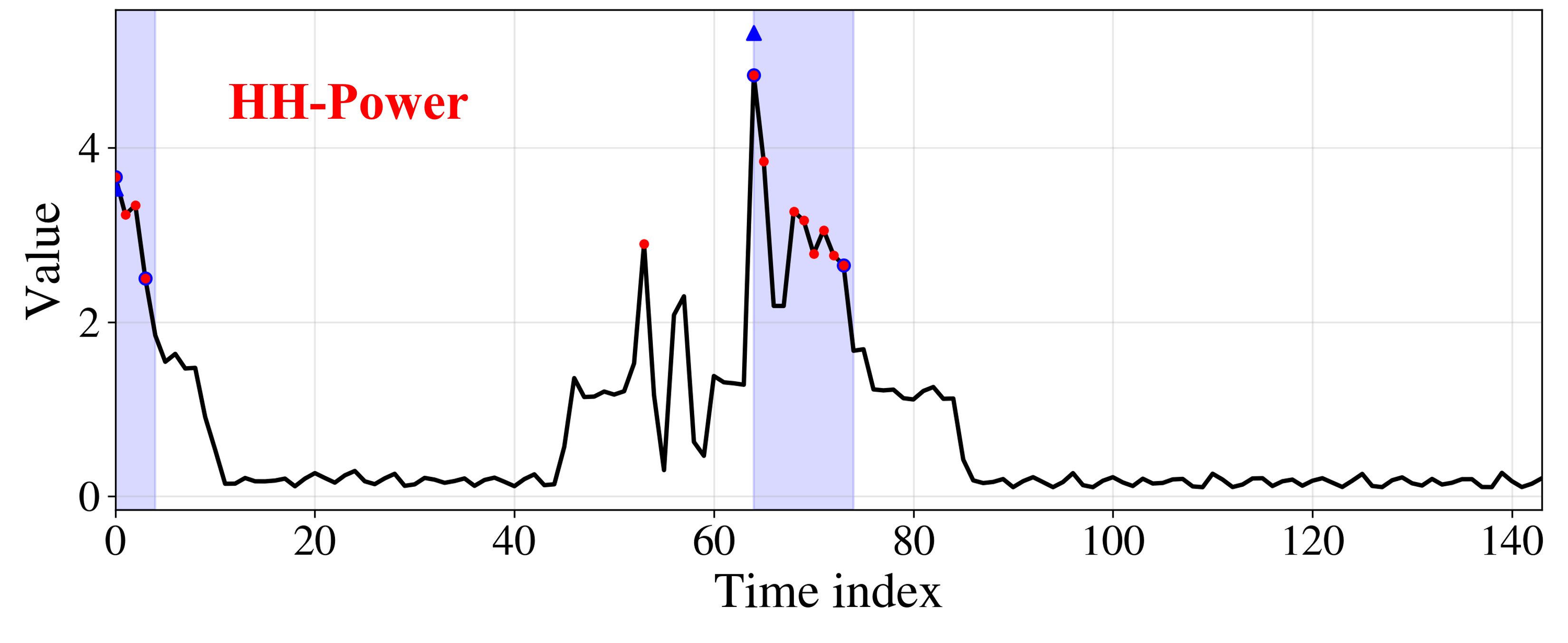}
        \label{fig:EPredictor_HHPower}
    \end{subfigure}
    
    \caption{Visualization of E-Predictor predictions for extreme-event semantics on Syn-Data, LTST-ECG, and HH-Power. For each dataset, one time-series sample is shown. The black curve denotes the ground-truth signal, red points indicate the ground-truth extreme-event regions, blue shaded areas represent the predicted extreme-event regions, and blue triangles mark the predicted peak positions. Overall, once E-Predictor identifies the presence of an extreme event, it can accurately localize the event region and estimate the peak position, although a small number of false negatives and false positives may still occur.}
    \label{fig:EPredictor_Visualization}
\end{figure}

Table~\ref{tab:epredictor_results} further quantifies the intermediate prediction quality of E-Predictor across all six datasets. The Event-F1 scores range from 0.713 to 0.872, showing that E-Predictor can reliably identify extreme events in most cases. The Span-IoU values range from 0.572 to 0.774, indicating that the predicted event regions substantially overlap with the ground-truth event spans, although exact boundary recovery remains challenging for some datasets. For intensity semantics, Peak-MAE and Prom-MAE remain within moderate ranges of 0.096--0.248 and 0.121--0.296, respectively, suggesting that E-Predictor can reasonably estimate the peak strength and prominence of extreme events. For temporal-shape semantics, Shape-DTW ranges from 0.214 to 0.487, showing that the predicted semantics preserve the main temporal evolution of extreme events, while more irregular events still lead to larger shape deviations. Overall, these results indicate that E-Predictor provides sufficiently accurate and semantically meaningful EES predictions to serve as control signals for E-Control, which is also consistent with the qualitative examples in Figure~\ref{fig:EPredictor_Visualization}.

\begin{table}[htbp]
\centering
\caption{Intermediate prediction quality of E-Predictor on extreme-event semantics.}
\label{tab:epredictor_results}
\footnotesize
\renewcommand{\arraystretch}{1.1}
\setlength{\tabcolsep}{6pt}
\begin{tabular}{lccccc}
\toprule
\textbf{Dataset} 
& \textbf{Event-F1} $\uparrow$
& \textbf{Span-IoU} $\uparrow$
& \textbf{Peak-MAE} $\downarrow$
& \textbf{Prom-MAE} $\downarrow$
& \textbf{Shape-DTW} $\downarrow$ \\
\midrule
\textbf{Syn-Data}  & 0.872 & 0.774 & 0.096 & 0.136 & 0.214 \\
\textbf{Wea-Temp}  & 0.801 & 0.684 & 0.157 & 0.188 & 0.331 \\
\textbf{Wea-Prec}  & 0.713 & 0.572 & 0.248 & 0.296 & 0.487 \\
\textbf{LTST-ECG}  & 0.834 & 0.728 & 0.132 & 0.121 & 0.276 \\
\textbf{HH-Power}  & 0.768 & 0.641 & 0.181 & 0.221 & 0.379 \\
\textbf{PEMS-SF}   & 0.742 & 0.618 & 0.204 & 0.247 & 0.423 \\
\bottomrule
\end{tabular}
\end{table}

\subsection{Ablation Study}
\label{Appendix:ablation}

To answer RQ6, we evaluate the contribution of each key design (extreme-event semantic representation, E-Activator, and E-Predictor) to the final generation performance. 
The three ablation variants are defined as follows:
\begin{itemize}
    \item \textbf{E4GEN$_{\setminus\mathrm{Sem}}$}: removes the extreme-event semantic representation. Instead of using the structured EES representation, this ablation variant uses the raw point sequence to represent each extreme event. When the event segment is longer than the fixed representation size, uniform sampling is applied to shorten it.
    \item \textbf{E4GEN$_{\setminus\mathrm{Act}}$}: removes E-Activator. Instead of dynamically selecting the dataset-specific Control Activation Step, this ablation variant disables step selection and keeps extreme-event enhancement active throughout the entire denoising process.
    \item \textbf{E4GEN$_{\setminus\mathrm{Pred}}$}: removes E-Predictor. Instead of predicting extreme-event semantics at the Control Activation Step, this ablation variant samples semantics directly from the empirical extreme-event semantic distribution of the training dataset.
\end{itemize}

Table~\ref{tab:overall_syn_ablation} and Table~\ref{tab:extreme_syn_ablation} report the ablation results on Syn-Data as a representative example. The full E4GEN achieves the best performance across all overall, downstream, and extreme-event metrics, showing that the three designs are jointly beneficial for extreme-aware generation. Among the ablation variants, E4GEN$_{\setminus\mathrm{Act}}$ suffers the largest degradation, indicating the importance of selecting an appropriate Control Activation Step rather than applying control throughout the entire denoising process. E4GEN$_{\setminus\mathrm{Sem}}$ and E4GEN$_{\setminus\mathrm{Pred}}$ perform better but still lag behind the full model, confirming that both structured EES representation and state-dependent semantic prediction are necessary for reliable control. Overall, the results validate the contribution of extreme-event semantic representation, E-Activator, and E-Predictor to the final generation quality.

\begin{table*}[htbp]
\centering
\footnotesize
\renewcommand{\arraystretch}{1.12}
\setlength{\tabcolsep}{2pt}

\vspace{-10pt}
\caption{Overall Generation Fidelity and General Downstream Utility on Syn-Data}
\label{tab:overall_syn_ablation}
\begin{tabular}{cclcccccccc}
\toprule
\textbf{Dataset} & \textbf{Type} & \textbf{Method} & \textbf{Wass.} & \textbf{KS} & \textbf{JS} & \textbf{MMD} & \textbf{ACD} & \textbf{Context\_FID} & \textbf{Pred} & \textbf{Recon} \\
\midrule
\multirow{4}{*}{\makecell{\textbf{Syn-}\\\textbf{Data}}}
& \multirow{3}{*}{\textbf{Ablation}}
& E4GEN$_{\setminus\mathrm{Sem}}$
& 0.0264 & 0.1097 & \underline{0.1698} & 0.1428 & \underline{0.0784} & 1.9462 & 0.1694 & \underline{0.0548} \\
&
& E4GEN$_{\setminus\mathrm{Act}}$
& 0.0418 & 0.1586 & 0.2564 & 0.2247 & 0.1369 & 4.1265 & 0.2137 & 0.0864 \\

&
& E4GEN$_{\setminus\mathrm{Pred}}$
& \underline{0.0241} & \underline{0.1016} & 0.1745 & \underline{0.1247} & 0.0809 & \underline{1.6428} & \underline{0.1648} & 0.0569 \\

\arrayrulecolor{gray!60}\cmidrule(lr){2-11}\arrayrulecolor{black}

& \textbf{Ours}
& \textbf{E4GEN}
& \textbf{0.0144} & \textbf{0.0568} & \textbf{0.1106} & \textbf{0.0447} & \textbf{0.0505} & \textbf{0.4501} & \textbf{0.1518} & \textbf{0.0368} \\
\bottomrule
\end{tabular}

\vspace{2.0em}

\caption{Extreme-Event Generation Fidelity and Extreme-Oriented Downstream Utility on Syn-Data}
\label{tab:extreme_syn_ablation}
\begin{tabular}{cclccccccccc}
\toprule
\textbf{Dataset} & \textbf{Type} & \textbf{Method} & \textbf{EM-W1} & \textbf{EC-Diff} & \textbf{ED-W1} & \textbf{EP-W1} & \textbf{EWass.} & \textbf{EJS} & \textbf{EMMD} & \textbf{EPred} & \textbf{ERecon} \\
\midrule
\multirow{4}{*}{\makecell{\textbf{Syn-}\\\textbf{Data}}}
& \multirow{3}{*}{\textbf{Ablation}}

& E4GEN$_{\setminus\mathrm{Sem}}$
& 0.1284 & 3348 & \underline{2.3567} & 0.1642 & 0.9147 & \underline{0.4628} & 0.7042 & 0.4527 & \underline{0.1038} \\
&
& E4GEN$_{\setminus\mathrm{Act}}$
& 0.2487 & 5126 & 3.2864 & 0.2765 & 1.6842 & 0.6938 & 0.9725 & 0.6247 & 0.2148 \\

&
& E4GEN$_{\setminus\mathrm{Pred}}$
& \underline{0.1096} & \underline{2976} & 2.4389 & \underline{0.1425} & \underline{0.7248} & 0.4875 & \underline{0.6236} & \underline{0.3864} & 0.1147 \\

\arrayrulecolor{gray!60}\cmidrule(lr){2-12}\arrayrulecolor{black}

& \textbf{Ours}
& \textbf{E4GEN}
& \textbf{0.0245} & \textbf{1232} & \textbf{1.5383} & \textbf{0.0360} & \textbf{0.3451} & \textbf{0.2689} & \textbf{0.3615} & \textbf{0.3129} & \textbf{0.0871} \\
\bottomrule
\end{tabular}
\renewcommand{\arraystretch}{1.0}
\normalsize
\end{table*}

\subsection{Sensitivity Analysis of Control-Related Hyperparameters}
\label{Appendix:Sensitivity Analysis}

To answer RQ7, this part provides a sensitivity analysis of the key control-related hyperparameters in E4GEN. Specifically, we examine the effects of the control activation step $t_{\mathrm{CA}}$ and the alignment start step $t_{\mathrm{AS}}$ introduced in Section~\ref{sec:DCT-NIS}. 

For $t_{\mathrm{CA}}$, 
Since E-Activator defines a valid search region $\mathcal{R}{\mathrm{CA}}$ for the control activation step $t{\mathrm{CA}}$, we focus on how to select $t_{\mathrm{CA}}$ within this region. In E4GEN, we selects $t_{\mathrm{CA}}$ through a lightweight validation search within $\mathcal{R}_{\mathrm{CA}}$. To examine whether E4GEN is sensitive to this selection, we compare the default strategy with two alternative strategies for choosing $t_{\mathrm{CA}}$ within $\mathcal{R}_{\mathrm{CA}}$:
\begin{itemize}
\item \textbf{E4GEN$_{\mathrm{Grid}}$}: uses fixed candidate steps that are uniformly spaced within $\mathcal{R}_{\mathrm{CA}}$, and reports the average performance over these candidates. This variant evaluates whether different fixed positions inside the control activation window lead to stable generation performance.

\item \textbf{E4GEN$_{\mathrm{Rand}}$}: randomly samples $t_{\mathrm{CA}}$ from $\mathcal{R}_{\mathrm{CA}}$ during each sampling process, and reports the average performance over multiple random trials. This variant evaluates whether stochastic control activation within the identified window affects generation stability.
\end{itemize}

\begin{figure}[b]
  \centering
  \includegraphics[width=0.65\columnwidth]{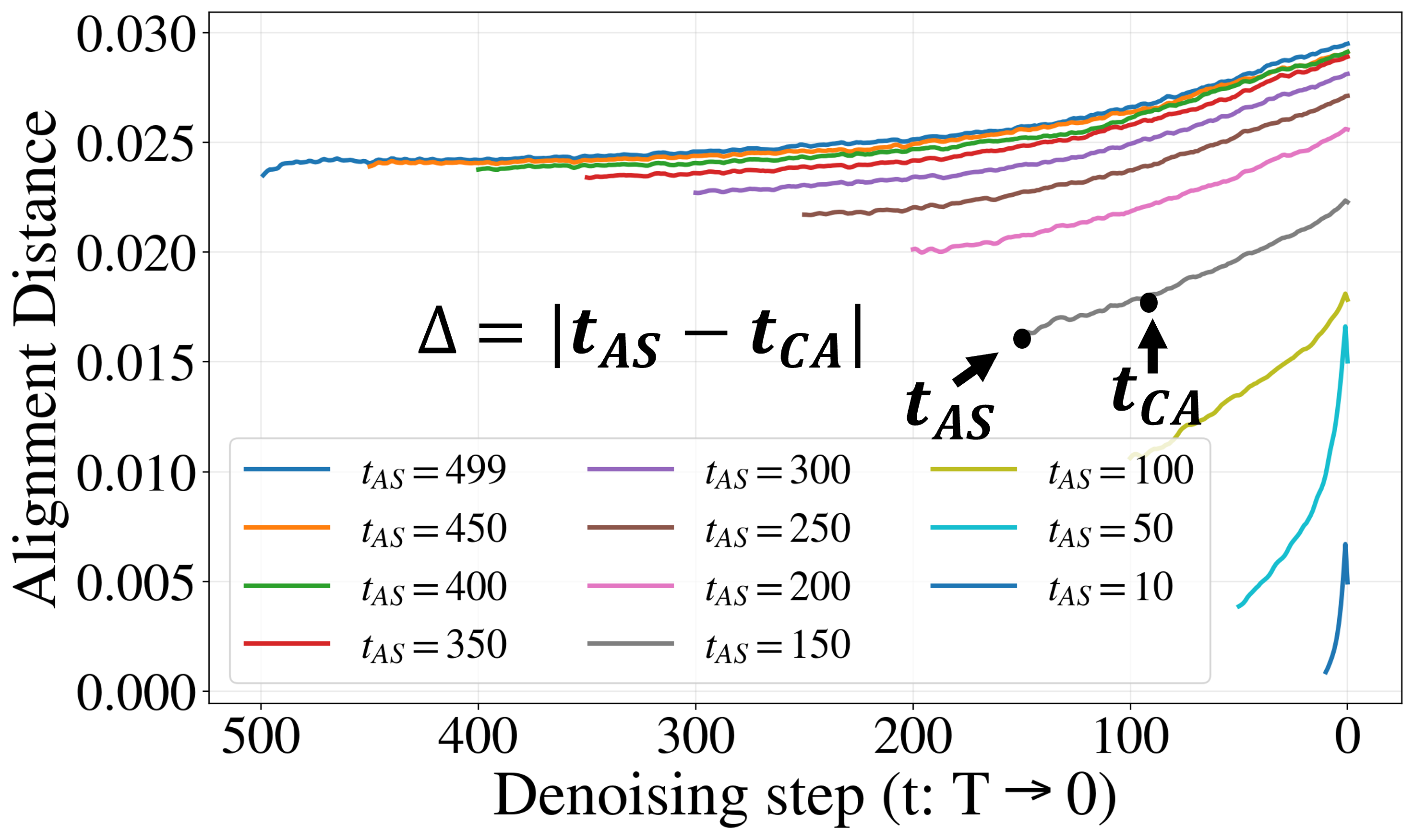}
\caption{Sensitivity analysis of E4GEN with respect to the alignment start step $t_{\mathrm{AS}}$. 
For each candidate value of $t_{\mathrm{AS}}$, we progressively vary $t_{\mathrm{CA}}$ within the control activation window $\mathcal{W}_{\mathrm{CA}}$ to increase the temporal offset $\Delta = |t_{\mathrm{AS}} - t_{\mathrm{CA}}|$. 
We measure the resulting alignment distance between the data-conditioned state distribution $\mathbf{x}_d(t_{\mathrm{CA}})$ and the corresponding noise-initiated sampling state distribution $\mathbf{x}_n(t_{\mathrm{CA}})$, where a smaller distance indicates better train--sample state alignment.}
  \label{fig:Alignment Distance}
\end{figure}

Tables~\ref{tab:tca_sensitivity_overall_syn} and~\ref{tab:tca_sensitivity_extreme_syn} report the sensitivity analysis of $t_{\mathrm{CA}}$ based on different selection strategy. Both E4GEN$_{\mathrm{Grid}}$ and E4GEN$_{\mathrm{Rand}}$ maintain competitive performance, suggesting that E4GEN is not highly sensitive to the exact activation step once a valid control activation window is identified. E4GEN$_{\mathrm{Grid}}$ consistently outperforms E4GEN$_{\mathrm{Rand}}$, indicating that fixed activation steps provide more stable control than stochastic step selection. The default E4GEN achieves the best performance across all metrics, demonstrating the advantage of selecting a better fixed $t_{\mathrm{CA}}$ through lightweight validation search.

\begin{table*}[t]
\centering
\footnotesize
\renewcommand{\arraystretch}{1.12}
\setlength{\tabcolsep}{2pt}

\caption{Overall Generation Fidelity and General Downstream Utility under Different $t_{\mathrm{CA}}$ Selection Strategies on Syn-Data}
\label{tab:tca_sensitivity_overall_syn}
\begin{tabular}{cclcccccccc}
\toprule
\textbf{Dataset} & \textbf{Type} & \textbf{Method} & \textbf{Wass.} & \textbf{KS} & \textbf{JS} & \textbf{MMD} & \textbf{ACD} & \textbf{Context\_FID} & \textbf{Pred} & \textbf{Recon} \\
\midrule
\multirow{3}{*}{\makecell{\textbf{Syn-}\\\textbf{Data}}}
& \multirow{2}{*}{\textbf{Sensitivity}}
& E4GEN$_{\mathrm{Grid}}$
& \underline{0.0162} & \underline{0.0619} & \underline{0.1187} & \underline{0.0508} & \underline{0.0534} & \underline{0.4936} & \underline{0.1576} & \underline{0.0395} \\
& 
& E4GEN$_{\mathrm{Rand}}$
& 0.0187 & 0.0684 & 0.1269 & 0.0583 & 0.0578 & 0.5724 & 0.1643 & 0.0431 \\
\cmidrule(lr){2-11}
& \textbf{Ours}
& E4GEN
& \textbf{0.0144} & \textbf{0.0568} & \textbf{0.1106} & \textbf{0.0447} & \textbf{0.0505} & \textbf{0.4501} & \textbf{0.1518} & \textbf{0.0368} \\
\bottomrule
\end{tabular}
\end{table*}

\begin{table*}[t]
\centering
\footnotesize
\renewcommand{\arraystretch}{1.12}
\setlength{\tabcolsep}{2pt}

\caption{Extreme-Event Generation Fidelity and Extreme-Oriented Downstream Utility under Different $t_{\mathrm{CA}}$ Selection Strategies on Syn-Data}
\label{tab:tca_sensitivity_extreme_syn}
\begin{tabular}{cclccccccccc}
\toprule
\textbf{Dataset} & \textbf{Type} & \textbf{Method} & \textbf{EM-W1} & \textbf{EC-Diff} & \textbf{ED-W1} & \textbf{EP-W1} & \textbf{EWass.} & \textbf{EJS} & \textbf{EMMD} & \textbf{EPred} & \textbf{ERecon} \\
\midrule
\multirow{3}{*}{\makecell{\textbf{Syn-}\\\textbf{Data}}}
& \multirow{2}{*}{\textbf{Sensitivity}}
& E4GEN$_{\mathrm{Grid}}$
& \underline{0.0318} & \underline{1396} & \underline{1.6824} & \underline{0.0437} & \underline{0.3895} & \underline{0.2916} & \underline{0.3982} & \underline{0.3265} & \underline{0.0948} \\
& 
& E4GEN$_{\mathrm{Rand}}$
& 0.0396 & 1584 & 1.8467 & 0.0521 & 0.4368 & 0.3194 & 0.4375 & 0.3417 & 0.1029 \\
\cmidrule(lr){2-12}
& \textbf{Ours}
& E4GEN
& \textbf{0.0245} & \textbf{1232} & \textbf{1.5383} & \textbf{0.0360} & \textbf{0.3451} & \textbf{0.2689} & \textbf{0.3615} & \textbf{0.3129} & \textbf{0.0871} \\
\bottomrule
\end{tabular}
\end{table*}

We further analyze the sensitivity of E4GEN to the alignment start step $t_{\mathrm{AS}}$. Since $t_{\mathrm{AS}}$ determines the starting point of the short reverse trajectory used to construct data-conditioned training states, we focus on its distance from $t_{\mathrm{CA}}$, denoted as $\Delta = |t_{\mathrm{AS}} - t_{\mathrm{CA}}|$. Specifically, we select ten candidate values of $t_{\mathrm{AS}}$ and progressively slide $t_{\mathrm{CA}}$ within the control activation window $\mathcal{W}_{\mathrm{CA}}$, thereby increasing the distance $\Delta = |t_{\mathrm{AS}} - t_{\mathrm{CA}}|$. For each pair of $(t_{\mathrm{AS}}, t_{\mathrm{CA}})$, we measure the distributional discrepancy between the resulting data-conditioned state $\mathbf{x}_d(t_{\mathrm{CA}})$ and the corresponding noise-initiated sampling state $\mathbf{x}_n(t_{\mathrm{CA}})$. This analysis evaluates how $t_{\mathrm{AS}}$ affects training-sampling alignment and whether the alignment remains stable across different values of $\Delta$.

As shown in Figure~\ref{fig:Alignment Distance}, smaller values of $t_{\mathrm{AS}}$, which are closer to the sampling endpoint, generally lead to lower alignment distance between $\mathbf{x}_d(t_{\mathrm{CA}})$ and $\mathbf{x}_n(t_{\mathrm{CA}})$. However, the alignment distance also increases more rapidly with the offset $\Delta = |t_{\mathrm{AS}} - t_{\mathrm{CA}}|$ under these settings. This indicates that later alignment start step $t_{\mathrm{AS}}$ reduce the initial train--sample discrepancy, but are more sensitive to mismatches between $\mathbf{x}_d(t_{\mathrm{CA}})$ and $\mathbf{x}_n(t_{\mathrm{CA}})$. Since $t_{\mathrm{CA}}$ is usually located at a relatively small diffusion step within the control activation window, this sensitivity analysis further supports the E4GEN design choice of selecting $t_{\mathrm{AS}}$ close to $t_{\mathrm{CA}}$, which helps reduce train--sample state discrepancy and preserve reliable state alignment.

\subsection{Efficiency Analysis}
\label{Appendix:efficiency}
This section analyzes the computational efficiency of E4GEN in terms of training time and per-sample sampling time. As shown in Table~\ref{tab:efficiency}, E4GEN is computationally efficient across all six datasets. The training time ranges from 0.09 to 0.12 seconds per epoch, while the sampling time ranges from 0.09 to 0.24 seconds per sample. For most datasets, E4GEN requires only about 0.09--0.10 seconds per training epoch and 0.13 seconds or less to generate one sample. The efficient training and sampling speeds demonstrate that the proposed control-based design is both practical and scalable, making E4GEN an effective and efficient framework for extreme-aware time-series generation.


\begin{table*}[htbp]
\centering
\footnotesize
\renewcommand{\arraystretch}{1.12}
\setlength{\tabcolsep}{2pt}

\vspace{-6pt}
\caption{Training Time and Per-Sample Sampling Time Across Datasets of E4GEN.}
\label{tab:efficiency}
\begin{tabular}{lcc}
\toprule
\textbf{Dataset} & \textbf{Training Time} (s / epoch) & \textbf{Sampling Time} (s / sample)  \\
\midrule

Syn-Data   & 0.10 & 0.19 \\
Wea-Temp   & 0.09 & 0.09 \\
Wea-Prec   & 0.10 & 0.13 \\
LTST-ECG   & 0.12 & 0.24 \\
HH-Power   & 0.09 & 0.13 \\
PEMS-SF    & 0.09 & 0.13 \\

\bottomrule
\end{tabular}

\vspace{-6pt}
\end{table*}

\subsection{Controllable Extreme-Event Generation with User-Specified Semantics}
\label{Appendix:Controllable Extreme-Event Generation}

For a time series of length 200, we manually define three user-specified extreme-event configurations as external control signals:
\begin{itemize}
    \item \textbf{Configuration 1:} two extreme events, where Event 1 spans time indices $66$--$73$ with a target peak intensity of $0.47$, and Event 2 spans time indices $120$--$180$ with a target peak intensity of $0.60$.
    \item \textbf{Configuration 2:} two extreme events, where Event 1 spans time indices $10$--$30$ with a target peak intensity of $0.47$, and Event 2 spans time indices $150$--$180$ with a target peak intensity of $0.50$.
    \item \textbf{Configuration 3:} three extreme events, where Event 1 spans time indices $52$--$58$ with a target peak intensity of $0.43$, Event 2 spans time indices $92$--$101$ with a target peak intensity of $0.52$, and Event 3 spans time indices $160$--$163$ with a target peak intensity of $0.41$.
\end{itemize}

As shown in Figure~\ref{fig:condition_generation_three}, E4GEN successfully generates time series that follow the three user-specified extreme-event configurations. Across all configurations, the generated samples exhibit clustered extreme values around the prescribed event locations, and the aggregated event regions are generally well aligned with the specified temporal spans. The generated peak intensities also remain close to the target intensity levels, indicating that E4GEN can effectively translate external semantic control signals into localized extreme-event patterns. These results demonstrate that E4GEN supports controllable extreme-event generation beyond self-predicted semantics, enabling users to guide both the occurrence and intensity of extreme events.

\begin{figure}[htbp]
    \centering

    \begin{subfigure}{0.32\linewidth}
        \centering
        \includegraphics[width=\linewidth]{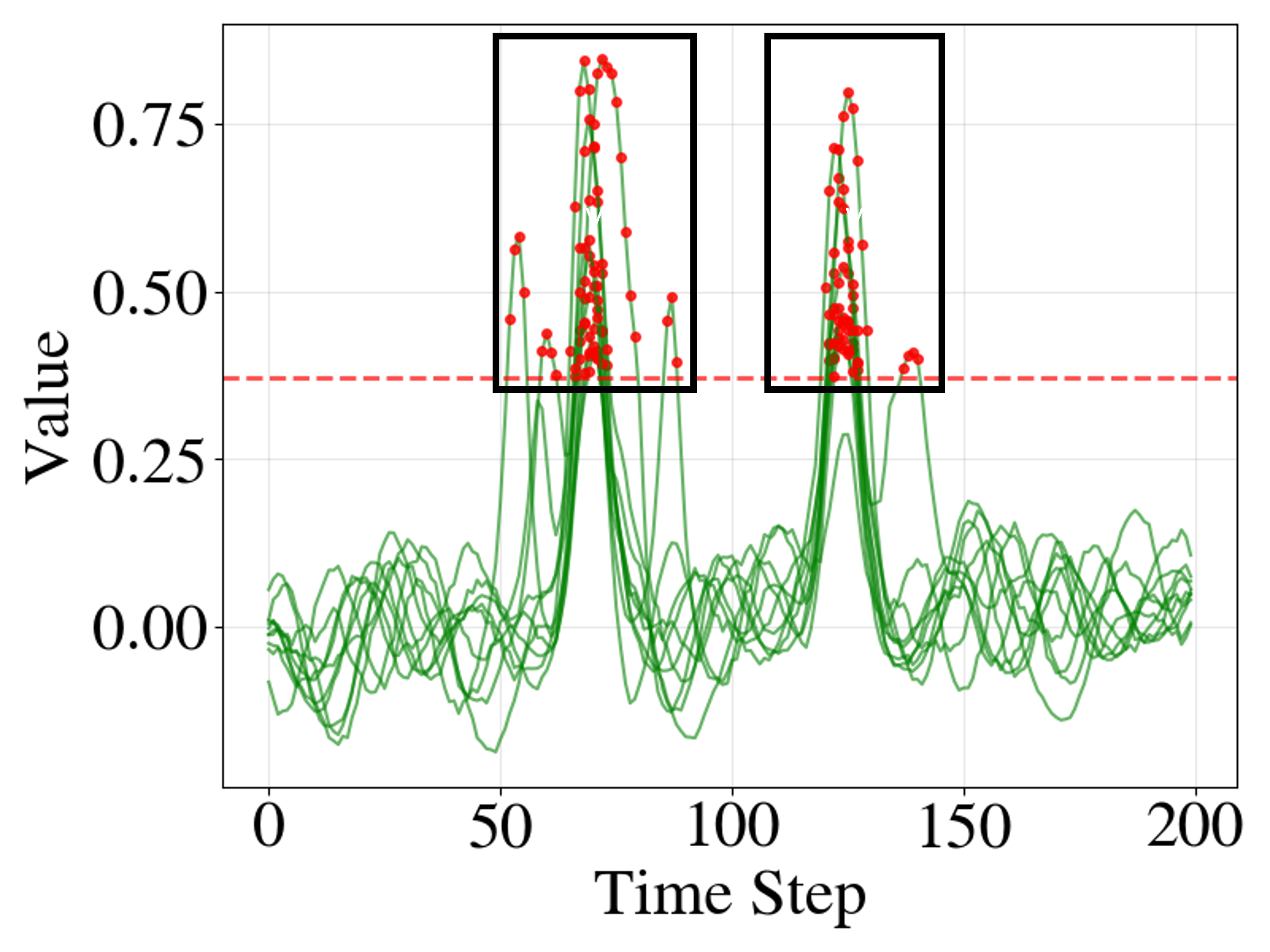}
        \label{fig:condition_generation_1}
    \end{subfigure}
    \hfill
    \begin{subfigure}{0.32\linewidth}
        \centering
        \includegraphics[width=\linewidth]{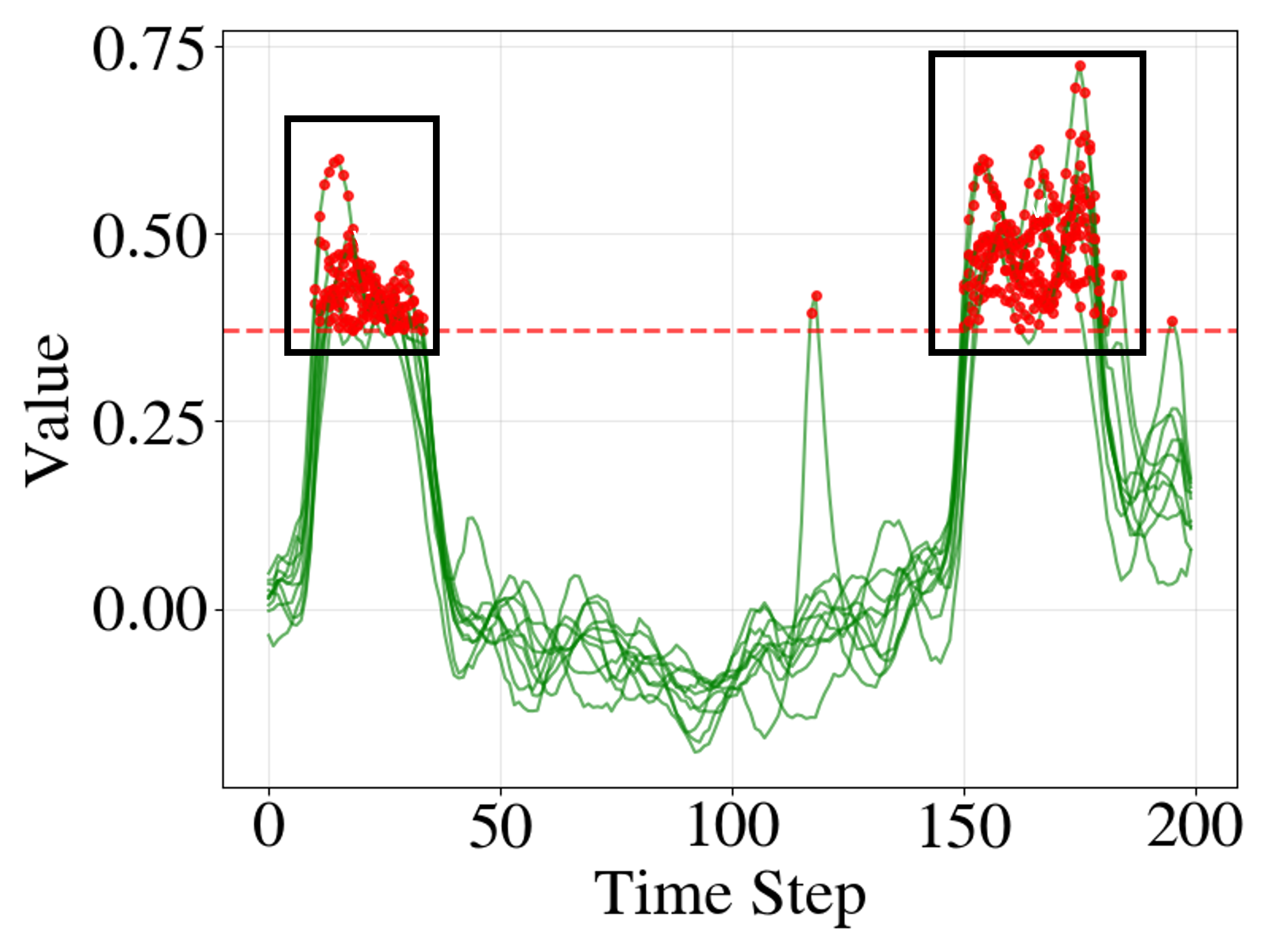}
        \label{fig:condition_generation_2}
    \end{subfigure}
    \hfill
    \begin{subfigure}{0.32\linewidth}
        \centering
        \includegraphics[width=\linewidth]{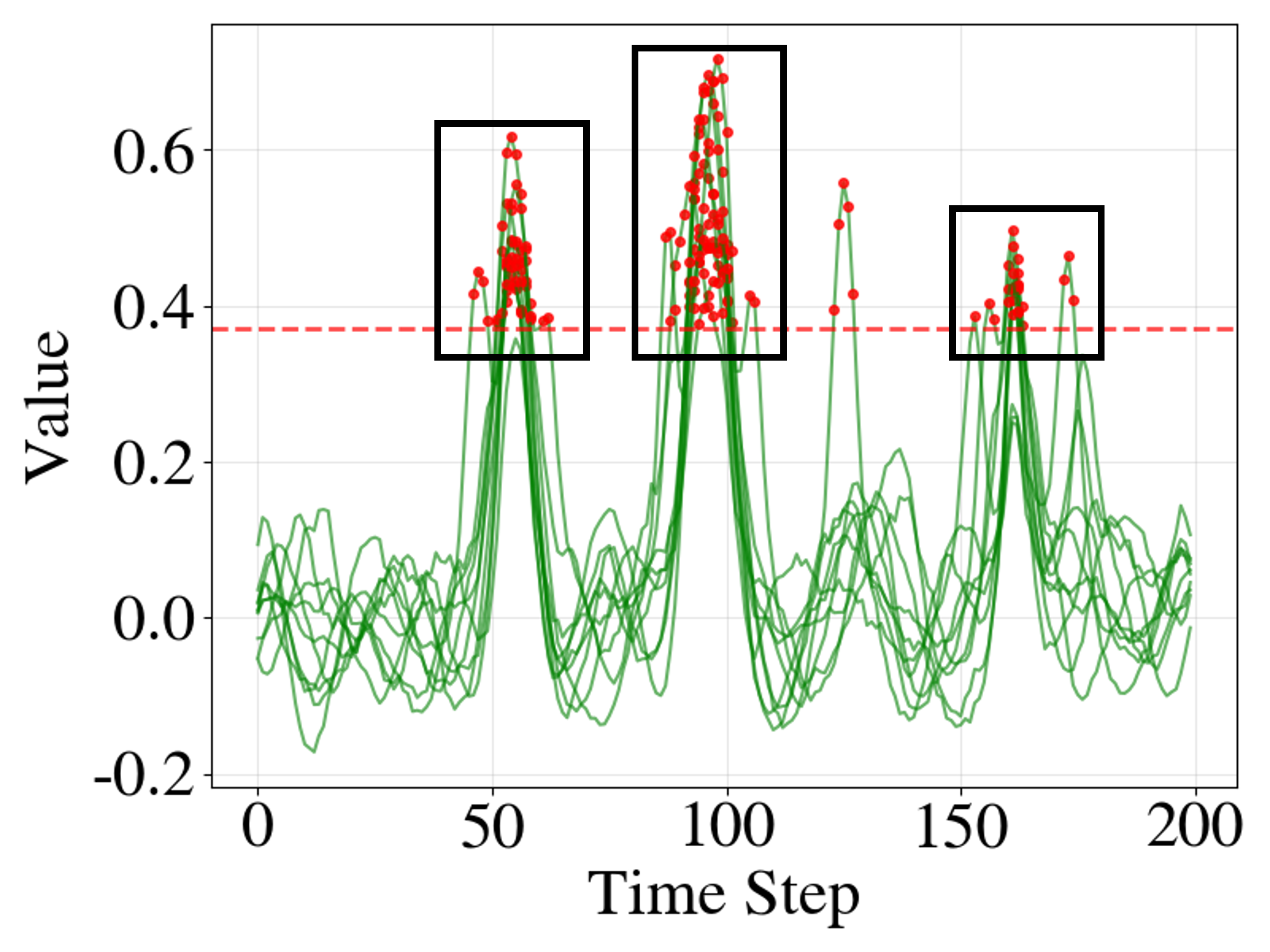}
        \label{fig:condition_generation_3}
    \end{subfigure}

\caption{Controllable extreme-event generation results under three user-specified semantic configurations on Syn-Data. 
For each configuration, 20 generated time-series samples are visualized to assess whether E4GEN follows the specified event-level control signals. 
The red dashed line denotes the extreme threshold, red points indicate extreme values, and black squares highlight the aggregated extreme-event regions.}
    \label{fig:condition_generation_three}
\end{figure}

\FloatBarrier
\section{Additional Experimental Results on Generation Fidelity and Utility}\label{Appendix:more_experiment_results}
In this part, we present additional experimental results on the remaining five datasets, including \textsc{Syn-Data}, \textsc{Wea-Prec}, \textsc{LTST-ECG}, \textsc{HH-Power}, and \textsc{PEMS-SF}. We also evaluate each dataset from three perspectives: overall generation fidelity, extreme-event generation fidelity, and downstream utility. In these tables, $\infty$ indicates that the corresponding metric cannot be computed due to an insufficient number of generated extreme events.

\begin{table*}[htbp]
\centering
\footnotesize
\renewcommand{\arraystretch}{1.12}
\setlength{\tabcolsep}{2pt}

\caption{Overall Generation Fidelity and General Downstream Utility on Syn-Data}
\label{tab:overall_syn}
\begin{tabular}{cclcccccccc}
\toprule
\textbf{Dataset} & \textbf{Type} & \textbf{Method} & \textbf{Wass.} & \textbf{KS} & \textbf{JS} & \textbf{MMD} & \textbf{ACD} & \textbf{Context\_FID} & \textbf{Pred} & \textbf{Recon} \\
\midrule
\multirow{11}{*}{\makecell{\textbf{Syn-}\\\textbf{Data}}}
& \textbf{GAN}
& TimeGAN (NeurIPS'19)~\cite{yoon2019time}
& 0.0267 & 0.1038 & 0.1995 & 0.1018 & 0.1592 & 2.0077 & 0.2258 & 0.1547 \\

\arrayrulecolor{gray!60}\cmidrule(lr){2-11}\arrayrulecolor{black}

& \multirow{2}{*}{\textbf{VAE}}
& TimeVAE (ArXiv'21)~\cite{desai2021timevae}
& 0.0650 & 0.3359 & 0.4424 & 0.2340 & 0.3998 & 16.284 & 0.2456 & 0.0764 \\

&
& koVAE (ICLR'24)~\cite{naimangenerative}
& 0.0425 & 0.1425 & 0.2669 & 0.2086 & 0.0779 & 6.7766 & 0.1746 & \underline{0.0422} \\

\arrayrulecolor{gray!60}\cmidrule(lr){2-11}\arrayrulecolor{black}

& \textbf{Flow}
& F-Flow (ICLR'21)~\cite{alaa2021generative}
& 0.1395 & 0.6424 & 0.5579 & 0.6539 & 0.0819 & 17.130 & 0.1833 & 0.0514 \\

\arrayrulecolor{gray!60}\cmidrule(lr){2-11}\arrayrulecolor{black}

& \multirow{2}{*}{\textbf{Diffusion}}
& DiffWave (ICLR'21)~\cite{kongdiffwave}
& 0.0557 & 0.2348 & 0.3767 & 0.2191 & 0.1630 & 10.798 & 0.2264 & 0.0706 \\

&
& Diffusion-TS (ICLR'24)~\cite{yuan2024diffusionts}
& 0.0236 & 0.1032 & 0.1719 & 0.3052 & 0.0920 & 1.9910 & 0.1641 & 0.0557 \\

\arrayrulecolor{gray!60}\cmidrule(lr){2-11}\arrayrulecolor{black}

& \textbf{LLM}
& SDForger (NeurIPS'25)~\cite{rousseau2025forging}
& \underline{0.0203} & \underline{0.0628} & 0.1369 & \underline{0.0708} & \textbf{0.0427} & \underline{1.2159} & \underline{0.1548} & 0.0525 \\

\arrayrulecolor{gray!60}\cmidrule(lr){2-11}\arrayrulecolor{black}

& \multirow{2}{*}{\makecell{\textbf{Extreme-}\\\textbf{Aware}}}
& FIDE (NeurIPS'24)~\cite{galib2024fide}
& 0.0309 & 0.1279 & \textbf{0.1075} & 0.1558 & 0.1614 & 4.4028 & 0.6670 & 0.8808 \\

&
& HeavyDiff (ICLR'25)~\cite{pandey2025heavy}
& 0.0231 & 0.0971 & 0.1718 & 0.2470 & 0.0784 & 1.5271 & 0.1609 & 0.0675 \\

\arrayrulecolor{gray!60}\cmidrule(lr){2-11}\arrayrulecolor{black}

& \textbf{Ours}
& \textbf{E4GEN (Ours)}
& \textbf{0.0144} & \textbf{0.0568} & \underline{0.1106} & \textbf{0.0447} & \underline{0.0505} & \textbf{0.4501} & \textbf{0.1518} & \textbf{0.0368} \\
\bottomrule
\end{tabular}

\vspace{6.0em}

\caption{Extreme-Event Generation Fidelity and Extreme-Oriented Downstream Utility on Syn-Data}
\label{tab:extreme_syn}
\begin{tabular}{cclccccccccc}
\toprule
\textbf{Dataset} & \textbf{Type} & \textbf{Method} & \textbf{EM-W1} & \textbf{EC-Diff} & \textbf{ED-W1} & \textbf{EP-W1} & \textbf{EWass.} & \textbf{EJS} & \textbf{EMMD} & \textbf{EPred} & \textbf{ERecon} \\
\midrule
\multirow{11}{*}{\makecell{\textbf{Syn-}\\\textbf{Data}}}
& \textbf{GAN}
& TimeGAN ~\cite{yoon2019time}
& 0.4755 & 6218 & 3.1726 & 0.4779 & $\infty$ & $\infty$ & $\infty$ & $\infty$ & $\infty$ \\

\arrayrulecolor{gray!60}\cmidrule(lr){2-12}\arrayrulecolor{black}

& \multirow{2}{*}{\textbf{VAE}}
& TimeVAE ~\cite{desai2021timevae}
& 0.4378 & 5820 & 3.2847 & 0.4619 & 2.1274 & 0.8318 & 1.2573 & 0.5816 & 2.3121 \\

&
& koVAE ~\cite{naimangenerative}
& 0.4334 & 6439 & 3.2525 & 0.4570 & 2.1200 & 0.8256 & 0.7437 & 0.5727 & 0.1014 \\

\arrayrulecolor{gray!60}\cmidrule(lr){2-12}\arrayrulecolor{black}

& \textbf{Flow}
& F-Flow ~\cite{alaa2021generative}
& 0.3511 & 5513 & 2.5745 & 0.3731 & 1.8070 & 0.7059 & 1.1253 & 0.5055 & 0.1012 \\

\arrayrulecolor{gray!60}\cmidrule(lr){2-12}\arrayrulecolor{black}

& \multirow{2}{*}{\textbf{Diffusion}}
& DiffWave ~\cite{kongdiffwave}
& 0.4381 & 6463 & 3.2765 & 0.4618 & $\infty$ & $\infty$ & $\infty$ & $\infty$ & $\infty$ \\

&
& Diffusion-TS ~\cite{yuan2024diffusionts}
& 0.1139 & 3061 & 3.0921 & 0.1696 & 0.6257 & 0.4534 & 0.5606 & \underline{0.3017} & \underline{0.0942} \\

\arrayrulecolor{gray!60}\cmidrule(lr){2-12}\arrayrulecolor{black}

& \textbf{LLM}
& SDForger ~\cite{rousseau2025forging}
& 0.3661 & 5212 & \underline{2.5214} & 0.3872 & 1.7097 & 0.6607 & 0.9180 & 0.4600 & 0.1224 \\

\arrayrulecolor{gray!60}\cmidrule(lr){2-12}\arrayrulecolor{black}

& \multirow{2}{*}{\makecell{\textbf{Extreme-}\\\textbf{Aware}}}
& FIDE ~\cite{galib2024fide}
& 0.1499 & 3864 & 2.5898 & 0.1678 & 9.7145 & 0.5724 & 0.7376 & 0.4647 & 10.997 \\

&
& HeavyDiff ~\cite{pandey2025heavy}
& \underline{0.0784} & \underline{2200} & 2.9818 & \underline{0.1216} & \underline{0.5113} & \underline{0.4026} & \underline{0.5528} & \textbf{0.2709} & 0.0978 \\

\arrayrulecolor{gray!60}\cmidrule(lr){2-12}\arrayrulecolor{black}

& \textbf{Ours}
& \textbf{E4GEN (Ours)}
& \textbf{0.0245} & \textbf{1232} & \textbf{1.5383} & \textbf{0.0360} & \textbf{0.3451} & \textbf{0.2689} & \textbf{0.3615} & 0.3129 & \textbf{0.0871} \\
\bottomrule
\end{tabular}

\renewcommand{\arraystretch}{1.0}
\normalsize
\end{table*}

\newpage

\begin{table*}[!b]

\centering

\footnotesize

\renewcommand{\arraystretch}{1.12}

\setlength{\tabcolsep}{2pt}


\caption{Overall Generation Fidelity and General Downstream Utility on WEA-Prec Dataset}

\label{tab:overall_prec}

\begin{tabular}{cclcccccccc}

\toprule

\textbf{Dataset} & \textbf{Type} & \textbf{Method} & \textbf{Wass.} & \textbf{KS} & \textbf{JS} & \textbf{MMD} & \textbf{ACD} & \textbf{Context\_FID} & \textbf{Pred} & \textbf{Recon} \\

\midrule

\multirow{11}{*}{\makecell{\textbf{WEA-}\\\textbf{Prec}}}

& \textbf{GAN}

& TimeGAN (NeurIPS'19)~\cite{yoon2019time}

& 484.09 & 0.7788 & 0.7501 & 0.2408 & 242.06 & 219.36 & 0.5706 & 1.0104 \\

\arrayrulecolor{gray!60}\cmidrule(lr){2-11}\arrayrulecolor{black}

& \multirow{2}{*}{\textbf{VAE}}

& TimeVAE (ArXiv'21)~\cite{desai2021timevae}

& 423.34 & 0.7357 & 0.6436 & 0.2279 & 211.75 & 182.57 & 0.6858 & 0.9595 \\

&

& koVAE (ICLR'24)~\cite{naimangenerative}

& 463.50 & 0.8015 & 0.7459 & 0.2420 & 231.83 & 210.14 & 0.9813 & 2.1157 \\

\arrayrulecolor{gray!60}\cmidrule(lr){2-11}\arrayrulecolor{black}

& \textbf{Flow}

& F-Flow (ICLR'21)~\cite{alaa2021generative}

& 327.02 & 0.4024 & 0.2930 & 0.3219 & 163.52 & 251.63 & 0.4624 & 1.1018 \\

\arrayrulecolor{gray!60}\cmidrule(lr){2-11}\arrayrulecolor{black}

& \multirow{2}{*}{\textbf{Diffusion}}

& DiffWave (ICLR'21)~\cite{kongdiffwave}

& 271.15 & 0.5122 & 0.2098 & 0.2824 & 135.58 & \underline{174.39} & 0.4119 & 0.9321 \\

&

& Diffusion-TS (ICLR'24)~\cite{yuan2024diffusionts}

& 172.16 & 0.5574 & \underline{0.1558} & \textbf{0.0876} & 86.088 & 301.04 & \underline{0.4053} & 1.4584 \\

\arrayrulecolor{gray!60}\cmidrule(lr){2-11}\arrayrulecolor{black}

& \textbf{LLM}

& SDForger (NeurIPS'25)~\cite{rousseau2025forging}

& 611.65 & \underline{0.3527} & 0.4921 & 0.1592 & 305.84 & 545.03 & 0.5214 & \underline{0.8450} \\

\arrayrulecolor{gray!60}\cmidrule(lr){2-11}\arrayrulecolor{black}

& \multirow{2}{*}{\makecell{\textbf{Extreme-}\\\textbf{Aware}}}

& FIDE (NeurIPS'24)~\cite{galib2024fide}

& \underline{137.61} & 0.4410 & 0.2329 & 0.0898 & \underline{68.818} & 311.95 & 0.4466 & 0.8916 \\

&

& HeavyDiff (ICLR'25)~\cite{pandey2025heavy}

& 158.32 & 0.5673 & 0.2240 & 0.1306 & 79.172 & 317.88 & 0.4117 & 0.8670 \\

\arrayrulecolor{gray!60}\cmidrule(lr){2-11}\arrayrulecolor{black}

& \textbf{Ours}

& \textbf{E4GEN (Ours)}

& \textbf{133.20} & \textbf{0.2114} & \textbf{0.1425} & \underline{0.0886} & \textbf{66.607} & \textbf{130.68} & \textbf{0.4033} & \textbf{0.7913} \\

\bottomrule

\end{tabular}

\vspace{6.0em}

\caption{Extreme-Event Generation Fidelity and Extreme-Oriented Downstream Utility on WEA-Prec Dataset}

\label{tab:extreme_prec}

\begin{tabular}{cclccccccccc}

\toprule

\textbf{Dataset} & \textbf{Type} & \textbf{Method} & \textbf{EM-W1} & \textbf{EC-Diff} & \textbf{ED-W1} & \textbf{EP-W1} & \textbf{EWass.} & \textbf{EJS} & \textbf{EMMD} & \textbf{EPred} & \textbf{ERecon} \\

\midrule

\multirow{11}{*}{\makecell{\textbf{WEA-}\\\textbf{Prec}}}

& \textbf{GAN}

& TimeGAN ~\cite{yoon2019time}

& 2.5847 & 9661 & 1.3385 & 2.7133 & $\infty$ & $\infty$ & $\infty$ & $\infty$ & $\infty$ \\

\arrayrulecolor{gray!60}\cmidrule(lr){2-12}\arrayrulecolor{black}

& \multirow{2}{*}{\textbf{VAE}}

& TimeVAE ~\cite{desai2021timevae}

& 2.5234 & 8532 & 1.3185 & 2.7128 & 16.450 & 0.8323 & 5.4374 & 1.6613 & \underline{2.2534} \\

&

& koVAE ~\cite{naimangenerative}

& 2.5088 & 9427 & 1.3157 & 2.6977 & 16.443 & 0.8311 & 5.4416 & 1.7204 & 5.1512 \\

\arrayrulecolor{gray!60}\cmidrule(lr){2-12}\arrayrulecolor{black}

& \textbf{Flow}

& F-Flow ~\cite{alaa2021generative}

& 2.5163 & 9434 & 1.3223 & 2.7052 & $\infty$ & $\infty$ & $\infty$ & $\infty$ & $\infty$ \\

\arrayrulecolor{gray!60}\cmidrule(lr){2-12}\arrayrulecolor{black}

& \multirow{2}{*}{\textbf{Diffusion}}

& DiffWave ~\cite{kongdiffwave}

& 1.7680 & 7834 & 0.5415 & 1.9528 & 13.436 & 0.7506 & 4.3148 & 1.5536 & 2.3741 \\

&

& Diffusion-TS ~\cite{yuan2024diffusionts}

& 1.2734 & 4331 & 0.1525 & 1.3156 & 8.1159 & 0.4793 & 2.0414 & 2.3124 & 3.0427 \\

\arrayrulecolor{gray!60}\cmidrule(lr){2-12}\arrayrulecolor{black}

& \textbf{LLM}

& SDForger ~\cite{rousseau2025forging}

& \underline{0.7438} & 13462 & 0.4953 & \underline{0.7883} & 13.625 & 0.5380 & 2.1673 & 1.3387 & 3.9173 \\

\arrayrulecolor{gray!60}\cmidrule(lr){2-12}\arrayrulecolor{black}

& \multirow{2}{*}{\makecell{\textbf{Extreme-}\\\textbf{Aware}}}

& FIDE ~\cite{galib2024fide}

& 1.1417 & \textbf{3130} & 0.1693 & 1.2455 & 48.397 & \textbf{0.2672} & \underline{1.4808} & \textbf{1.2212} & 17.326 \\

&

& HeavyDiff ~\cite{pandey2025heavy}

& 1.3451 & 3990 & \underline{0.1436} & 1.3779 & \underline{6.9423} & 0.4110 & \textbf{1.1951} & 1.4888 & 3.1175 \\

\arrayrulecolor{gray!60}\cmidrule(lr){2-12}\arrayrulecolor{black}

& \textbf{Ours}

& \textbf{E4GEN (Ours)}

& \textbf{0.4286} & \underline{3619} & \textbf{0.1004} & \textbf{0.4341} & \textbf{6.4997} & \underline{0.4056} & 1.7371 & \underline{1.3368} & \textbf{2.2081} \\

\bottomrule

\end{tabular}

\renewcommand{\arraystretch}{1.0}

\normalsize

\end{table*}

\newpage
\begin{table*}[!b]
\centering
\footnotesize
\renewcommand{\arraystretch}{1.12}
\setlength{\tabcolsep}{2pt}

\vspace{-10pt}
\caption{Overall Generation Fidelity and General Downstream Utility on HH-Power Dataset}
\label{tab:overall_elect}
\begin{tabular}{cclcccccccc}
\toprule
\textbf{Dataset} & \textbf{Type} & \textbf{Method} & \textbf{Wass.} & \textbf{KS} & \textbf{JS} & \textbf{MMD} & \textbf{ACD} & \textbf{Context\_FID} & \textbf{Pred} & \textbf{Recon} \\
\midrule
\multirow{11}{*}{\makecell{\textbf{HH-}\\\textbf{Power}}}
& \textbf{GAN}
& TimeGAN (NeurIPS'19)~\cite{yoon2019time}
& 0.1904 & 0.1721 & 0.2317 & 0.2122 & 0.2285 & 1.6236 & 0.3206 & 0.7460 \\

\arrayrulecolor{gray!60}\cmidrule(lr){2-11}\arrayrulecolor{black}

& \multirow{2}{*}{\textbf{VAE}}
& TimeVAE (ArXiv'21)~\cite{desai2021timevae}
& 0.3103 & 0.2509 & 0.2976 & 0.1898 & 0.2678 & 4.3785 & 0.3498 & 0.2720 \\

&
& koVAE (ICLR'24)~\cite{naimangenerative}
& 0.3206 & 0.1882 & 0.3334 & 0.2714 & 0.2572 & 2.8463 & 0.3294 & 0.2671 \\

\arrayrulecolor{gray!60}\cmidrule(lr){2-11}\arrayrulecolor{black}

& \textbf{Flow}
& F-Flow (ICLR'21)~\cite{alaa2021generative}
& 0.9752 & 0.7554 & 0.6198 & 0.6760 & 0.4951 & 22.1413 & 0.5614 & 0.4091 \\

\arrayrulecolor{gray!60}\cmidrule(lr){2-11}\arrayrulecolor{black}

& \multirow{2}{*}{\textbf{Diffusion}}
& DiffWave (ICLR'21)~\cite{kongdiffwave}
& 0.4087 & 0.2659 & 0.3705 & 0.3140 & 0.2245 & 4.6475 & 0.3439 & 0.2495 \\

&
& Diffusion-TS (ICLR'24)~\cite{yuan2024diffusionts}
& \underline{0.1737} & 0.1584 & \textbf{0.1350} & 0.2473 & \underline{0.1273} & 1.0810 & \underline{0.3023} & 0.2535 \\

\arrayrulecolor{gray!60}\cmidrule(lr){2-11}\arrayrulecolor{black}

& \textbf{LLM}
& SDForger (NeurIPS'25)~\cite{rousseau2025forging}
& 0.2182 & 0.1636 & 0.3021 & 0.1727 & 0.1463 & 136.4964 & 0.3385 & 0.3391 \\

\arrayrulecolor{gray!60}\cmidrule(lr){2-11}\arrayrulecolor{black}

& \multirow{2}{*}{\makecell{\textbf{Extreme-}\\\textbf{Aware}}}
& FIDE (NeurIPS'24)~\cite{galib2024fide}
& 0.1769 & 0.1540 & 0.2170 & \underline{0.1414} & 0.1944 & 1.8600 & 0.5878 & 0.4356 \\

&
& HeavyDiff (ICLR'25)~\cite{pandey2025heavy}
& 0.1856 & \underline{0.1420} & 0.1872 & 0.2218 & 0.1286 & \textbf{0.9149} & 0.3123 & \underline{0.2328} \\

\arrayrulecolor{gray!60}\cmidrule(lr){2-11}\arrayrulecolor{black}

& \textbf{Ours}
& \textbf{E4GEN (Ours)}
& \textbf{0.1698} & \textbf{0.1315} & \underline{0.1741} & \textbf{0.1375} & \textbf{0.1246} & \underline{1.0035} & \textbf{0.2999} & \textbf{0.2278} \\
\bottomrule
\end{tabular}

\vspace{6.0em}

\caption{Extreme-Event Generation Fidelity and Extreme-Oriented Downstream Utility on HH-Power Dataset}
\label{tab:extreme_elect}
\begin{tabular}{cclccccccccc}
\toprule
\textbf{Dataset} & \textbf{Type} & \textbf{Method} & \textbf{EM-W1} & \textbf{EC-Diff} & \textbf{ED-W1} & \textbf{EP-W1} & \textbf{EWass.} & \textbf{EJS} & \textbf{EMMD} & \textbf{EPred} & \textbf{ERecon} \\
\midrule
\multirow{11}{*}{\makecell{\textbf{HH-}\\\textbf{Power}}}
& \textbf{GAN}
& TimeGAN ~\cite{yoon2019time}
& 1.0383 & 3817 & 10.9922 & 1.2187 & 3.0033 & 0.5337 & 5.3360 & 0.8896 & 2.1582 \\

\arrayrulecolor{gray!60}\cmidrule(lr){2-12}\arrayrulecolor{black}

& \multirow{2}{*}{\textbf{VAE}}
& TimeVAE ~\cite{desai2021timevae}
& 2.2889 & 10623 & 2.8694 & 2.5049 & 8.6922 & 0.6534 & 6.0586 & 0.9083 & 0.8072 \\

&
& koVAE ~\cite{naimangenerative}
& 0.8455 & 13305 & 1.7526 & 1.0439 & 10.0234 & 0.6979 & 4.2831 & 0.6331 & 1.0922 \\

\arrayrulecolor{gray!60}\cmidrule(lr){2-12}\arrayrulecolor{black}

& \textbf{Flow}
& F-Flow ~\cite{alaa2021generative}
& 2.7518 & 15812 & 2.6638 & 3.9851 & $\infty$ & $\infty$ & $\infty$ & $\infty$ & $\infty$ \\

\arrayrulecolor{gray!60}\cmidrule(lr){2-12}\arrayrulecolor{black}

& \multirow{2}{*}{\textbf{Diffusion}}
& DiffWave ~\cite{kongdiffwave}
& 2.9345 & 15783 & 2.7963 & 3.1761 & $\infty$ & $\infty$ & $\infty$ & $\infty$ & $\infty$ \\

&
& Diffusion-TS ~\cite{yuan2024diffusionts}
& 0.8144 & 4505 & 1.8059 & 1.0288 & 2.2833 & 0.3284 & \textbf{0.2816} & 0.5794 & 0.9291 \\

\arrayrulecolor{gray!60}\cmidrule(lr){2-12}\arrayrulecolor{black}

& \textbf{LLM}
& SDForger ~\cite{rousseau2025forging}
& \underline{0.5677} & \underline{1383} & 4.1994 & 0.6450 & \underline{1.4516} & \underline{0.2082} & 4.1405 & 0.6264 & 1.0743 \\

\arrayrulecolor{gray!60}\cmidrule(lr){2-12}\arrayrulecolor{black}

& \multirow{2}{*}{\makecell{\textbf{Extreme-}\\\textbf{Aware}}}
& FIDE ~\cite{galib2024fide}
& 1.2496 & 3902 & 1.7325 & \underline{0.4221} & 3.0098 & 0.3105 & 2.8839 & 0.6496 & 11.6615 \\

&
& HeavyDiff ~\cite{pandey2025heavy}
& 0.8059 & 3943 & \underline{1.1517} & 0.9038 & 2.0354 & 0.3142 & \underline{0.5353} & \underline{0.5600} & \underline{0.7368} \\

\arrayrulecolor{gray!60}\cmidrule(lr){2-12}\arrayrulecolor{black}

& \textbf{Ours}
& \textbf{E4GEN (Ours)}
& \textbf{0.0862} & \textbf{886} & \textbf{0.7834} & \textbf{0.1606} & \textbf{0.6672} & \textbf{0.1458} & 1.4301 & \textbf{0.5239} & \textbf{0.7284} \\
\bottomrule
\end{tabular}

\renewcommand{\arraystretch}{1.0}
\normalsize
\end{table*}

\newpage
\begin{table*}[!b]
\centering
\footnotesize
\renewcommand{\arraystretch}{1.12}
\setlength{\tabcolsep}{2pt}

\vspace{-10pt}
\caption{Overall Generation Fidelity and General Downstream Utility on LTST-ECG Dataset}
\label{tab:overall_ecg}
\begin{tabular}{cclcccccccc}
\toprule
\textbf{Dataset} & \textbf{Type} & \textbf{Method} & \textbf{Wass.} & \textbf{KS} & \textbf{JS} & \textbf{MMD} & \textbf{ACD} & \textbf{Context\_FID} & \textbf{Pred} & \textbf{Recon} \\
\midrule
\multirow{11}{*}{\makecell{\textbf{LTST-}\\\textbf{ECG}}}
& \textbf{GAN}
& TimeGAN (NeurIPS'19)~\cite{yoon2019time}
& 0.2190 & 0.2115 & 0.2985 & 0.1540 & 0.1988 & 3.5454 & 0.1660 & 0.0521 \\

\arrayrulecolor{gray!60}\cmidrule(lr){2-11}\arrayrulecolor{black}

& \multirow{2}{*}{\textbf{VAE}}
& TimeVAE (ArXiv'21)~\cite{desai2021timevae}
& 0.2526 & 0.2971 & 0.3580 & 0.2177 & 0.3487 & 16.310 & 0.1885 & 0.0580 \\

&
& koVAE (ICLR'24)~\cite{naimangenerative}
& 0.2531 & 0.2790 & 0.3423 & 0.2952 & 0.3276 & 10.033 & 0.1126 & 0.0217 \\

\arrayrulecolor{gray!60}\cmidrule(lr){2-11}\arrayrulecolor{black}

& \textbf{Flow}
& F-Flow (ICLR'21)~\cite{alaa2021generative}
& 0.7239 & 0.8665 & 0.7669 & 0.6604 & 0.3803 & 31.992 & 0.3033 & 0.0501 \\

\arrayrulecolor{gray!60}\cmidrule(lr){2-11}\arrayrulecolor{black}

& \multirow{2}{*}{\textbf{Diffusion}}
& DiffWave (ICLR'21)~\cite{kongdiffwave}
& 0.2864 & 0.3803 & 0.4826 & 0.2535 & 0.4048 & 17.473 & 0.1877 & 0.0601 \\

&
& Diffusion-TS (ICLR'24)~\cite{yuan2024diffusionts}
& 0.0538 & 0.0934 & \underline{0.1064} & 0.2368 & \underline{0.0557} & \underline{1.1740} & 0.0945 & 0.0171 \\

\arrayrulecolor{gray!60}\cmidrule(lr){2-11}\arrayrulecolor{black}

& \textbf{LLM}
& SDForger (NeurIPS'25)~\cite{rousseau2025forging}
& 0.2390 & 0.1830 & 0.2973 & \underline{0.1137} & 0.1794 & 3.5274 & 0.1008 & 0.0192 \\

\arrayrulecolor{gray!60}\cmidrule(lr){2-11}\arrayrulecolor{black}

& \multirow{2}{*}{\makecell{\textbf{Extreme-}\\\textbf{Aware}}}
& FIDE (NeurIPS'24)~\cite{galib2024fide}
& 0.2624 & 0.2200 & 0.1659 & 0.2371 & 0.2434 & 69.932 & 0.5078 & 0.6218 \\

&
& HeavyDiff (ICLR'25)~\cite{pandey2025heavy}
& \underline{0.0443} & \underline{0.0700} & 0.1144 & 0.2054 & 0.0599 & 1.1824 & \underline{0.0708} & \textbf{0.0159} \\

\arrayrulecolor{gray!60}\cmidrule(lr){2-11}\arrayrulecolor{black}

& \textbf{Ours}
& \textbf{E4GEN (Ours)}
& \textbf{0.0352} & \textbf{0.0383} & \textbf{0.0866} & \textbf{0.1067} & \textbf{0.0357} & \textbf{0.4920} & \textbf{0.0703} & \underline{0.0171} \\
\bottomrule
\end{tabular}

\vspace{6.0em}

\caption{Extreme-Event Generation Fidelity and Extreme-Oriented Downstream Utility on LTST-ECG Dataset}
\label{tab:extreme_ecg}
\begin{tabular}{cclccccccccc}
\toprule
\textbf{Dataset} & \textbf{Type} & \textbf{Method} & \textbf{EM-W1} & \textbf{EC-Diff} & \textbf{ED-W1} & \textbf{EP-W1} & \textbf{EWass.} & \textbf{EJS} & \textbf{EMMD} & \textbf{EPred} & \textbf{ERecon} \\
\midrule
\multirow{11}{*}{\makecell{\textbf{LTST-}\\\textbf{ECG}}}
& \textbf{GAN}
& TimeGAN ~\cite{yoon2019time}
& 0.6881 & 33741 & 7.4728 & 1.2060 & 3.0622 & 0.6407 & 2.2042 & 0.7011 & 0.1413 \\

\arrayrulecolor{gray!60}\cmidrule(lr){2-12}\arrayrulecolor{black}

& \multirow{2}{*}{\textbf{VAE}}
& TimeVAE ~\cite{desai2021timevae}
& 1.3423 & 20290 & 10.199 & 1.9551 & 3.7131 & 0.7479 & 2.8147 & 0.7452 & 0.5056 \\

&
& koVAE ~\cite{naimangenerative}
& 1.0624 & 22875 & 9.5807 & 1.5670 & 3.9207 & 0.7441 & 2.0463 & 0.4008 & 0.0604 \\

\arrayrulecolor{gray!60}\cmidrule(lr){2-12}\arrayrulecolor{black}

& \textbf{Flow}
& F-Flow ~\cite{alaa2021generative}
& 0.9259 & 24703 & 9.9093 & 1.4716 & 4.5078 & 0.7858 & 2.3123 & 0.9146 & 0.2565 \\

\arrayrulecolor{gray!60}\cmidrule(lr){2-12}\arrayrulecolor{black}

& \multirow{2}{*}{\textbf{Diffusion}}
& DiffWave ~\cite{kongdiffwave}
& 1.3480 & 24996 & 11.475 & 1.9541 & 4.3085 & 0.8322 & 2.7932 & $\infty$ & 0.5797 \\

&
& Diffusion-TS ~\cite{yuan2024diffusionts}
& 0.0936 & 4960 & 3.1839 & \textbf{0.0444} & 0.9857 & 0.3071 & 0.3888 & 0.1893 & 0.0693 \\

\arrayrulecolor{gray!60}\cmidrule(lr){2-12}\arrayrulecolor{black}

& \textbf{LLM}
& SDForger ~\cite{rousseau2025forging}
& 0.6669 & 22481 & 3.9868 & 1.1220 & 3.4591 & 0.5351 & 1.4759 & 0.2774 & \textbf{0.0460} \\

\arrayrulecolor{gray!60}\cmidrule(lr){2-12}\arrayrulecolor{black}

& \multirow{2}{*}{\makecell{\textbf{Extreme-}\\\textbf{Aware}}}
& FIDE ~\cite{galib2024fide}
& 0.6834 & 16922 & 10.406 & 1.0252 & 75.5862 & 0.6080 & 1.9742 & 0.9475 & 3.4025 \\

&
& HeavyDiff ~\cite{pandey2025heavy}
& \underline{0.0331} & \underline{4754} & \underline{3.1284} & 0.0988 & \underline{0.9545} & \textbf{0.2783} & \underline{0.3306} & \underline{0.1885} & 0.0703 \\

\arrayrulecolor{gray!60}\cmidrule(lr){2-12}\arrayrulecolor{black}

& \textbf{Ours}
& \textbf{E4GEN (Ours)}
& \textbf{0.0328} & \textbf{4633} & \textbf{2.4458} & \underline{0.0561} & \textbf{0.8911} & \underline{0.2823} & \textbf{0.2650} & \textbf{0.1674} & \underline{0.0558} \\
\bottomrule
\end{tabular}

\renewcommand{\arraystretch}{1.0}
\normalsize
\end{table*}

\newpage
\begin{table*}[!b]
\centering
\footnotesize
\renewcommand{\arraystretch}{1.12}
\setlength{\tabcolsep}{2pt}

\vspace{-10pt}
\caption{Overall Generation Fidelity and General Downstream Utility on PEMS-SF Dataset}
\label{tab:overall_traffic}
\begin{tabular}{cclcccccccc}
\toprule
\textbf{Dataset} & \textbf{Type} & \textbf{Method} & \textbf{Wass.} & \textbf{KS} & \textbf{JS} & \textbf{MMD} & \textbf{ACD} & \textbf{Context\_FID} & \textbf{Pred} & \textbf{Recon} \\
\midrule
\multirow{11}{*}{\makecell{\textbf{PEMS-}\\\textbf{SF}}}
& \textbf{GAN}
& TimeGAN (NeurIPS'19)~\cite{yoon2019time}
& 0.0097 & 0.1139 & 0.1910 & 0.3917 & 0.0793 & 0.9527 & 0.2490 & 0.2349 \\

\arrayrulecolor{gray!60}\cmidrule(lr){2-11}\arrayrulecolor{black}

& \multirow{2}{*}{\textbf{VAE}}
& TimeVAE (ArXiv'21)~\cite{desai2021timevae}
& 0.0061 & 0.0810 & 0.1405 & 0.2063 & 0.0355 & 0.5406 & 0.2255 & 0.1386 \\

&
& koVAE (ICLR'24)~\cite{naimangenerative}
& 0.0106 & 0.1227 & 0.2120 & 0.3148 & 0.0901 & 0.6285 & 0.2159 & 0.1437 \\

\arrayrulecolor{gray!60}\cmidrule(lr){2-11}\arrayrulecolor{black}

& \textbf{Flow}
& F-Flow (ICLR'21)~\cite{alaa2021generative}
& 0.0272 & 0.2830 & 0.2540 & 0.3814 & 0.0254 & 0.4721 & 0.1985 & 0.1212 \\

\arrayrulecolor{gray!60}\cmidrule(lr){2-11}\arrayrulecolor{black}

& \multirow{2}{*}{\textbf{Diffusion}}
& DiffWave (ICLR'21)~\cite{kongdiffwave}
& 0.0194 & 0.2237 & 0.3908 & 0.4820 & 0.0585 & 2.2068 & 0.2550 & 0.1681 \\

&
& Diffusion-TS (ICLR'24)~\cite{yuan2024diffusionts}
& \underline{0.0032} & \underline{0.0749} & 0.0858 & 0.1138 & \underline{0.0237} & \underline{0.0758} & 0.1879 & 0.1218 \\

\arrayrulecolor{gray!60}\cmidrule(lr){2-11}\arrayrulecolor{black}

& \textbf{LLM}
& SDForger (NeurIPS'25)~\cite{rousseau2025forging}
& 0.0049 & 0.0744 & 0.1334 & \underline{0.1045} & 0.0285 & 0.2002 & \textbf{0.1848} & 0.1252 \\

\arrayrulecolor{gray!60}\cmidrule(lr){2-11}\arrayrulecolor{black}

& \multirow{2}{*}{\makecell{\textbf{Extreme-}\\\textbf{Aware}}}
& FIDE (NeurIPS'24)~\cite{galib2024fide}
& 0.0203 & 0.2319 & 0.1561 & 0.3026 & 0.0612 & 0.7426 & 0.2444 & 0.1729 \\

&
& HeavyDiff (ICLR'25)~\cite{pandey2025heavy}
& 0.0051 & 0.0940 & \underline{0.0785} & 0.1328 & \textbf{0.0189} & 0.0842 & 0.1942 & \underline{0.1202} \\

\arrayrulecolor{gray!60}\cmidrule(lr){2-11}\arrayrulecolor{black}

& \textbf{Ours}
& \textbf{E4GEN (Ours)}
& \textbf{0.0031} & \textbf{0.0293} & \textbf{0.0554} & \textbf{0.0969} & 0.0382 & \textbf{0.0755} & \underline{0.1862} & \textbf{0.1167} \\
\bottomrule
\end{tabular}

\vspace{6.0em}

\caption{Extreme-Event Generation Fidelity and Extreme-Oriented Downstream Utility on PEMS-SF Dataset}
\label{tab:extreme_traffic}
\begin{tabular}{cclccccccccc}
\toprule
\textbf{Dataset} & \textbf{Type} & \textbf{Method} & \textbf{EM-W1} & \textbf{EC-Diff} & \textbf{ED-W1} & \textbf{EP-W1} & \textbf{EWass.} & \textbf{EJS} & \textbf{EMMD} & \textbf{EPred} & \textbf{ERecon} \\
\midrule
\multirow{11}{*}{\makecell{\textbf{PEMS-}\\\textbf{SF}}}
& \textbf{GAN}
& TimeGAN ~\cite{yoon2019time}
& 0.0649 & 11176 & 3.3566 & 0.0732 & 0.6028 & 0.5959 & 0.3982 & 0.8045 & 0.9457 \\

\arrayrulecolor{gray!60}\cmidrule(lr){2-12}\arrayrulecolor{black}

& \multirow{2}{*}{\textbf{VAE}}
& TimeVAE ~\cite{desai2021timevae}
& 0.0321 & 5517 & \underline{1.6787} & 0.0400 & 0.3642 & 0.3446 & 0.2809 & 0.6911 & 0.9729 \\

&
& koVAE ~\cite{naimangenerative}
& 0.0710 & 16200 & 2.4943 & 0.0785 & 0.8908 & 0.7015 & 0.4640 & 0.6735 & 0.8006 \\

\arrayrulecolor{gray!60}\cmidrule(lr){2-12}\arrayrulecolor{black}

& \textbf{Flow}
& F-Flow ~\cite{alaa2021generative}
& 0.0323 & 42717 & 9.5380 & 0.0378 & 1.1185 & 0.5896 & \underline{0.1772} & 0.4918 & 0.6230 \\

\arrayrulecolor{gray!60}\cmidrule(lr){2-12}\arrayrulecolor{black}

& \multirow{2}{*}{\textbf{Diffusion}}
& DiffWave ~\cite{kongdiffwave}
& 0.1012 & 19121 & 4.1817 & 0.1104 & $\infty$ & $\infty$ & $\infty$ & $\infty$ & $\infty$ \\

&
& Diffusion-TS ~\cite{yuan2024diffusionts}
& \underline{0.0185} & 2100 & 2.5681 & \underline{0.0200} & 0.3262 & 0.4283 & 0.2523 & \underline{0.4730} & 0.6254 \\

\arrayrulecolor{gray!60}\cmidrule(lr){2-12}\arrayrulecolor{black}

& \textbf{LLM}
& SDForger ~\cite{rousseau2025forging}
& 0.0290 & 5365 & 5.9978 & 0.0287 & 0.3428 & \textbf{0.3364} & 0.2174 & 0.4759 & 0.6466 \\

\arrayrulecolor{gray!60}\cmidrule(lr){2-12}\arrayrulecolor{black}

& \multirow{2}{*}{\makecell{\textbf{Extreme-}\\\textbf{Aware}}}
& FIDE ~\cite{galib2024fide}
& 0.0242 & 25025 & 2.0519 & 0.0279 & 4.1184 & 0.5097 & 0.2424 & 0.5760 & 0.8926 \\

&
& HeavyDiff ~\cite{pandey2025heavy}
& 0.0217 & \underline{1767} & 2.7466 & 0.0373 & \textbf{0.2817} & 0.3951 & 0.2133 & 0.4791 & \underline{0.6213} \\

\arrayrulecolor{gray!60}\cmidrule(lr){2-12}\arrayrulecolor{black}

& \textbf{Ours}
& \textbf{E4GEN (Ours)}
& \textbf{0.0138} & \textbf{1473} & \textbf{0.7813} & \textbf{0.0154} & \underline{0.2980} & \underline{0.3400} & \textbf{0.1505} & \textbf{0.4667} & \textbf{0.6046} \\
\bottomrule
\end{tabular}

\renewcommand{\arraystretch}{1.0}
\normalsize
\end{table*}


\end{document}